\documentclass[twoside]{article}

\usepackage[accepted]{aistats2025}
\usepackage[toc,page,header]{appendix}
\usepackage{minitoc}

\usepackage{amsmath,amsfonts,bm}

\def\eqref#1{Equation~\ref{#1}}

\def\ceil#1{\left\lceil #1 \right\rceil}

\def\1{\bm{1}}

\DeclareMathAlphabet{\mathsfit}{\encodingdefault}{\sfdefault}{m}{sl}
\SetMathAlphabet{\mathsfit}{bold}{\encodingdefault}{\sfdefault}{bx}{n}

\newcommand{\R}{\mathbb{R}}

\newcommand{\KL}{D_{\mathrm{KL}}}

\usepackage{multirow}
\usepackage{tabularray}
\usepackage{adjustbox}
\usepackage{etoolbox}
\usepackage{footnote}
\usepackage{dsfont}
\makesavenoteenv{tabular}
\usepackage{dcolumn}
\usepackage{makecell}

\usepackage{enumitem}
\setlist[itemize]{leftmargin=5.5mm}
\usepackage{color}
\usepackage{microtype}
\usepackage{graphicx}
\usepackage{subfigure}
\usepackage{booktabs} 

\definecolor{mydarkblue}{rgb}{0,0.08,0.45}
\usepackage[colorlinks=true, linkcolor=mydarkblue, citecolor=mydarkblue]{hyperref}

\usepackage{amsmath}
\usepackage{amssymb}
\usepackage{mathtools}
\usepackage{amsthm}
\newif\ifdoublecolumn

\usepackage[capitalize,noabbrev]{cleveref}

\theoremstyle{plain}
\newtheorem{theorem}{Theorem}[section]

\newtheorem{lemma}[theorem]{Lemma}

\theoremstyle{definition}
\newtheorem{definition}[theorem]{Definition}

\theoremstyle{remark}
\newtheorem{remark}[theorem]{Remark}

\DeclareMathOperator{\proj}{proj}

\newtheorem*{statement*}{Statement}

\usepackage{thmtools}
\usepackage{thm-restate}

\allowdisplaybreaks

\newcommand{\calL}{\mathcal{L}}

\newcommand{\calS}{\mathcal{S}}

\newcommand{\calR}{\mathcal{R}}
\newcommand{\calV}{\mathcal{V}}

\newcommand{\calA}{\mathcal{A}}
\newcommand{\calQ}{\mathcal{Q}}
\newcommand{\expp}{\mathbb{E}}
\newcommand{\norm}[1]{\left\lVert#1\right\rVert}
\newcommand{\cleandata}{\overline{D}}

\newcommand{\br}[1]{\left({#1}\right)}
\newcommand{\bs}[1]{\left[{#1}\right]}

\newcommand{\abs}[1]{\left\vert#1\right\vert}

\newcommand{\expectover}[2]{\mathbb{E}_{#1}\!\bs{#2}}

\usepackage{natbib}
\usepackage{hyperref}
\begin{document}
%
\runningtitle{Policy Teaching in Learning from Human Preferences}

%
\runningauthor{Nika, Nöther, Mandal, Kamalaruban, Singla and Radanović}

\twocolumn[
\aistatstitle{Policy Teaching via Data Poisoning \\ in Learning from Human Preferences}
\aistatsauthor{ Andi Nika\\MPI-SWS \And Jonathan Nöther\\MPI-SWS \And  Debmalya Mandal\\University of Warwick \AND Parameswaran Kamalaruban\\Featurespace \And Adish Singla\\MPI-SWS \And Goran Radanović\\MPI-SWS }
\vspace{0.5cm}]



\begin{abstract}
     \looseness-1We study data poisoning attacks in learning from human preferences. More specifically, we consider the problem of teaching/enforcing a target policy $\pi^\dagger$ by synthesizing preference data. We seek to understand the susceptibility of different preference-based learning paradigms to poisoned preference data by analyzing the number of samples required by the attacker to enforce $\pi^\dagger$. We first propose a general data poisoning formulation in learning from human preferences and then study it for two popular paradigms, namely: (a) reinforcement learning from human feedback (RLHF) that operates by learning a reward model using preferences; (b) direct preference optimization (DPO) that directly optimizes policy using preferences. We conduct a theoretical analysis of the effectiveness of data poisoning in a setting where the attacker is allowed to augment a pre-existing dataset and also study its special case where the attacker can synthesize the entire preference dataset from scratch. As our main results, we provide lower/upper bounds on the number of samples required to enforce $\pi^\dagger$. Finally, we discuss the implications of our results in terms of the susceptibility of these learning paradigms under such data poisoning attacks. 
\end{abstract}

\doparttoc
\faketableofcontents

\section{Introduction}\label{sec:introduction}

\looseness-1Learning from human preferences has recently attracted considerable attention, largely due to its effectiveness in fine-tuning large language models (LLMs). Unlike the traditional approach, which uses training data labeled with absolute scores, this method relies on pairs of examples marked with binary signals indicating preference—essentially assigning a relative score to each example. As such, it has proven to be practically beneficial since comparative feedback between two examples is more easily accessible than their individual absolute scores.

Despite its practical advantages, learning from human preferences is susceptible to data poisoning attacks \citep{biggio2012poisoning}, due to the potential presence of malicious human feedback in the training data. Malignant third parties could easily alter preference datasets to steer LLMs toward generating biased or harmful content which can lead to undesired model behaviours. This vulnerability is particularly concerning, given the rapid integration of LLMs into critical applications. It is thus essential to understand these attacks in order to design models with robust guarantees against them. 

These concerns have motivated a lot of recent work, all of which has focused on empirical investigations of poisoning attacks. For example, \cite{wang2023exploitability} demonstrated the effectiveness of ranking poisoning attacks, where attackers manipulate preference labels without altering the underlying data. \cite{shi2023badgpt} showed how an attacker could inject trigger words into training prompts, influencing the LLM's sentiment analysis. \cite{rando2023universal} proposed universal backdoor attacks, embedding hidden functionalities within LLMs, and \cite{baumgartner2024best} investigated data augmentation attacks that inject entirely new preference pairs. 

Despite these attempts, a strong theoretical foundation for understanding the robustness of learning from human preferences against data poisoning attacks remains elusive. Motivated by this, we initiate a theoretical study of data poisoning attacks in learning from human preferences. We seek to analyze these attacks from the attacker’s viewpoint which, in turn, would allow us to identify more robust settings and design effective defenses, ultimately ensuring the security and reliability of preference-based learning in real-world applications.

\looseness-1In particular, we focus on two of the most prominent techniques of learning from human preferences, namely, reinforcement learning from human feedback (RLHF) \citep{stiennon2020learning, ouyang2022training, ziegler2019fine, gao2023scaling, menick2022teaching, glaese2022improving, bai2022training, brown2019learning, shin2023benchmarks} and direct preference optimization (DPO) \citep{DBLP:conf/nips/RafailovSMMEF23}. We consider an attacker that aims to enforce a target policy $\pi^\dagger$ by generating new
preference data $\widehat{D}$. The aim of the attacker is to ensure an RLHF/DPO learner trained on $\widehat{D}$ converges to a policy that is $\epsilon$-close to $\pi^\dagger$. We analyze the number of samples the attacker requires to enforce $\pi^\dagger$ in settings where the attacker can synthesize the entire preference dataset from scratch and where it has to augment a pre-existing dataset $\overline{D}$. Our contributions are summarized below:

\begin{itemize}
\looseness-1\item \textbf{Attack problem formulation:} We propose a general data poisoning formulation in learning from human preferences and instantiate it for two popular paradigms of RLHF and DPO. In particular, we propose an attack problem formulation using the $\ell_1$-norm as a constraint, motivated by previous formulations suited for reward-based RL with deterministic target policies. We consider two types of RLHF learners: unregularized RLHF and regularized RLHF, depending on whether the learner is restricted to remain close to a given reference policy $\mu$ or not. 
\item \looseness-1\textbf{Attacks on RLHF:} We analyze the sample complexity of the attack in unregularized and regularized RLHF settings. In the unregularized setting, our bounds depend on the state-action space cardinality, the attack granularity parameter $\epsilon$, the covariance matrix of the pre-existing data, and its size $\overline{n}$. In the regularized RLHF setting, the dependence on the state-action space size is replaced with dependence on the regularization temperature $\beta$ and the gap between $\pi^\dagger$ and $\mu$. 
\item\textbf{Attacks on DPO:} Moreover, we provide lower and upper bounds on the attack sample complexity in the DPO setting, by showing that feasible regions for the surrogate problem act both as relaxations and restrictions of the original feasible region, for different values of attack granularity parameter. In this setting, the sample complexity depends on the squared norms of the parameters of $\pi^\dagger$ and $\mu$ and the pre-existing dataset size $\overline{n}$.
\item \textbf{Comparison}: Finally, we derive conclusions on the susceptibility of DPO to attack relative to RLHF, both in the data augmentation and data synthesis setting. Our results suggest that, the farther away the target policy $\pi^\dagger$ is from the reference policy $\mu$ in the parameter space, the stronger the tendency of DPO to remain closer to $\mu$ under attacks, relative to RLHF.
\end{itemize}

\section{Preliminaries and Background on Learning from Human Preferences}\label{sec:setting}

In this section, we provide the necessary technical background that will be used throughout the paper. 

\vspace{-1.5mm}
\subsection{Preliminaries}\label{sec:preliminaries}
\vspace{-1mm}

\looseness-1\paragraph{Environment.} Let $\mathcal{M}=\langle \calS,\calA,P,r,\gamma,\rho\rangle$ be an infinite-horizon discounted Markov decision process (MDP), where $\calS$ denotes the state space and $\calA$ denotes the action space, with cardinalities $S$ and $A$, respectively; $P:\calS\times \calA\times \calS \rightarrow [0,1]$ denotes the transition function, where $P(s,a,s')$ denotes the probability of transitioning to state $s'$ when taking action $a$ in state $s$; the reward is denoted by $r:\calS\times \calA\rightarrow\mathbb{R}$ and the discount factor by $\gamma\in[0,1)$. Finally, we let $\rho$ be the initial state distribution. A contextual bandit can be viewed as a special case of this MDP formalism where the state transitions are independent of the actions taken, and the discount factor $\gamma$ is set to $0$.

\looseness-1\paragraph{Policies and value functions.} Stochastic policies are mappings from states to action simplices, $\pi:\calS\rightarrow\Delta(\calA)$, where $\Delta(\calA)$ is the probability simplex with support in $\calA$. A deterministic policy is a mapping $\pi:\calS\rightarrow \calA$, a special case of stochastic policies. Let $\Pi$ and $\Pi^\textnormal{det}$ denote the set of all stochastic and deterministic policies defined over $\calS$ and $\calA$, respectively. For a policy $\pi$ and state-action pair $(s,a)$, we define $d^\pi_{s,a}(s',a') = (1-\gamma)\sum^\infty_{t=1}\gamma^t \mathbb{P}\left( s_t=s',a_t=a'|s_0=s,a_0=a,\pi\right)$, and $d^\pi_{s}(s',a') = (1-\gamma)\sum^\infty_{t=1}\gamma^t \mathbb{P}\left( s_t=s',a_t=a'|s_0=s,\pi\right)$. Furthermore, we define $d^\pi_{s,a}(s')=\sum_{a'}d^\pi_{s,a}(s',a')\pi(a'|s')$, $d^\pi_s(s')=\sum_{a'}d^\pi_{s}(s',a')\pi(a'|s')$, and $d^\pi_\rho(s',a') = \expp_{s\sim\rho} [d^\pi_s(s',a')]$. We focus on \textit{ergodic} MDPs where every state is reachable under any policy and initial distribution. The value function of a policy $\pi\in\Pi$ with respect to a given reward function $r$ is given as $$V^\pi_r(s)=\expp_{\pi,P}\left[ \sum^\infty_{t=0}\gamma^t r(s_t,a_t) | s_0=s \right]~.$$ Similarly, its action-value function is given as $$ Q^\pi_r(s,a)=\expp_{\pi,P}\left[ \sum^\infty_{t=0}\gamma^t r(s_t,a_t)| s_0=s,a_0=a\right]~.$$
We also define $V^\pi_r(\rho)=\expp_{s\sim\rho}[V^\pi_r(s)]$. A policy $\pi^\star\in\Pi$ is said to be optimal in $\Pi$ with respect to $r$ if $V^{\pi^\star}_r(\rho)\geq V^\pi_r(\rho)$, for all $\pi\in\Pi$. Moreover, $\pi^\star$ is said to be $\epsilon$-robust optimal in $\Pi$ with respect to $r$, for a given $\epsilon >0$, if $V^{\pi^\star}_r(\rho)\geq V^\pi_r(\rho)+\epsilon$, for all $\pi\in\Pi\setminus \{\pi^\star\}$. We define the KL-divergence between two policies $\pi$ and $\pi'$ as $$D_\textnormal{KL}(\pi||\pi')=\sum_{s,a}\rho(s)\pi(a|s)\left(\log\pi(s|a) - \log\pi'(a|s)\right)~.$$ When $\pi,\pi'\in\Pi^\textnormal{det}$, we define $D_\textnormal{KL}(\pi||\pi')=0$ iff $\pi=\pi'$, and $\infty$ otherwise. Later in our analysis, we consider the following classes.
\begin{definition}[\textit{Linear rewards}]\label{def:linear_rewards}
    Let $\phi$ be a $d$-dimensional feature mapping $\phi:\calS\times \calA\rightarrow \mathbb{R}^d$ with $\max_{s,a}\norm{\phi(s,a)}_2\leq 1$. We consider the following class of linear reward functions:
    \begin{equation*}
        \resizebox{\hsize}{!}{$\calR = \{r_\omega : r_\omega(s,a)=\omega^\top\phi(s,a),\forall (s,a)\in \calS\times \calA \textnormal{ for } \omega \in \mathbb{R}^d\}~.$}
    \end{equation*}
\end{definition} 
\begin{definition}[\textit{Loglinear policies}]\label{def:loglinear_policies} Let $\psi$ be a $d'$-dimensional feature mapping $\psi:\calS\times \calA\rightarrow \mathbb{R}^{d'}$ with $\max_{s,a}\norm{\psi(s,a)} \leq 1$. We consider the following class of loglinear policies:\vspace{-1mm}
\begin{align*}
    \Pi^\textnormal{log} & = \Bigg\{\pi_\theta : \pi_\theta(a|s)=\frac{\exp\br{\psi(s,a)^\top\theta}}{\sum_{a'}\exp\br{\psi(s,a')^\top\theta}}, \vspace{-2mm}\\ & \hspace{2cm}  \forall (s,a)\in \calS\times \calA \textnormal{ where } \theta\in\mathbb{R}^{d'}\Bigg\}~.
\end{align*}
\end{definition}\vspace{-3mm}
By overloading the notation, for a trajectory $\tau = (s_0, a_0, s_1, \ldots)$, we define $\phi(\tau) = \sum_{t=0}^\infty \gamma^t \phi(s_t, a_t)$ and $\psi(\tau) = \sum_{t=0}^\infty \gamma^t \psi(s_t, a_t)$. Furthermore, for a policy $\pi$ and state-action pair $(s,a)$, we define $\phi^\pi(s,a) = \sum_{s',a'} d^\pi_{s,a}(s',a') \phi(s',a')$ and $\psi^\pi(s,a) = \sum_{s',a'} d^\pi_{s,a}(s',a') \psi(s',a')$.

\subsection{Background on Learning from Human Preferences}\label{sec:learning_from_human_pref}
\vspace{-1mm}

\looseness-1In learning from human preferences, a learner refines behavior by iteratively learning from human feedback. Given a preference dataset $D = \{(\tau, \tau', o)\}$, where $o=1$ indicates that trajectory $\tau$ is preferred over $\tau'$, and $o=-1$ indicates the opposite, the learner $L$ outputs a policy $\pi = L(D, \mu)$ that aligns better with human preferences. Here, $\mu$ represents a reference policy, which could be a pre-trained model. Examples of such learners include agents that use reinforcement learning from human feedback (RLHF) $L_\textnormal{RLHF}$ \citep{ziegler2019fine} or direct preference optimization (DPO) $L_\textnormal{DPO}$ \citep{DBLP:conf/nips/RafailovSMMEF23}. Before introducing these methods, we first define the preference model.
\begin{definition}[\textit{Bradley-Terry preference model} \citep{bradley1952rank}]\label{def:bt_model}
The Bradley-Terry preference model w.r.t. a reward function $r$ is defined as follows: for every tuple $(\tau, \tau', o)$, we have $\mathbb{P}\left( o=1 | \tau,\tau'\right) = \sigma\left(\sum^\infty_{t=0}\gamma^tr(s_t,a_t)-\sum^\infty_{t=0}\gamma^tr(s'_t,a'_t)\right)$, where $\sigma(z)=1/(1+\exp(-z))$.
\end{definition}

\paragraph{Reinforcement learning from human feedback.} With access to the dataset $D$ and reference policy $\mu$, RLHF~\citep{ziegler2019fine} proceeds in two phases. In the first phase, a reward function is learned from $D$ using maximum likelihood estimation (MLE) based on the Bradley-Terry preference model. This involves solving the following regularized MLE problem (with regularization parameter $\lambda > 0$):
\begin{align}\label{eq:reward_mle}
    & \min_\omega \; \ell^\omega_\textnormal{RLHF}(D):= -\sum_{(\tau,\tau',o)\in D}\log \nonumber\\ & \sigma\br{o\cdot\sum_{t\geq 0}\gamma^t\br{r_\omega(s_t,a_t)-r_\omega(s'_t,a'_t)}} + \frac{\lambda}{2}\norm{\omega}^2~. \tag{P:RLHF.Reward}
\end{align}
Let $\widehat{\omega}$ denote the solution of Problem~\eqref{eq:reward_mle}. In the second phase, RLHF solves the following regularized policy optimization based on the learned reward $r_{\widehat{\omega}}$ (with regularization parameter $\beta >0$) to obtain optimal solution ${\pi}^\textnormal{reg}_{r_{\widehat{\omega}}}$ in:
\begin{align}\label{eq:regularized_objective}
    \arg\max_{\pi\in\Pi}  \expectover{\pi,P}{\sum_{t\geq 0}\gamma^t\br{ r_{\widehat{\omega}}(s_t,a_t)-\beta\log\frac{\pi(a_t|s_t)}{\mu(a_t|s_t)}}}~. \tag{P:RLHF.Policy}
\end{align} 
We denote by $\calV^\pi_{r_{\widehat{\omega}}}(\rho)$ the regularized objective above.
\vspace{-2mm}
\paragraph{Direct preference optimization.} DPO~\citep{DBLP:conf/nips/RafailovSMMEF23} leverages the relationship between a reward function $r$ and its corresponding regularized optimal policy ${\pi}^\textnormal{reg}_r$ to bypass the reward-learning phase, and directly optimize the policy. However, DPO is currently limited to contextual bandit settings, where we have $\tau = (s, a)$ and $\tau' = (s, a')$. Given $D$ and $\mu$, DPO solves the following optimization problem (with regularization parameter $\lambda > 0$) to obtain optimal solution $\widehat{\theta}$ as:
\begin{align}\label{eq:dpo_loss}
    & \min_\theta \; \ell^\theta_\textnormal{DPO}\left( D\right) :=- \sum_{(\tau,\tau',o)\in D} \log  \nonumber \\ & \sigma \br{o\cdot\br{\beta \log \frac{\pi_\theta(a|s)}{\mu(a|s)} - \beta\log\frac{\pi_\theta(a'|s)}{\mu(a'|s)}}} + \frac{\lambda}{2}\norm{\theta-\theta_\mu}^2\tag{P:DPO}.
\end{align}

\vspace{-5mm}
\section{Preference Poisoning Attack Setup}
\label{sec:pref-poisoning-attack}
\vspace{-1mm}

In this section, we formulate the problem of data poisoning attacks on learning from human preferences. We consider an attacker aiming to impose a target policy $\pi^\dagger$ onto a learner. To achieve this, the attacker modifies the preference dataset so that the learner trained on the altered data produces a policy close to $\pi^\dagger$. Ideally, the attacker should make minimal modifications to achieve this outcome. 

\paragraph{Attack goal and poisoned data.} Let $\overline{D}$ represent the clean preference dataset originating from the environment $\mathcal{M}$, and let $\mu$ denote the reference policy. The learner, denoted by $L$, is trained on the preference dataset. The attacker is modeled as a mapping $\mathcal{P}$, which takes the clean dataset $\overline{D}$, environment $\mathcal{M}$, reference policy $\mu$, and learner $L$ as inputs, and outputs a poisoned dataset $\widehat{D} = \mathcal{P}(\overline{D}, \mu, \mathcal{M}, L)$. We assume the attacker has full knowledge of the environment. Given a margin parameter $\epsilon$, the attacker poisons the dataset $\overline{D}$ so that the learner $L$ trained on the poisoned dataset $\widehat{D}$ converges to a policy within an $\epsilon$ distance of $\pi^\dagger$, i.e., $\norm{\pi^\dagger-\pi^L}_1 \leq \epsilon$, where $\pi^L = L(\widehat{D}, \mu)$ and $\norm{\pi-\pi'}_1=\sum_{s,a}\rho(s)\abs{\pi(a|s)-\pi'(a|a)}$ is the $\ell_1$ norm between given policies $\pi$ and $\pi'$.

\paragraph{Attack cost and formulation.} We consider a setting where the attacker can \textit{augment} a pre-existing preference dataset $\cleandata$. Moreover, we instantiate this setting to the case when $\cleandata$ is empty.
Since adding samples incurs a cost, the attacker aims to enforce $\pi^\dagger$ by minimally adding additional samples. The attacker's optimization problem can be formalized as
\begin{align}\label{op:original-attack-problem}
    \min_D & |D| \;\; \textnormal{such that}\;\; \norm{\pi^\dagger - \pi^L}_1\leq\epsilon~,  \nonumber\\ &  \textnormal{where}\;\; \pi^L = L(\cleandata \cup D, \mu)~.\tag{P:Attack}
\end{align}

\paragraph{Attack feasibility and synthesis.} We consider poisoning attacks on two learning paradigms in learning from human preferences, RLHF and DPO, as introduced in Section~\ref{sec:learning_from_human_pref}. We identify the specific conditions within both learning paradigms that make such attacks feasible:\vspace{-3mm}
\begin{itemize}
\item We consider an attacker that has the synthesis capability to find a trajectory pair $(\tau, \tau')$ for any given $z \in \mathbb{R}^d$ such that $\phi(\tau) - \phi(\tau') = z$ or $\psi(\tau)-\psi(\tau')=z$. In practice, this assumption translates to the following requirement. Given a prompt-response pair $(x,y):=\tau$ and a vector $z\in\mathbb{R}^d$, the attacker can find a response $y'$ such that the difference between $(x,y)$ and $(x,y'):=\tau'$ in the embedding space is approximately equal to $z$, i.e., $\phi(\tau)-\phi(\tau')\approx z$ or $\psi(\tau)-\psi(\tau')\approx z$.\vspace{-2mm}
\item For an unregularized RLHF attack to be feasible, we must have $\pi^\dagger\in \Pi^\textnormal{det}$. For a regularized RLHF and DPO attack to be feasible, we must have $\pi^\dagger,\mu\in\Pi^\textnormal{log}$. 
\end{itemize}\vspace{-3mm}

\looseness-1\paragraph{Notation.} Next we introduce some notation that will be useful in the following sections. As usual, $[n]=\{1,2,\ldots,n\}$ denotes the set of first $n$ natural numbers. $\langle v,z \rangle = v^\top z$ denotes the inner product of two compatible vectors $v$ and $z$. Further, we denote by $\Phi$ and $\Psi$ the matrices with columns $\phi(s,a)^\top$ and $\psi(s,a)^\top$, respectively, for every $(s,a)\in\mathcal{S}\times\mathcal{A}$. We assume throughout that $\Phi$ and $\Psi$ are full rank.
Unless otherwise specified, $\norm{v}$ denotes the Euclidean norm, $\norm{M}$ denotes the spectral norm for matrix $M$, and $M^+$ denotes its pseudoinverse. $\mathbf{0}$ and $\mathbf{1}$ denote the vectors of zeroes and ones, respectively, and $I$ denotes the identity matrix. We denote by $\sigma_{\max}(M)$ and $\sigma_{\min}(M)$ the maximum and minimum eigenvalues (or singular values) of a given square (rectangular) matrix $M$, respectively. Let $\xi_{\max}:=\max_x x/(1+\exp(x))$ and denote by $x^\star=\arg\max_x x/(1+\exp(x))$. Let $\xi^{-1}(a)$ denote the solution to $a=x/(1+\exp(x))$, for any $a< \xi_{\max}$, where the domain of $x/(1+\exp(x))$ is $(-\infty,x^\star]$. Moreover, we define $\xi_1(a)=\xi^{-1}(a\ceil{a/\xi_{\max}}^{-1})$ and $\xi_2(a)=\xi^{-1}(2a\ceil{a/(2\xi_{\max})}^{-1})$.

\section{Poisoning Attacks on RLHF}\label{sec:attack-to-rlhf}

\begin{figure*}
    \centering
    \includegraphics[scale=0.35]{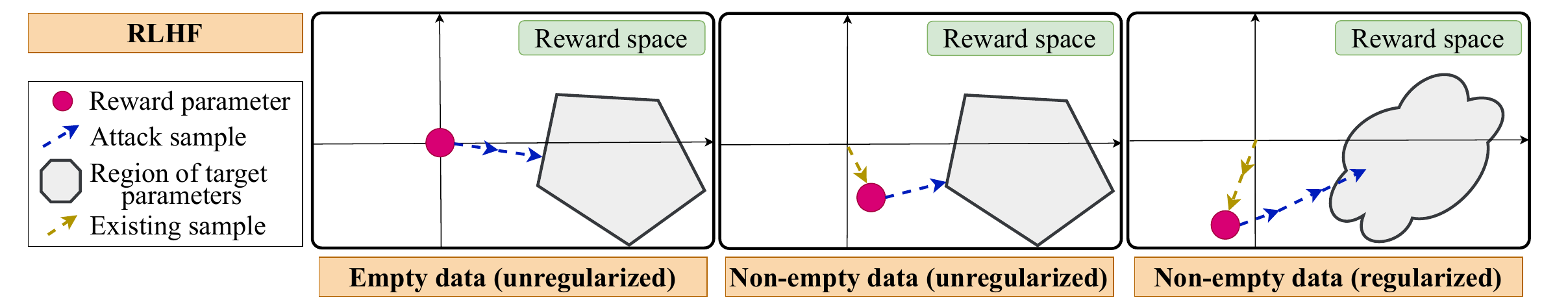}
    \caption{\small A geometric illustration of our attack model for RLHF. The shaded regions represent the reward parameter spaces where optimal policies are $\epsilon$-close to $\pi^\dagger$. The blue arrows represent attack samples, while the yellow arrows represent the pre-existing data samples from $\overline{D}$. Finally, the red shape represents the optimal reward parameters with respect to the generated dataset $\widehat{D}$. Each added attack sample moves the optimal parameter closer to the shaded region. For unregularized RLHF with empty $\cleandata$ (left), the attack problem is solved in the reward parameter space, and the target space is a polytope. For unregularized RLHF with non-empty $\cleandata$ (middle), the required samples depend on the alignment of $\pi^\dagger$ with $\cleandata$. For regularized RLHF (right), since the optimal policy is not necessarily deterministic, the geometry of the target space becomes non-linear.}
    \label{fig:rlhf-diagram}
\end{figure*}

In this section, we study data poisoning attacks on RLHF, where the attacker synthesizes $\widehat{D}$ from $\cleandata$. We start our discussion by formulating the general attack problem for RLHF. We do this by instantiating Problem \ref{op:original-attack-problem} for this setting as:
\begin{align}\label{op:rlhf-attack}
    & \min_{D}\; |D|\;\; \textnormal{such that}\;\; \widehat{\omega}=\arg\min_\omega\ell^\omega_\textnormal{RLHF}\left(\cleandata\cup D\right) \nonumber\\ & \quad \textnormal{and} \;\; \norm{\pi^\dagger -\pi^\textnormal{reg}_{r_{\widehat{\omega}}}}^2_1 \leq \epsilon~.\tag{P:Attack:RLHF.1}
\end{align}
In Appendix \ref{sec:rlhf_feasibility}, we show that Problem \ref{op:rlhf-attack} is feasible whenever $\pi^\dagger\in\Pi^\textnormal{det}$ or $\pi^\dagger,\mu\in\Pi^\textnormal{log}$. We study the following scenarios: (i) the general unregularized RLHF ($\beta=0$) setting with data augmentation ($\cleandata\neq\emptyset$) and data synthesis ($\cleandata=\emptyset$), for deterministic $\pi^\dagger$; (ii) the regularized RLHF setting ($\beta>0$) with data augmentation and data generation for general policies $\pi^\dagger$.

\subsection{Unregularized RLHF} 

For $\beta=0$, Problem~\ref{eq:regularized_objective} reduces to a standard value maximization problem, for which optimal policies $\pi^L$ are known to be deterministic. Thus, if we are given a pre-existing dataset $\overline{D}$ and can augment it into $\widehat{D}$ such that, when used to solve Problem~\ref{eq:reward_mle}, yields a reward function making $\pi^\dagger$ $\epsilon'$-robust optimal, for some $\epsilon'>0$, then $\pi^L$ is guaranteed to converge to $\pi^\dagger$ as the unique optimal policy.
\cite{DBLP:journals/jmlr/RakhshaRDZS21} show that checking the $\epsilon'$-robust optimality condition for neighboring policies $\pi^\dagger\{s,a\}$ of $\pi^\dagger$ for all $(s,a)$ pairs is sufficient. Here, $\pi^\dagger\{s,a\}(s')=\pi^\dagger(s)$, if $s\neq s'$, and $\pi^\dagger\{s,a\}(s')=a$, otherwise.
In the linear reward setting, these constraints are represented compactly by the polytope $M_{\pi^\dagger}^\top\omega \geq \boldsymbol{\epsilon}'$ (see Figure~\ref{fig:rlhf-diagram} for a geometric illustration), where $M_{\pi^\dagger}$ is a $d\times S(A-1)$-dimensional matrix.\footnote{Note that, since the vector $\boldsymbol{\epsilon}'$ has all entries $\epsilon$, the inequality condition can be satisfied even when the rank of $M_{\pi^\dagger}$ is less than $SA$. The minimal requirement is that all inequalities of the system are consistent, i.e., they yield intersecting regions.} Its columns are defined as $\sum_{s'}\phi(s',\pi^\dagger(s')) - \phi(s',\pi^\dagger\{s,a\}(s'))$, for all $(s,a)$, and $\boldsymbol{\epsilon}'$ is the $S(A-1)$-dimensional vector with entries $\epsilon'$. The columns of matrix $M_{\pi^\dagger}$ represent the differences in the average feature distribution between $\pi^\dagger$ and its neighbors. 

Leveraging this polytope constraint, we instantiate the general preference poisoning attack problem~\ref{op:rlhf-attack} for the RLHF paradigm in the setting where the clean preference dataset $\cleandata$ is non-empty (with size $\abs{\cleandata} = \overline{n}$), as follows:
\begin{align}\label{op:syn-rlhf-augment}
    & \min_{D}\; |D|\;\;
    \textnormal{such that}\;\; \widehat{\omega} = \arg\min_\omega \ell^\omega_\textnormal{RLHF}\left(\cleandata\cup D\right)\nonumber\\ & \quad  \textnormal{and}\;\; M_{\pi^\dagger}^\top\widehat{\omega} \geq \boldsymbol{\epsilon}'~.\tag{P:Attack:RLHF.2}
\end{align}
Note that the above is a surrogate of Problem \ref{op:rlhf-attack}, since, in unregularized RLHF $\pi^\textnormal{reg}_{r_{\widehat{\omega}}}$ is the optimal policy with respect to $r_{\widehat{\omega}}$ which, as we explained above, is enforced by the constraint $M_{\pi^\dagger}^\top\widehat{\omega} \geq \boldsymbol{\epsilon}'$.  The solution to the above data augmentation attack problem depends on data-dependent quantities due to the pre-existing data. Let us define the covariance matrix with respect to $\cleandata$ as $$\Sigma^\phi_{\cleandata}=(1/\overline{n})\sum_{(\tau,\tau')\in\overline{D}}(\phi(\tau)-\phi(\tau'))(\phi(\tau)-\phi(\tau'))^\top~.$$
We are now ready to state our first result.
\begin{restatable}{retheorem}{rlhfsynthesisabs}\label{thm:warmup-rlhf-augment} Let $\cleandata$ be a given preference dataset of $\overline{n}$ samples, let $\beta =0$, $\epsilon'>0$ and $\pi^\dagger\in\Pi^\textnormal{det}$. Furthermore, let $\overline{\omega}$ be optimal for $\ell^\omega_\textnormal{RLHF}(\cleandata)$, define $\omega^\dagger$ as
\begin{align*}
    \proj_{\omega:M_{\pi^\dagger}\omega\geq\boldsymbol{\epsilon}'}\left(\overline{\omega}\right) = \overline{\omega} + M_{\pi^\dagger}\left(M_{\pi^\dagger}^\top M_{\pi^\dagger}\right)^+\left(\boldsymbol{\epsilon}' - M_{\pi^\dagger}\overline{\omega}\right)
\end{align*}
and let $\gamma\geq 1-2\norm{\omega^\dagger}/(\xi_{\max}+1)$. Then, the dataset of $\ceil{\abs{(\omega^\dagger)^\top\nabla_\omega\ell^{\omega^\dagger}_\textnormal{RLHF}(\cleandata)}/\xi_{\max}}$ identical samples satisfying $ o=1$ and
\begin{align*}
    \phi(\tau)-\phi(\tau')=\xi_1\left( \left(\omega^\dagger\right)^\top\nabla_\omega \ell^{\omega^\dagger}_\textnormal{RLHF}(\cleandata)\right)\frac{\omega^\dagger}{\norm{\omega^\dagger}^2},
\end{align*}
is a feasible solution to Problem \ref{op:rlhf-attack}. Furthermore, there exists an optimal solution $\widehat{D}$ for Problem \ref{op:rlhf-attack} with $\widehat{n}_\textnormal{RLHF}$ identical samples such that 
\begin{align*}
     \widehat{n}_\textnormal{RLHF} \leq \left\lceil\frac{2\overline{n}+\lambda}{\xi_{\max}} \left(\frac{(\epsilon')^2SA}{\sigma^2_{\min}\left(M_{\pi^\dagger}\right)} +\norm{\overline{\omega}}  \frac{\epsilon'\sqrt{SA}}{\sigma_{\min}(M_{\pi^\dagger})}\right)\right\rceil~.
\end{align*}
\end{restatable}
\textit{Sketch of proof.} We start by considering the reward learning subproblem for the surrogate problem \ref{op:syn-rlhf-augment}. We utilize the solution of the problem of machine teaching to logistic regression learners \citep{DBLP:journals/jmlr/LiuZ16} and relate it to our reward subproblem. Next, we proceed to solving the modified optimization problem with respect to the final constraint. In our setting, the samples from $\cleandata$ are fixed, and thus cannot be treated as variable. Therefore, we need to design a new attack dataset $\widehat{D}$ with samples depending on the gradient of the loss with respect to $\cleandata$, which captures how aligned $\cleandata$ is with $\pi^\dagger$. In the best case, the optimal parameter with respect to $\cleandata$ is already in the target polytope, meaning that the number of samples in this case is $0$.  In order to obtain closed-form bounds, we construct a feasible solution using the projection of the optimal solution with respect to $\cleandata$ onto the polytope and make use of strong convexity and Lipschitzness of $\ell^\omega_\textnormal{RLHF}(\cleandata)$ to get our results. We also show in Appendix \ref{appendix:technical_lemmas} that $M_{\pi^\dagger}$ is full rank under mild assumptions. This guarantees that our bounds are finite.\qed

Before going to the next section, we instantiate Problem \ref{op:syn-rlhf-augment} to the case when the pre-existing dataset is empty. The following result is a corollary of Theorem \ref{thm:warmup-rlhf-augment}.
\begin{restatable}{recorollary}{warmupresultrlhf}\label{thm:warm_up_rlhf}
    Let $\cleandata=\emptyset$, $\beta =0$ and $\epsilon'>0$, and let $\pi^\dagger\in\Pi^\textnormal{det}$. Define $
        \omega^\dagger = M_{\pi^\dagger}\left(M_{\pi^\dagger}^\top M_{\pi^\dagger}\right)^+\boldsymbol{\epsilon}'$.
    Then, the dataset of $ \ceil{\frac{\lambda\norm{\omega^\dagger}^2}{\xi_{\max}}}$
    identical samples satisfying 
    \begin{align*}
        \phi(\tau)-\phi(\tau') = \xi_1\left(\lambda\norm{\omega^\dagger}\right)\cdot\frac{\omega^\dagger}{\norm{\omega^\dagger}^2},\;\;\; o = 1
    \end{align*}
    is a feasible solution for Problem \ref{op:rlhf-attack}.
    Furthermore, there exists an optimal solution $\widehat{D}$ for Problem \ref{op:rlhf-attack} with $\widehat{n}_\textnormal{RLHF}$ samples such that
    \begin{align*}
        \widehat{n}_\textnormal{RLHF} \leq  \ceil{\frac{(\epsilon')^2\lambda SA}{\xi_{\max}\sigma^2_{\min}\br{M_{\pi^\dagger}}}}~.
    \end{align*}
\end{restatable}

\subsection{Regularized RLHF.} 

Now, we consider Problem \ref{op:rlhf-attack} under a general setting, where the regularization parameter $\beta > 0$ and the preference dataset $\cleandata \neq \emptyset$. In this setting, we are dealing with loglinear policies (see Section \ref{sec:pref-poisoning-attack}) which, since they are stochastic, lead to infinitely many constraints -- hence, it is challenging to apply the polytope constraint idea from the above unregularized RLHF setting. Therefore, we consider a surrogate problem with constraints that are suited to stochastic policies. We will make use of the KL divergence for that purpose. We write the surrogate attack problem as
\begin{align}\label{op:reg-rlhf-attack}
    & \min_{D}\; |D|\;\; \textnormal{such that}\;\; \widehat{\omega}=\arg\min_\omega\ell^\omega_\textnormal{RLHF}\left(\cleandata\cup D\right) \nonumber\\ & \quad \textnormal{and} \;\; \KL\left(\pi^\dagger||\pi^\textnormal{reg}_{r_{\widehat{\omega}}}\right) \leq \epsilon'~,\tag{P:Attack:RLHF.3}
\end{align}
where $\pi^\textnormal{reg}_{r_{\widehat{\omega}}}$ is the optimal policy for the regularized objective \ref{eq:regularized_objective}. We obtain the following bounds for this setting:
\begin{restatable}{retheorem}{synrlhfwarmupreg}\label{thm:rlhf-reg-gen}
    Let $\cleandata$ be a given preference dataset of $\overline{n}$ samples, $\beta >0$ and $0<\epsilon'\leq\epsilon$. Moreover, define
    \begin{align*}
        \Gamma^\omega_\phi(\pi^\dagger||\pi^\textnormal{reg}_{r_\omega}) = \sum_{s,a}\rho(s)(\pi^\dagger(a|s)-\pi^\textnormal{reg}_{r_\omega}(a|s))\phi^{\pi^\textnormal{reg}_{r_\omega}}(s,a)~,
    \end{align*}
    for any given $\omega$. 
    Then, there exists a feasible solution $\widehat{D}$ for Problem \ref{op:rlhf-attack} with $\widehat{n}_\textnormal{RLHF}$ samples which yields $\omega^\dagger$  when solving Problem \ref{eq:reward_mle} on dataset $\widehat{D}$, such that $\Gamma^\omega_\phi(\pi^\dagger||\pi^\textnormal{reg}_{r_{\omega^\dagger}})\neq 0$ and
    \begin{align*}
        \widehat{n}_\textnormal{RLHF} & \leq O\left(\frac{\beta^2\left(D_\textnormal{KL}\left(\pi^\dagger||\mu\right)-\epsilon'\right)^2}{(1-\gamma)^2\sigma^2_{\min}(\Sigma^\phi_{\cleandata})\norm{\Gamma^{\omega^\dagger}_\phi(\pi^\dagger||\pi^\textnormal{reg}_{r_{\omega^\dagger}})}^2}\right. \\ & \quad\quad + \left.\frac{\overline{n}}{(1-\gamma)^4\sigma^4_{\min}(\Sigma^\phi_{\cleandata})}\right)~.
    \end{align*}
\end{restatable}
\textit{Sketch of proof.} We start by showing that solving the surrogate subproblem is enough to obtain upper bounds on the sample size for suitable $\epsilon'$. Then, we  solve the reward learning subproblem and obtain a dataset $\widehat{D}$ with $\widehat{n}_\textnormal{RLHF}$ identical samples that satisfy $\phi(\tau)-\phi(\tau') = \xi^{-1}\left(\lambda\left|\nabla_\omega \ell^\omega_\textnormal{RLHF}(\overline{D})^\top\omega\right| / \widehat{n}_\textnormal{RLHF} \right)\omega/\norm{\omega}^2$ with $o = 1$, where $\widehat{n}_\textnormal{RLHF}$ is given in terms of $\left|\nabla_\omega \ell^\omega_\textnormal{RLHF}(\overline{D})^\top\omega\right|$. Then, we solve the equivalent problem of minimizing $\left|\nabla_\omega \ell^\omega_\textnormal{RLHF}(\overline{D})^\top\omega\right|$, subject to $\omega$ yielding a regularized optimal policy $\pi^\textnormal{reg}_{r_{\omega}}$ that satisfies $D_\textnormal{KL}(\pi^\dagger||\pi^\textnormal{reg}_{r_{\omega}})\leq \epsilon$. 

Using $\omega^\dagger$ as a feasible solution in the expression, we finalize the bounds. 

\begin{remark}
    Note that the term $\Gamma^{\omega^\dagger}_\phi(\pi^\dagger||\pi^\textnormal{reg}_{r_{\omega^\dagger}})$ represents the average trajectory feature difference between $\pi^\dagger$ and $\pi^\textnormal{reg}_{r_\omega}$ when rolling out trajectories using $\pi^\textnormal{reg}_{r_{\omega^\dagger}}$. If this term is uniformly $\boldsymbol{0}$ at the problem solution $\omega^\dagger$, then this means that, either $\pi^\dagger=\pi^\textnormal{reg}_{r_{\omega^\dagger}}$, or at least the two policies are identical in the average trajectory feature space. Therefore, less samples are needed to satisfy our objective. The following result provides bounds that depend on such a solution $\omega^\dagger$. Its proof follows immediately from the proof of Theorem \ref{thm:rlhf-reg-gen}.
\end{remark}
\begin{restatable}{recorollary}{rlhfcorollary1}\label{cor:rlhf-cor-1}
    There exists $\omega^\dagger$ such that $\pi^\dagger=\pi^\textnormal{reg}_{r_{\omega^\dagger}}$. Moreover, we have that $\omega^\dagger$ is a feasible solution for Problem \ref{op:rlhf-attack} and 
    \begin{align*}
        \widehat{n}_\textnormal{RLHF} \leq \ceil{\frac{1}{\xi_{\max}}\left(\lambda\norm{\omega^\dagger}^2 + \norm{\omega^\dagger}\cdot\frac{2\overline{n}}{(1-\gamma)^2}\right)}~.
    \end{align*}
\end{restatable}

Finally, we instantiate the above result in the case when the pre-existing data is empty. 
\begin{restatable}{recorollary}{synrlhfwarmupregcor}
    Let $\cleandata=\emptyset$, $\beta >0$ and $0<\epsilon'\leq\epsilon$. There exists a feasible solution $\widehat{D}$ for Problem \ref{op:rlhf-attack} with $\widehat{n}_\textnormal{RLHF}$ samples such that 
    \begin{align*}
        \widehat{n}_\textnormal{RLHF} \leq \ceil{ \frac{\lambda\norm{\omega^\dagger}^2}{\xi_{\max}}}~,
    \end{align*}
    where $\omega^\dagger$ is defined as in Corollary \ref{cor:rlhf-cor-1}.
\end{restatable}
\begin{remark}
    Note that, when $\cleandata=\emptyset$, the number of samples required for an efficient attack on regularized RLHF depends on the norm of the reward parameter that makes $\pi^\dagger$ nearly-optimal.
\end{remark}

\section{Poisoning Attacks on DPO}\label{sec:dpo-attack}

As we have mentioned, the DPO objective is formulated only for the contextual bandit setting. Recall that trajectories here are defined in terms of context-action pairs $\tau=(s,a)$. We instantiate the general poisoning attack problem~\ref{op:original-attack-problem} for the DPO paradigm as follows:\vspace{1cm}
\begin{align}\label{op:attack-dpo}
    & \min_{D}\;  |D|\;\; \textnormal{such that}\;\; \widehat{\theta}=\arg\min_\theta \ell^\theta_\textnormal{DPO}(D\cup\overline{D})\nonumber\\ & \quad \textnormal{and}\;\;\norm{\pi^\dagger - \pi_{\widehat{\theta}}}^2_1 \leq \epsilon~.\tag{P:Attack:DPO.1}
\end{align}
\begin{figure*}
    \centering
    \includegraphics[scale=0.35]{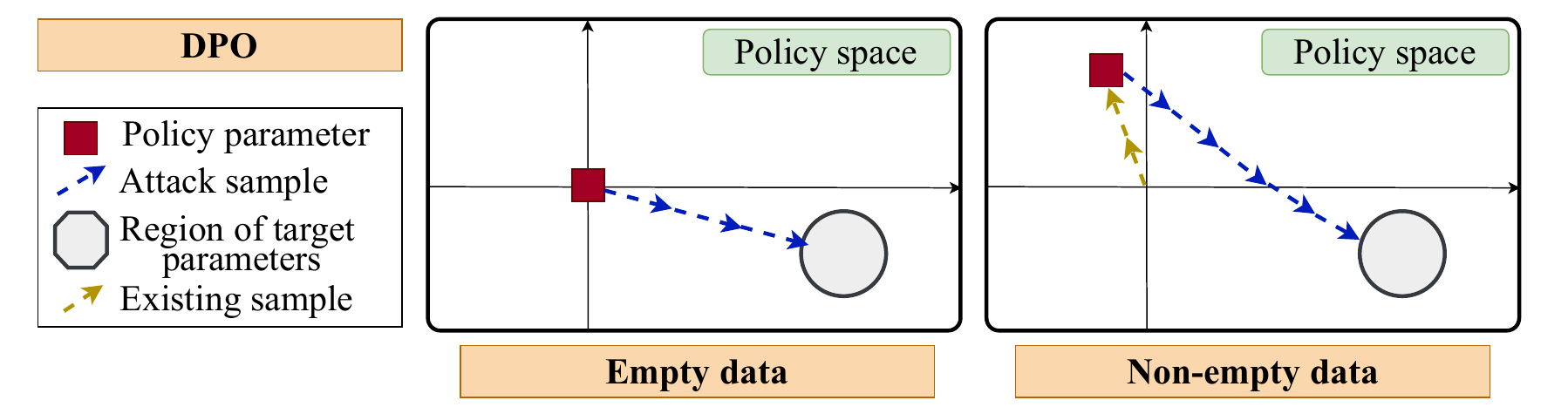}
    \caption{\small A geometric illustration of our attack model for DPO. Here, the distinction between empty and non-empty $\cleandata$ is similar to Figure~\ref{fig:rlhf-diagram}. In contrast to the RLHF setting, here the attacker operates directly in the policy parameter space and the target feasible region is a ball centered around $\theta^\dagger$ with radius $\epsilon$ as outlined in the formulation of Problem~\ref{op:dpo-aug}.}
    \label{fig:dpo-diagram}
\end{figure*}\\
In general, obtaining an intuitive form of attack construction for the above problem, as presented in Theorem~\ref{thm:warm_up_rlhf}, is challenging. To facilitate a more intuitive and efficient attack construction, we consider the following surrogate problem: 
\begin{align}\label{op:dpo-aug}
    & \min_{D}\; |D|\;\;\textnormal{such that}\;\; \widehat{\theta} =\arg\min_\theta\ell^\theta_\textnormal{DPO}\left(\cleandata\cup D\right)\nonumber\\ & \quad \textnormal{and}\;\; \norm{\widehat{\theta}-\theta^\dagger}^2\leq \epsilon'~,\tag{P:Attack.DPO.2}
\end{align}
where $\theta^\dagger \in \mathbb{R}^{d'}$ such that $\pi^\dagger = \pi_{\theta^\dagger}$ (see Section~\ref{sec:pref-poisoning-attack}). Specifically, we have replaced the $\ell_1$-norm-based constraint with an $\ell_2$-norm-based constraint on the policy parameter space. The attack is successful if the learner's policy parameter converges closer to the target parameter $\pi_{\theta^\dagger}$. We derive the following result for this setting:
\begin{restatable}{retheorem}{dpoaugment}\label{thm:dpo-augment}
    Let $\cleandata$ be a given preference dataset of $\overline{n}$ samples, let $\beta >0$ and $0<\epsilon'\leq\epsilon/2$. Furthermore, let $\overline{\theta}$ be the optimal point for $\ell^\theta_\textnormal{DPO}(\cleandata)$ and define
    \begin{align*}
        \widetilde{\theta} = \proj_{\theta:\norm{\theta-\theta^\dagger}\leq\epsilon'}\left(\overline{\theta}\right) = \theta^\dagger + \frac{(\epsilon')^2}{\norm{\overline{\theta}-\theta^\dagger}^2}\left(\overline{\theta}-\theta^\dagger\right)~.
    \end{align*}
    Then, the dataset $\widehat{D}$ containing $$2\ceil{|(\nabla_\theta\ell^{\widetilde{\theta}}_\textnormal{DPO}(\cleandata))^\top(\widetilde{\theta}-\theta_\mu)|/(2\xi_{\max})}$$ identical samples satisfying
    \begin{align*}
        & \beta\br{\widetilde{\theta}-\theta_\mu}^\top\br{\psi (s,a)-\psi(s,a')} \\ & \quad\quad= o \cdot \xi_2\left(\nabla_\theta\ell^{\widetilde{\theta}}_\textnormal{DPO}(\cleandata))^\top\left(\widetilde{\theta}-\theta_\mu\right)\right),
    \end{align*}
    with $o=1$ for half the samples and $o=-1$ for the remaining, is a feasible solution to Problem \ref{op:attack-dpo}. Furthermore, there exists an optimal solution $\widehat{D}$ to Problem \ref{op:attack-dpo} with $\widehat{n}_\textnormal{DPO}$ identical samples such that  
    \begin{align*}
        & \widehat{n}_\textnormal{DPO}\leq 2\cdot \\ & \ceil{(\overline{n}\beta+\lambda)\frac{\abs{\norm{\overline{\theta}-\theta^\dagger}^2-(\epsilon')^2}}{2\xi_{\max}\norm{\theta^\dagger-\overline{\theta}}}\left( 3\norm{\theta^\dagger} + \norm{\theta_\mu} + \sqrt{\epsilon'}\right)}.
    \end{align*}
\end{restatable}
\textit{Sketch of proof.} First, we show that, for suitable $\epsilon'$, any feasible solution to the surrogate problem is feasible for the original problem. The first main challenge in this setting for constructing $\widehat{D}$ is the presence of both $\cleandata$ and $\theta_\mu$ in the DPO objective. We first show that, given parameter $\theta$, the sample size that makes $\pi_\theta$ optimal for $\ell^\theta_\textnormal{DPO}(\cleandata\cup D)$ is a factor of $\abs{\nabla_\theta\ell^\theta_\textnormal{DPO}(\cleandata)^\top(\theta-\theta_\mu)}$. Using this, we redefine the objective of our problem and use the $\ell_2$-ball centered at $\theta^\dagger$ with radius $\sqrt{\epsilon'}$ as constraint. The second main challenge consists of dealing with inverses of sums of matrices, due to the effect of the pre-existing data, for which we use the Woodbury inversion formula, and then proceed to solve a quadratic equation to obtain a fixed-point solution of our problem. Using the properties of that solution, we obtain the bounds on the norm of the optimal parameter $\widetilde{\theta}$. To obtain bounds on the norm of the gradient, we utilize Lipschitzness of $\ell^\theta_\textnormal{DPO}(\cleandata)$, and the geometrical relationship between $\overline{\theta}$ and $\theta^\dagger$. \qed

Next, we establish a lower bound on the attack sample complexity for the DPO setting. 
\begin{restatable}{retheorem}{dpolowerbound}\label{thm:dpo-lower-bound}
    Let $\cleandata$ be a given preference dataset of $\overline{n}$ samples and let $\beta >0$. Then, there exists $\eta_{\min}>0$, such that for any $\epsilon'\geq\epsilon/\eta_{\min}$, we have
    \begin{align*}
        \widehat{n}_\textnormal{DPO} \geq 2\ceil{\frac{\lambda}{2\xi_{\max}}\left(\norm{\theta^\dagger-\theta_\mu}-\sqrt{\epsilon'}\right)^2} -\overline{n}~.
    \end{align*}
\end{restatable}
\textit{Sketch of proof}. Let $\Theta_1 =\{\theta:\norm{\pi_\theta-\pi_{\theta^\dagger}}_1^2\leq\epsilon$ and $\Theta_2 = \{ \theta: \norm{\theta-\theta^\dagger}^2\leq\epsilon'\} $. First, we prove that there exists a positive constant $\eta_{\min}$ such that $\Theta_1\subseteq\Theta_2$, for any $\epsilon'\geq\epsilon/\eta_{\min}$. This means that the feasible region of the surrogate problem is larger, which implies that a lower bound on the solution of Problem \ref{op:dpo-aug} is also a lower bound on the solution of Problem \ref{op:attack-dpo}. Next, we focus on Problem \ref{op:dpo-aug} and show that we can reduce it to a convex program. Using KKT conditions, we obtain an exact solution to the problem. We then use this solution as a lower bound for Problem \ref{op:attack-dpo}.\qed

\begin{remark}
    Note that the upper bounds from Theorem \ref{thm:dpo-augment} hold for different value of $\epsilon'$ than the one required for the lower bounds of Theorem \ref{thm:dpo-lower-bound}. This is because we are essentially tuning the radius of the feasible region of Problem \ref{op:dpo-aug} so that it is either contained in the feasible region of Problem \ref{op:attack-dpo} (which is what we need for the upper bounds), or it contains the feasible region of  Problem \ref{op:attack-dpo} (which is what we need for lower bounds). 
\end{remark}

When there is no pre-existing data, the attacker can synthesise any poisoning dataset from scratch. An immediate solution to this problem is the instantiation of Theorem \ref{thm:dpo-augment} when $\cleandata=\emptyset$. However, below we also provide tight upper bounds that match the lower bounds of Theorem \ref{thm:dpo-lower-bound} for the empty data setting.

\begin{restatable}{retheorem}{dpogenerationdagger}\label{thm:dpo-gen-dagger}
    Let $\cleandata=\emptyset$, let $\beta >0$ and $0<\epsilon'\leq\epsilon/2$. Furthermore, let $\pi^\dagger, \mu\in\Pi^\textnormal{log}$ be loglinear with parameters $\theta^\dagger$ and $\theta_\mu$, respectively. Define $$\widetilde{\theta} = \theta^\dagger + e\sqrt{\epsilon'}(\theta_\mu - 2\theta^\dagger)/\norm{\theta_\mu - 2\theta^\dagger},$$ where $e=1$, if ${\theta^\dagger}^\top\left(\theta^\dagger-\theta_\mu\right)\geq \sqrt{\epsilon'}-\epsilon'$, and $e=-1$, otherwise. Then, the dataset of $ 2\ceil{\frac{\lambda \abs{\lambda\widetilde{\theta}^\top\br{\widetilde{\theta}-\theta_\mu}}}{2\xi_{\max}}}$ samples satisfying 
    \begin{align*}
        \beta\br{\widetilde{\theta}-\theta_\mu}^\top\br{\psi (s,a)-\psi(s,a')} = o \cdot \xi_2\left( \lambda\norm{\widetilde{\theta}-\theta_\mu}^2\right)
    \end{align*}
    with $o = 1$ for half of the samples, and $o = -1$ for the remaining is a feasible solution to Problem \ref{op:attack-dpo}. Furthermore, there exists an optimal solution $\widehat{D}$ to Problem \ref{op:attack-dpo} with $\widehat{n}_\textnormal{DPO}$ identical samples such that
    \begin{align*}
        \widehat{n}_\textnormal{DPO} \leq   2\ceil{\frac{\lambda}{2\xi_{\max}}\left(\norm{\theta^\dagger-\theta_\mu}-\sqrt{\epsilon'}\right)^2}~.
    \end{align*}
    Finally, there exists $\eta_{\min}>0$, such that, for any $\epsilon'\geq \epsilon/\eta_{\min}$, we have
    \begin{align*}
        \widehat{n}_\textnormal{DPO} \geq   2\ceil{\frac{\lambda}{2\xi_{\max}}\left(\norm{\theta^\dagger-\theta_\mu}-\sqrt{\epsilon'}\right)^2}~.
    \end{align*}
\end{restatable}

\section{Comparison between RLHF and DPO for Attack Susceptibility}\label{sec:comparison-of-paradigms}

In this section, we present some interesting takeaways from the analysis of previous sections. We aim to provide a comparative analysis of the attack sample complexities between the RLHF and DPO paradigms. Specifically, we focus on the contextual bandit setting.

The following result establishes an explicit relationship between the sample complexities of data augmentation attacks on the RLHF and DPO paradigms. Specifically, we compare the sample complexities, $\widehat{n}_\textnormal{RLHF}$ and $\widehat{n}_\textnormal{DPO}$, required by the optimal solutions to Problems~\ref{op:rlhf-attack} and~\ref{op:attack-dpo}, respectively (with $\abs{\cleandata} = \overline{n}$). 
\begin{restatable}{retheorem}{comparisonwithdata}\label{thm:comparison-augmentation}
    Let $\pi^\dagger, \mu\in\Pi^\textnormal{log}$ be loglinear with parameters $\theta^\dagger$ and $\theta_\mu$, respectively. Furthermore, let 
    $\epsilon$ be such that $\epsilon\leq 1/(2\xi_{\max})$, $\epsilon'\geq \epsilon/\eta_{\min}$, where $\eta_{\min}>0$ is an absolute constant, and let $\omega^\dagger$ be a feasible solution to Problem \ref{op:rlhf-attack}. Define $\kappa_1$ as
    \begin{align*}
         & \left(\frac{\lambda}{\xi_{\max}}\left(\norm{\theta^\dagger-\theta_\mu}-\sqrt{\epsilon'}\right)^2 -\overline{n}\right) \\ & \quad\quad\quad \cdot  \ceil{\frac{1}{\xi_{\max}}\left(\lambda\norm{\omega^\dagger}^2 + \norm{\omega^\dagger}\cdot\frac{2\overline{n}}{(1-\gamma)^2}\right)}^{-1}
    \end{align*}
    Then, we have $\widehat{n}_\textnormal{DPO} \geq \kappa_1 \cdot\widehat{n}_\textnormal{RLHF}$.
\end{restatable}
The value of $\kappa_1$ is proportional to 
the distance between $\theta^\dagger$ and $\theta_\mu$, which captures how far $\pi^\dagger$ is from the reference policy $\mu$. We observe that, the greater the distance between $\pi^\dagger$ and $\mu$, the less susceptible DPO becomes relative to RLHF. 
Lower susceptibility implies DPO has a stronger tendency to remain close to $\mu$. Furthermore, note that if $\frac{\lambda}{\xi_{\max}}\left(\norm{\theta^\dagger-\theta_\mu}-\sqrt{\epsilon'}\right)^2 \leq \overline{n}$, the lower bound in Theorem~\ref{thm:dpo-lower-bound} becomes vacuous. Therefore, we should assume that $\pi^\dagger$ and $\mu$ are far enough, or that the size of the $\cleandata$ is small, for this bound to be meaningful. This assumption is not restrictive, as in practice, the nature of $\pi^\dagger$ often differs significantly from $\mu$, leading to large divergence terms.

\section{Related Work}

\paragraph{Adversarial attacks in machine learning (ML).} The problem of adversarial attacks in ML has a long history \citep{szegedy2013intriguing, biggio2013evasion, nguyen2015deep, papernot2017practical, biggio2012poisoning, li2016data, xiao2012adversarial}, where various types of attacks   have been considered, including training-time attacks, test-time attacks, and backdoor attacks. The  focus of the present study is on training-time attacks. The closest to our work in this domain is that of \cite{DBLP:journals/jmlr/LiuZ16}, who consider the teaching problem (via data synthesis) to various types of learners, including logistic regression learners. Similar to \citep{DBLP:journals/jmlr/LiuZ16}, we also consider logistic regression in our optimization problems. However, our attack problems include additional constraints, which necessitate the usage of additional technical machinery. 

\looseness-1\paragraph{Data poisoning attacks and defenses in  (multi-agent) reinforcement learning ((MA)RL).} Adversarial attacks in RL have been explored extensively in the literature~\citep{huang2017adversarial, gleave2019adversarial, lin2017tactics, sun2020stealthy, rangi2022understanding, ma2019policy,DBLP:conf/icml/RakhshaRD0S20,DBLP:journals/jmlr/RakhshaRDZS21}, including training-time attacks \citep{rakhsha2020policy, xu2021transferable}, test-time attacks \citep{behzadan2017whatever, huang2017adversarial, kos2017delving, sun2020stealthy}, backdoor attacks \citep{kiourti2020trojdrl, wang2021stop, yang2019design} and attacks to MARL systems \citep{wu2024data, nikadefending, mohammadi2023implicit, nika2024corruption}.  Our research focuses on training-time poisoning attacks against single agents, where adversaries manipulate training data within certain constraints~\citep{mei2015using,xiao2015feature,DBLP:conf/icml/RakhshaRD0S20,DBLP:journals/jmlr/RakhshaRDZS21}. In the unregularized RLHF setting, our attack  problem utilizes constraints that determine the target policy's strict optimality as in \citep{DBLP:journals/jmlr/RakhshaRDZS21}. Different from these works, our focus is on the studying attacks in RLHF.

Complementing this, significant research has also been conducted on robust RL methods designed to defend against poisoning attacks~\citep{zhang2021robusta, lykouris2021corruption, kumar2021policy, rangi2022saving, wu2022copa, zhang2022corruption, mcmahan2024optimal, banihashemdefense, nika2023online}. Our work diverges from these studies by being the first to theoretically investigate the inherent robustness of RLHF and DPO paradigms against poisoning attacks.

\looseness-1\paragraph{Data poisoning attacks and defenses in learning from human preferences.} Recent studies have empirically explored the vulnerability of the RLHF paradigm to poisoning attacks~\citep{wang2023exploitability,shi2023badgpt,rando2023universal,baumgartner2024best}, addressing various attack types such as label flipping, backdoor, and data augmentation. Despite these empirical explorations, prior work has primarily focused on evaluating the effectiveness of these attacks, whereas our research aims to provide a theoretical understanding of these vulnerabilities. On the defense side, \citet{mandal2024corruption} and \citet{chowdhury2024provably} have proposed robust RLHF and DPO algorithms to handle data corruption. However, our work aims to understand the natural robustness of these algorithms against structured data poisoning attacks and studies the problem from an attacker's perspective.

\looseness-1\paragraph{Theoretical analysis of learning from human preferences.} Significant research has focused on theoretically understanding and improving the performance of RLHF and DPO~\citep{zhu2023principled,zhan2023provable,an2023direct,gheshlaghi2023general,wang2023beyond,hejna2023contrastive,nika2024reward}. While in \citep{zhu2023principled, zhan2023provable} the focus is on unregularized RLHF, \cite{gheshlaghi2023general, hejna2023contrastive, nika2024reward} study regularized RLHF and the effect of regularization. Our work also tries to further theoretical understanding of learning from human preferences, by undertaking a rigorous analysis of data poisoning attacks on these methods.

\section{Concluding Discussion}\label{sec:conclusion}
\looseness-1We considered data poisoning attacks in learning from human preferences. Specifically, we studied data augmentation attacks on RLHF and DPO. Based on our findings, we compared these paradigms in terms of susceptibility to attacks. There are several directions for future work. First, it would be interesting to relax the (log)linearity assumptions and solve the problem for general parametrizations. It is currently not clear whether RLHF and DPO attacks are feasible for general formulations. Second, as it is not possible to obtain closed-form solutions to our problems whenever KL constraints are present, it would be useful to study this attack framework with alternate constraints, such as total variation distance. Third, it would also be important to study other forms of attacks, e.g., \textit{label-flipping attacks} where the attacker is only allowed to flip a fraction of the preference labels in the dataset. Finally, it would be interesting to study attacks against more recent preference-based RL methods and understand the effectiveness of these attacks when a learner uses robust variants of these methods.

\subsection*{Acknowledgements}
The work of Andi Nika and Goran Radanovic was funded by the Deutsche Forschungsgemeinschaft (DFG, German Research Foundation) – project number 467367360.

\bibliography{bibliography}

\begin{thebibliography}{64}
\providecommand{\natexlab}[1]{#1}
\providecommand{\url}[1]{\texttt{#1}}
\expandafter\ifx\csname urlstyle\endcsname\relax
  \providecommand{\doi}[1]{doi: #1}\else
  \providecommand{\doi}{doi: \begingroup \urlstyle{rm}\Url}\fi

\bibitem[An et~al.(2023)An, Lee, Zuo, Kosaka, Kim, and Song]{an2023direct}
Gaon An, Junhyeok Lee, Xingdong Zuo, Norio Kosaka, Kyung-Min Kim, and Hyun~Oh Song.
\newblock Direct {P}reference-based {P}olicy {O}ptimization without {R}eward {M}odeling.
\newblock In \emph{{NeurIPS}}, 2023.

\bibitem[Azar et~al.(2023)Azar, Rowland, Piot, Guo, Calandriello, Valko, and Munos]{gheshlaghi2023general}
Mohammad~Gheshlaghi Azar, Mark Rowland, Bilal Piot, Daniel Guo, Daniele Calandriello, Michal Valko, and R{\'{e}}mi Munos.
\newblock A {G}eneral {T}heoretical {P}aradigm to {U}nderstand {L}earning from {H}uman {P}references.
\newblock \emph{CoRR}, abs/2310.12036, 2023.

\bibitem[Bai et~al.(2022)]{bai2022training}
Yuntao Bai et~al.
\newblock Training a {H}elpful and {H}armless {A}ssistant with {R}einforcement {L}earning from {H}uman {F}eedback.
\newblock \emph{CoRR}, abs/2204.05862, 2022.

\bibitem[Banihashem et~al.(2023)Banihashem, Singla, and Radanovic]{banihashemdefense}
Kiarash Banihashem, Adish Singla, and Goran Radanovic.
\newblock Defense {A}gainst {R}eward {P}oisoning {A}ttacks in {R}einforcement {L}earning.
\newblock \emph{{TMLR}}, 2023.

\bibitem[Baumg{\"a}rtner et~al.(2024)Baumg{\"a}rtner, Gao, Alon, and Metzler]{baumgartner2024best}
Tim Baumg{\"a}rtner, Yang Gao, Dana Alon, and Donald Metzler.
\newblock {B}est-of-{V}enom: {A}ttacking {RLHF} by {I}njecting {P}oisoned {P}reference {D}ata.
\newblock \emph{CoRR}, abs/2404.05530, 2024.

\bibitem[Behzadan and Munir(2017)]{behzadan2017whatever}
Vahid Behzadan and Arslan Munir.
\newblock Whatever does not {K}ill {D}eep {R}einforcement {L}earning, {M}akes it {S}tronger.
\newblock \emph{CoRR}, abs/1712.09344, 2017.

\bibitem[Biggio et~al.(2012)Biggio, Nelson, and Laskov]{biggio2012poisoning}
Battista Biggio, Blaine Nelson, and Pavel Laskov.
\newblock {P}oisoning {A}ttacks against {S}upport {V}ector {M}achines.
\newblock In \emph{ICML}, 2012.

\bibitem[Biggio et~al.(2013)]{biggio2013evasion}
Battista Biggio et~al.
\newblock {E}vasion {A}ttacks {A}gainst {M}achine {L}earning at {T}est {T}ime.
\newblock In \emph{ECML PKDD}, 2013.

\bibitem[Bradley and Terry(1952)]{bradley1952rank}
Ralph~Allan Bradley and Milton~E Terry.
\newblock {R}ank {A}nalysis of {I}ncomplete {B}lock {D}esigns: I. {T}he {M}ethod of {P}aired {C}omparisons.
\newblock \emph{Biometrika}, 39\penalty0 (3/4), 1952.

\bibitem[Brown et~al.(2019)Brown, Goo, Nagarajan, and Niekum]{brown2019learning}
Daniel~S. Brown, Wonjoon Goo, Prabhat Nagarajan, and Scott Niekum.
\newblock Extrapolating {B}eyond {S}uboptimal {D}emonstrations via {I}nverse {R}einforcement {L}earning from {O}bservations.
\newblock In \emph{{ICML}}, 2019.

\bibitem[Chowdhury et~al.(2024)Chowdhury, Kini, and Natarajan]{chowdhury2024provably}
Sayak~Ray Chowdhury, Anush Kini, and Nagarajan Natarajan.
\newblock Provably {R}obust {DPO}: Aligning {L}anguage {M}odels with {N}oisy {F}eedback.
\newblock \emph{CoRR}, abs/2403.00409, 2024.

\bibitem[Gao et~al.(2023)Gao, Schulman, and Hilton]{gao2023scaling}
Leo Gao, John Schulman, and Jacob Hilton.
\newblock Scaling {L}aws for {R}eward {M}odel {O}veroptimization.
\newblock In \emph{{ICML}}, 2023.

\bibitem[Glaese et~al.(2022)Glaese, McAleese, Tr{k{e}}bacz, Aslanides, Firoiu, Ewalds, Rauh, Weidinger, Chadwick, Thacker, et~al.]{glaese2022improving}
Amelia Glaese, Nat McAleese, Maja Tr{k{e}}bacz, John Aslanides, Vlad Firoiu, Timo Ewalds, Maribeth Rauh, Laura Weidinger, Martin Chadwick, Phoebe Thacker, et~al.
\newblock Improving {A}lignment of {D}ialogue {A}gents via {T}argeted {H}uman {J}udgements.
\newblock \emph{CoRR}, abs/2209.14375, 2022.

\bibitem[Gleave et~al.(2020)Gleave, Dennis, Wild, Kant, Levine, and Russell]{gleave2019adversarial}
Adam Gleave, Michael Dennis, Cody Wild, Neel Kant, Sergey Levine, and Stuart Russell.
\newblock {A}dversarial {P}olicies: {A}ttacking {D}eep {R}einforcement {L}earning.
\newblock In \emph{ICLR}, 2020.

\bibitem[Hejna et~al.(2023)Hejna, Rafailov, Sikchi, Finn, Niekum, Knox, and Sadigh]{hejna2023contrastive}
Joey Hejna, Rafael Rafailov, Harshit Sikchi, Chelsea Finn, Scott Niekum, W~Bradley Knox, and Dorsa Sadigh.
\newblock Contrastive {P}refence {L}earning: Learning from {H}uman {F}eedback without {RL}.
\newblock \emph{CoRR}, abs/2310.13639, 2023.

\bibitem[Hoorfar and Hassani(2008)]{hoorfar2008inequalities}
Abdolhossein Hoorfar and Mehdi Hassani.
\newblock Inequalities on the {L}ambert {W} {F}unction and {H}yperpower {F}unction.
\newblock \emph{J. Inequal. Pure and Appl. Math}, 9\penalty0 (2):\penalty0 5--9, 2008.

\bibitem[Huang et~al.(2017)Huang, Papernot, Goodfellow, Duan, and Abbeel]{huang2017adversarial}
Sandy Huang, Nicolas Papernot, Ian Goodfellow, Yan Duan, and Pieter Abbeel.
\newblock {A}dversarial {A}ttacks on {N}eural {N}etwork {P}olicies.
\newblock \emph{CoRR}, abs/1702.02284, 2017.

\bibitem[Kiourti et~al.(2020)Kiourti, Wardega, Jha, and Li]{kiourti2020trojdrl}
Panagiota Kiourti, Kacper Wardega, Susmit Jha, and Wenchao Li.
\newblock Trojdrl: Evaluation of {B}ackdoor {A}ttacks on {D}eep {R}einforcement {L}earning.
\newblock In \emph{{ACM/IEEE (DAC)}}, 2020.

\bibitem[Kos and Song(2017)]{kos2017delving}
Jernej Kos and Dawn Song.
\newblock Delving into {A}dversarial {A}ttacks on {D}eep {P}olicies.
\newblock \emph{CoRR}, abs/1705.06452, 2017.

\bibitem[Kumar et~al.(2021)Kumar, Levine, and Feizi]{kumar2021policy}
Aounon Kumar, Alexander Levine, and Soheil Feizi.
\newblock Policy {S}moothing for {P}rovably {R}obust {R}einforcement {R}earning.
\newblock \emph{CoRR}, abs/2106.11420, 2021.

\bibitem[Li et~al.(2016)Li, Wang, Singh, and Vorobeychik]{li2016data}
Bo~Li, Yining Wang, Aarti Singh, and Yevgeniy Vorobeychik.
\newblock Data {P}oisoning {A}ttacks on {F}actorization-based {C}ollaborative {F}iltering.
\newblock In \emph{NeurIPS}, 2016.

\bibitem[Lin et~al.(2017)Lin, Hong, Liao, Shih, Liu, and Sun]{lin2017tactics}
Yen-Chen Lin, Zhang-Wei Hong, Yuan-Hong Liao, Meng-Li Shih, Ming-Yu Liu, and Min Sun.
\newblock Tactics of {A}dversarial {A}ttack on {D}eep {R}einforcement {L}earning {A}gents.
\newblock In \emph{IJCAI}, 2017.

\bibitem[Liu and Zhu(2016)]{DBLP:journals/jmlr/LiuZ16}
Ji~Liu and Xiaojin Zhu.
\newblock The {T}eaching {D}imension of {L}inear {L}earners.
\newblock \emph{Journal of Machine Learning Resesearch}, 17:\penalty0 162:1--162:25, 2016.

\bibitem[Lykouris et~al.(2021)Lykouris, Simchowitz, Slivkins, and Sun]{lykouris2021corruption}
Thodoris Lykouris, Max Simchowitz, Alex Slivkins, and Wen Sun.
\newblock Corruption-robust {E}xploration in {E}pisodic {R}einforcement {L}earning.
\newblock In \emph{{COLT}}, 2021.

\bibitem[Ma et~al.(2019)Ma, Zhang, Sun, and Zhu]{ma2019policy}
Yuzhe Ma, Xuezhou Zhang, Wen Sun, and Jerry Zhu.
\newblock Policy {P}oisoning in {B}atch {R}einforcement {L}earning and {C}ontrol.
\newblock In \emph{NeurIPS}, 2019.

\bibitem[Mandal et~al.(2024)Mandal, Nika, Kamalaruban, Singla, and Radanovi{\'c}]{mandal2024corruption}
Debmalya Mandal, Andi Nika, Parameswaran Kamalaruban, Adish Singla, and Goran Radanovi{\'c}.
\newblock Corruption {R}obust {O}ffline {R}einforcement {L}earning with {H}uman {F}eedback.
\newblock \emph{CoRR}, abs/2402.06734, 2024.

\bibitem[McMahan et~al.(2024)McMahan, Wu, Zhu, and Xie]{mcmahan2024optimal}
Jeremy McMahan, Young Wu, Xiaojin Zhu, and Qiaomin Xie.
\newblock Optimal {A}ttack and {D}efense for {R}einforcement {L}earning.
\newblock In \emph{{AAAI}}, 2024.

\bibitem[Mei and Zhu(2015)]{mei2015using}
Shike Mei and Xiaojin Zhu.
\newblock Using {M}achine {T}eaching to {I}dentify {O}ptimal {T}raining-set {A}ttacks on {M}achine {L}earners.
\newblock In \emph{AAAI}, 2015.

\bibitem[Menick et~al.(2022)]{menick2022teaching}
Jacob Menick et~al.
\newblock Teaching {L}anguage {M}odels to {S}upport {A}nswers with {V}erified {Q}uotes.
\newblock \emph{CoRR}, abs/2203.11147, 2022.

\bibitem[Mohammadi et~al.(2023)Mohammadi, N{\"o}ther, Mandal, Singla, and Radanovic]{mohammadi2023implicit}
Mohammad Mohammadi, Jonathan N{\"o}ther, Debmalya Mandal, Adish Singla, and Goran Radanovic.
\newblock Implicit {P}oisoning {A}ttacks in {T}wo-agent {R}einforcement {L}earning: Adversarial {P}olicies for {T}raining-time {A}ttacks.
\newblock In \emph{{AAMAS}}, 2023.

\bibitem[Nachum et~al.(2017)Nachum, Norouzi, Xu, and Schuurmans]{nachum2017bridging}
Ofir Nachum, Mohammad Norouzi, Kelvin Xu, and Dale Schuurmans.
\newblock Bridging the {G}ap {B}etween {V}alue and {P}olicy based {R}einforcement {L}earning.
\newblock In \emph{{NeurIPS}}, 2017.

\bibitem[Nguyen et~al.(2015)Nguyen, Yosinski, and Clune]{nguyen2015deep}
Anh Nguyen, Jason Yosinski, and Jeff Clune.
\newblock Deep {N}eural {N}etworks are {E}asily {F}ooled: High {C}onfidence {P}redictions for {U}nrecognizable {I}mages.
\newblock In \emph{CVPR}, 2015.

\bibitem[Nika et~al.(2023)Nika, Singla, and Radanovic]{nika2023online}
Andi Nika, Adish Singla, and Goran Radanovic.
\newblock Online {D}efense {S}trategies for {R}einforcement {L}earning {A}gainst {A}daptive {R}eward {P}oisoning.
\newblock In \emph{{AISTATS}}, 2023.

\bibitem[Nika et~al.(2024{\natexlab{a}})Nika, Mandal, Kamalaruban, Tzannetos, Radanovi{\'c}, and Singla]{nika2024reward}
Andi Nika, Debmalya Mandal, Parameswaran Kamalaruban, Georgios Tzannetos, Goran Radanovi{\'c}, and Adish Singla.
\newblock Reward {M}odel {L}earning vs. {D}irect {P}olicy {O}ptimization: A {C}omparative {A}nalysis of {L}earning from {H}uman {P}references.
\newblock In \emph{ICML}, 2024{\natexlab{a}}.

\bibitem[Nika et~al.(2024{\natexlab{b}})Nika, Mandal, Singla, and Radanovic]{nika2024corruption}
Andi Nika, Debmalya Mandal, Adish Singla, and Goran Radanovic.
\newblock Corruption-robust {O}ffline {T}wo-player {Z}ero-sum {M}arkov {G}ames.
\newblock In \emph{{AISTATS}}, pages 1243--1251. PMLR, 2024{\natexlab{b}}.

\bibitem[Nika et~al.(2024{\natexlab{c}})Nika, N{\"o}ther, Singla, and Radanovic]{nikadefending}
Andi Nika, Jonathan N{\"o}ther, Adish Singla, and Goran Radanovic.
\newblock Defending {A}gainst {U}nknown {C}orrupted {A}gents: {R}einforcement {L}earning of {A}dversarially {R}obust {N}ash {E}quilibria.
\newblock \emph{{TMLR}}, 2024{\natexlab{c}}.

\bibitem[Ouyang et~al.(2022)]{ouyang2022training}
Long Ouyang et~al.
\newblock Training {L}anguage {M}odels to {F}ollow {I}nstructions with {H}uman {F}eedback.
\newblock In \emph{{NeurIPS}}, 2022.

\bibitem[Papernot et~al.(2017)Papernot, McDaniel, Goodfellow, Jha, Celik, and Swami]{papernot2017practical}
Nicolas Papernot, Patrick McDaniel, Ian Goodfellow, Somesh Jha, Z~Berkay Celik, and Ananthram Swami.
\newblock Practical {B}lack-box {A}ttacks {A}gainst {M}achine {L}earning.
\newblock In \emph{ACM}, 2017.

\bibitem[Rafailov et~al.(2023)Rafailov, Sharma, Mitchell, Manning, Ermon, and Finn]{DBLP:conf/nips/RafailovSMMEF23}
Rafael Rafailov, Archit Sharma, Eric Mitchell, Christopher~D. Manning, Stefano Ermon, and Chelsea Finn.
\newblock Direct {P}reference {O}ptimization: Your {L}anguage {M}odel is {S}ecretly a {R}eward {M}odel.
\newblock In \emph{{NeurIPS}}, 2023.

\bibitem[Rakhsha et~al.(2020{\natexlab{a}})Rakhsha, Radanovic, Devidze, Zhu, and Singla]{DBLP:conf/icml/RakhshaRD0S20}
Amin Rakhsha, Goran Radanovic, Rati Devidze, Xiaojin Zhu, and Adish Singla.
\newblock {P}olicy {T}eaching via {E}nvironment {P}oisoning: {T}raining-time {A}dversarial {A}ttacks against {R}einforcement {L}earning.
\newblock In \emph{ICML}, 2020{\natexlab{a}}.

\bibitem[Rakhsha et~al.(2020{\natexlab{b}})Rakhsha, Radanovic, Devidze, Zhu, and Singla]{rakhsha2020policy}
Amin Rakhsha, Goran Radanovic, Rati Devidze, Xiaojin Zhu, and Adish Singla.
\newblock Policy {T}eaching via {E}nvironment poisoning: Training-time {A}dversarial {A}ttacks {A}gainst {R}einforcement {L}earning.
\newblock In \emph{{ICML}}, 2020{\natexlab{b}}.

\bibitem[Rakhsha et~al.(2021)Rakhsha, Radanovic, Devidze, Zhu, and Singla]{DBLP:journals/jmlr/RakhshaRDZS21}
Amin Rakhsha, Goran Radanovic, Rati Devidze, Xiaojin Zhu, and Adish Singla.
\newblock Policy {T}eaching in {R}einforcement {L}earning via {E}nvironment {P}oisoning attacks.
\newblock \emph{Journal of Machine Learning Research}, 22:\penalty0 210:1--210:45, 2021.

\bibitem[Rando and Tram{\`e}r(2023)]{rando2023universal}
Javier Rando and Florian Tram{\`e}r.
\newblock {U}niversal {J}ailbreak {B}ackdoors from {P}oisoned {H}uman {F}eedback.
\newblock In \emph{ICLR}, 2023.

\bibitem[Rangi et~al.(2022{\natexlab{a}})Rangi, Tran-Thanh, Xu, and Franceschetti]{rangi2022saving}
Anshuka Rangi, Long Tran-Thanh, Haifeng Xu, and Massimo Franceschetti.
\newblock Saving {S}tochastic {B}andits from {P}oisoning {A}ttacks via {L}imited {D}ata {V}erification.
\newblock In \emph{{AAAI}}, 2022{\natexlab{a}}.

\bibitem[Rangi et~al.(2022{\natexlab{b}})Rangi, Xu, Tran-Thanh, and Franceschetti]{rangi2022understanding}
Anshuka Rangi, Haifeng Xu, Long Tran-Thanh, and Massimo Franceschetti.
\newblock Understanding the {L}imits of {P}oisoning {A}ttacks in {E}pisodic {R}einforcement {L}earning.
\newblock In \emph{IJCAI}, 2022{\natexlab{b}}.

\bibitem[Shi et~al.(2023)Shi, Liu, Zhou, and Sun]{shi2023badgpt}
Jiawen Shi, Yixin Liu, Pan Zhou, and Lichao Sun.
\newblock Badgpt: Exploring {S}ecurity {V}ulnerabilities of {C}hat{GPT} via {B}ackdoor {A}ttacks to {I}nstruct{GPT}.
\newblock \emph{CoRR}, abs/2304.12298, 2023.

\bibitem[Shin et~al.(2023)Shin, Dragan, and Brown]{shin2023benchmarks}
Daniel Shin, Anca~D. Dragan, and Daniel~S. Brown.
\newblock Benchmarks and {A}lgorithms for {O}ffline {P}reference-{B}ased {R}eward {L}earning.
\newblock \emph{Transactions of Machine Learning Research}, 2023.

\bibitem[Stiennon et~al.(2020)Stiennon, Ouyang, Wu, Ziegler, Lowe, Voss, Radford, Amodei, and Christiano]{stiennon2020learning}
Nisan Stiennon, Long Ouyang, Jeffrey Wu, Daniel Ziegler, Ryan Lowe, Chelsea Voss, Alec Radford, Dario Amodei, and Paul~F Christiano.
\newblock Learning to {S}ummarize with {H}uman {F}eedback.
\newblock In \emph{{NeurIPS}}, 2020.

\bibitem[Sun et~al.(2020)Sun, Zhang, Xie, Ma, Zheng, Chen, and Liu]{sun2020stealthy}
Jianwen Sun, Tianwei Zhang, Xiaofei Xie, Lei Ma, Yan Zheng, Kangjie Chen, and Yang Liu.
\newblock Stealthy and {E}fficient {A}dversarial {A}ttacks {A}gainst {D}eep {R}einforcement {L}earning.
\newblock In \emph{AAAI}, 2020.

\bibitem[Szegedy et~al.(2013)Szegedy, Zaremba, Sutskever, Bruna, Erhan, Goodfellow, and Fergus]{szegedy2013intriguing}
Christian Szegedy, Wojciech Zaremba, Ilya Sutskever, Joan Bruna, Dumitru Erhan, Ian Goodfellow, and Rob Fergus.
\newblock Intriguing {P}roperties of {N}eural {N}etworks.
\newblock \emph{CoRR}, abs/1312.6199, 2013.

\bibitem[Wang et~al.(2023{\natexlab{a}})Wang, Jiang, Yang, Liu, and Chen]{wang2023beyond}
Chaoqi Wang, Yibo Jiang, Chenghao Yang, Han Liu, and Yuxin Chen.
\newblock Beyond {R}everse {KL}: Generalizing {D}irect {P}reference {O}ptimization with {D}iverse {D}ivergence {C}onstraints.
\newblock \emph{CoRR}, abs/2309.16240, 2023{\natexlab{a}}.

\bibitem[Wang et~al.(2023{\natexlab{b}})Wang, Wu, Chen, Vorobeychik, and Xiao]{wang2023exploitability}
Jiongxiao Wang, Junlin Wu, Muhao Chen, Yevgeniy Vorobeychik, and Chaowei Xiao.
\newblock On the {E}xploitability of {R}einforcement {L}earning with {H}uman {F}eedback for {L}arge {L}anguage {M}odels.
\newblock \emph{CoRR}, abs/2311.09641, 2023{\natexlab{b}}.

\bibitem[Wang et~al.(2021)Wang, Sarkar, Li, Maniatakos, and Jabari]{wang2021stop}
Yue Wang, Esha Sarkar, Wenqing Li, Michail Maniatakos, and Saif~Eddin Jabari.
\newblock Stop-and-go: Exploring {B}ackdoor {A}ttacks on {D}eep {R}einforcement {L}earning-based {T}raffic {C}ongestion {C}ontrol {S}ystems.
\newblock \emph{IEEE Transactions on Information Forensics and Security}, 16, 2021.

\bibitem[Wu et~al.(2022)Wu, Li, Xu, Zhang, Kailkhura, Kenthapadi, Zhao, and Li]{wu2022copa}
Fan Wu, Linyi Li, Chejian Xu, Huan Zhang, Bhavya Kailkhura, Krishnaram Kenthapadi, Ding Zhao, and Bo~Li.
\newblock Copa: Certifying {R}obust {P}olicies for {O}ffline {R}einforcement {L}earning against {P}oisoning {A}ttacks.
\newblock \emph{CoRR}, abs/2203.08398, 2022.

\bibitem[Wu et~al.(2024)Wu, McMahan, Zhu, and Xie]{wu2024data}
Young Wu, Jeremy McMahan, Xiaojin Zhu, and Qiaomin Xie.
\newblock Data {P}oisoning to {F}ake a {N}ash {E}quilibria for {M}arkov {G}ames.
\newblock In \emph{{AAAI}}, 2024.

\bibitem[Xiao et~al.(2012)Xiao, Xiao, and Eckert]{xiao2012adversarial}
Han Xiao, Huang Xiao, and Claudia Eckert.
\newblock Adversarial {L}abel {F}lips {A}ttack on {S}upport {V}ector {M}achines.
\newblock In \emph{ECAI}, 2012.

\bibitem[Xiao et~al.(2015)Xiao, Biggio, Brown, Fumera, Eckert, and Roli]{xiao2015feature}
Huang Xiao, Battista Biggio, Gavin Brown, Giorgio Fumera, Claudia Eckert, and Fabio Roli.
\newblock Is {F}eature {S}election {S}ecure {A}gainst {T}raining {D}ata {P}oisoning?
\newblock In \emph{ICML}, 2015.

\bibitem[Xu et~al.(2021)Xu, Wang, Raizman, and Rabinovich]{xu2021transferable}
Hang Xu, Rundong Wang, Lev Raizman, and Zinovi Rabinovich.
\newblock Transferable {E}nvironment {P}oisoning: Training-time {A}ttack on {R}einforcement {L}earning.
\newblock In \emph{{ICAAMS}}, 2021.

\bibitem[Yang et~al.(2019)Yang, Iyer, Reimann, and Virani]{yang2019design}
Zhaoyuan Yang, Naresh Iyer, Johan Reimann, and Nurali Virani.
\newblock Design of {I}ntentional {B}ackdoors in {S}equential {M}odels.
\newblock \emph{CoRR}, abs1902.09972, 2019.

\bibitem[Zhan et~al.(2023)Zhan, Uehara, Kallus, Lee, and Sun]{zhan2023provable}
Wenhao Zhan, Masatoshi Uehara, Nathan Kallus, Jason~D Lee, and Wen Sun.
\newblock Provable {O}ffline {R}einforcement {L}earning with {H}uman {F}eedback.
\newblock \emph{CoRR}, abs/:2305.14816, 2023.

\bibitem[Zhang et~al.(2021)Zhang, Chen, Zhu, and Sun]{zhang2021robusta}
Xuezhou Zhang, Yiding Chen, Xiaojin Zhu, and Wen Sun.
\newblock Robust {P}olicy {G}radient {A}gainst {S}trong {D}ata {C}orruption.
\newblock In \emph{{ICML}}, 2021.

\bibitem[Zhang et~al.(2022)Zhang, Chen, Zhu, and Sun]{zhang2022corruption}
Xuezhou Zhang, Yiding Chen, Xiaojin Zhu, and Wen Sun.
\newblock Corruption-robust {O}ffline {R}einforcement {L}earning.
\newblock In \emph{{AISTATS}}, 2022.

\bibitem[Zhu et~al.(2023)Zhu, Jordan, and Jiao]{zhu2023principled}
Banghua Zhu, Michael~I. Jordan, and Jiantao Jiao.
\newblock Principled {R}einforcement {L}earning with {H}uman {F}eedback from {P}airwise or {K}-wise {C}omparisons.
\newblock In \emph{{ICML}}, 2023.

\bibitem[Ziegler et~al.(2019)Ziegler, Stiennon, Wu, Brown, Radford, Amodei, Christiano, and Irving]{ziegler2019fine}
Daniel~M Ziegler, Nisan Stiennon, Jeffrey Wu, Tom~B Brown, Alec Radford, Dario Amodei, Paul Christiano, and Geoffrey Irving.
\newblock Fine-tuning {L}anguage {M}odels from {H}uman {P}references.
\newblock \emph{CoRR}, abs/1909.08593, 2019.

\end{thebibliography}
\bibliographystyle{plainnat}
\clearpage

\onecolumn
\appendix
\addcontentsline{toc}{section}{Appendix}
\part{Appendix}

\parttoc

\section{Feasibility of Problem \ref{op:rlhf-attack}.}\label{sec:rlhf_feasibility}

In this section, we  prove that Problem \ref{op:rlhf-attack} is feasible for certain classes of $\pi^\dagger$. 

\begin{theorem}\label{thm:rlhf_feasibility}
    Let $\epsilon, \beta >0$ and $\mu$ be with full support. Then, the following statements hold:
    \begin{itemize}
        \item Let $\pi^\dagger\in\Pi^\textnormal{det}$ and assume that $\omega^\dagger_\textnormal{opt}$ is such that $\pi^\dagger$ is $\epsilon$-robust optimal with respect to $r_{\omega^\dagger_\textnormal{opt}}$. Then, there exists $c>0$ such that $c\cdot\omega^\dagger_\textnormal{opt}$ is a feasible solution for Problem \ref{op:rlhf-attack}.
        \item Let $\pi^\dagger,\mu\in\Pi^\textnormal{log}$ and assume that the column space of $\Phi$ is a subspace of the column space of $\Psi$. Then the solution to $\Phi\omega = \beta(\log\widehat
        {\pi}^\dagger-\log\mu)$, where $\log\widehat{\pi}^\dagger-\log\mu$ is the vector with entries $\log\widehat{\pi}^\dagger(a|s)-\log\mu(a|s)$, for each $(s,a)$, is a feasible solution for Problem \ref{op:rlhf-attack}, for any $\widehat{\pi}^\dagger$ such that $\KL(\pi^\dagger||\widehat{\pi}^\dagger)\leq \epsilon'$.
    \end{itemize}
\end{theorem}
\begin{proof}
    We start with the first statement.
    Given dataset $D$, logistic regression applied on $D$ returns a reward function $r$. If $D$ is to be feasible for Problem \ref{op:rlhf-attack}, this necessitates that the optimal regularized policy $\pi^\textnormal{reg}_r$ with respect to reward function $r$ is $\epsilon$-close to $\pi^\dagger$ in the sense that $\norm{\pi^\dagger-\pi^\textnormal{reg}_r}_1\leq\epsilon$. Note that, if we can show that $\norm{\pi^\dagger-\pi^\textnormal{reg}_r}_\infty\leq\epsilon$, then we are done. Thus, for the rest of this proof, we will focus on $\norm{\cdot}_\infty$. 

    Now, given reward function $r$ with parameter $\omega$, Theorem \ref{thm:logistic_regression_teaching} shows us that we can always synthesise a dataset $D$ that makes $r$ the outcome when logistic regression is applied on it. Therefore, we focus our attention on finding the right reward function $r$ for which our constraint is satisfied. Note that our problem  can be written as
    \begin{align*}
        \arg\max_\pi &\; V^\pi_r(\rho) - \beta \KL^\gamma(\pi||\mu) := \pi^\textnormal{reg}_r\\
            \textnormal{s.t.} &\; \norm{\pi^\dagger-\pi^\textnormal{reg}_r}_\infty \leq \epsilon~,
    \end{align*}
    where
    \begin{align}
        \KL^\gamma(\pi||\mu) & = \expp\left[\sum_{t\geq 0}\gamma^t\log\frac{\pi(a_t|s_t)}{\mu(a_t|s_t)}\bigg|\rho,\pi\right] \nonumber\\ & = \expp\left[\sum_{t\geq 0}\gamma^t\sum_a\pi(a|s_t)\log\frac{\pi(a|s_t)}{\mu(a|s_t)}\bigg|\rho,\pi\right] \nonumber\\ & = \expp\left[\sum_{t\geq 0}\gamma^t\KL\left(\pi(\cdot|s_t)||\mu(\cdot|s_t)\right)\bigg|\rho,\pi\right]~.\label{eq:discounted_kl}
    \end{align}
    Given $\kappa << \epsilon$, let us define the set $\Pi_\kappa = \left\{ (1-\kappa)\cdot\pi + \kappa\cdot u | \pi\in\Pi\right\}$, where $u$ denotes the policy that takes any action uniformly at random, at any given state. Furthermore, let $l=\min_{\pi,s}d^\pi_\rho(s)$. By the ergodicity assumption, we have that $l>0$. Let $A^\pi_\rho(s,a)$ denote the advantage function of $\pi$ at state-action $(s,a)$. 

    By assumption, there exists an $\omega^\dagger$ such that $M_{\pi^\dagger}^\top\omega^\dagger\geq\boldsymbol{\epsilon}$, for a given $\epsilon$. Obviously, the policy $\pi^\dagger$ is optimal with respect to the reward function $r_{\omega^\dagger}$. We will use the short-hand notation $r^\dagger = r_{\omega^\dagger}$.
    
    Now, let us define $\Delta >0$ such that $A^{\pi^\dagger}(s,a) \leq -\Delta$, for all $a\neq\pi^\dagger(s)$, let $c>0$ be an arbitrary positive constant and let $r^{\dagger}_{\max}$ denote the maximum component of $r^\dagger$. Let $\pi_1=(1-\kappa)\cdot\pi^\dagger +\kappa\cdot u$. For any policy $\pi$, we have
    \begin{align*}
        V^{\pi_1}_{c\cdot r^\dagger}(\rho) - V^{\pi}_{c\cdot r^\dagger}(\rho) = \left(V^{\pi_1}_{c\cdot r^\dagger}(\rho) - V^{\pi^\dagger}_{c\cdot r^\dagger}(\rho)\right) + \left(V^{\pi^\dagger}_{c\cdot r^\dagger}(\rho)- V^{\pi}_{c\cdot r^\dagger}(\rho)\right)~.
    \end{align*}
    We will bound each term separately. First, using vector notation for the reward and occupancy measures, note that
    \begin{align}
        V^{\pi_1}_{c\cdot r^\dagger}(\rho) - V^{\pi^\dagger}_{c\cdot r^\dagger}(\rho) & = c\cdot\left(d^{\pi_1}_\rho-d^{\pi^\dagger}_\rho\right)^\top r^\dagger \label{eq:feasibility_proof_01}\\
            & \geq - c\cdot \norm{d^{\pi_1}_\rho - d^{\pi^\dagger}_\rho}_1\cdot\norm{r^\dagger}_\infty\label{eq:feasibility_proof_02}\\
            & \geq -\frac{\gamma\cdot c\cdot\kappa\cdot r^{\dagger}_{\max}}{1-\gamma}~,\label{eq:feasibility_proof_03}
    \end{align}
    where \eqref{eq:feasibility_proof_01} uses the occupancy measure expression of the discounted return; \eqref{eq:feasibility_proof_02} uses Cauchy-Schwarz; for \eqref{eq:feasibility_proof_03} we have used that $\norm{\pi_1-\pi^\dagger}_1 =\kappa$ and that \textit{similar policies imply similar state visitations}, i.e., if $\norm{\pi_1-\pi_2}_1\leq \zeta$, then $\norm{d^{\pi_1}_\rho-d^{\pi_2}_\rho}_1\leq\zeta\gamma/(1-\gamma)$, for given positive $\zeta$ (see the RL Theory book Lemma 14.1). For the second term, we have
    \begin{align}
        V^{\pi^\dagger}_{c\cdot r^\dagger}(\rho)- V^{\pi}_{c\cdot r^\dagger}(\rho) & = - \left( V^{\pi}_{c\cdot r^\dagger}(\rho) - V^{\pi^\dagger}_{c\cdot r^\dagger}(\rho)\right) \nonumber\\
            & = -\left( \sum_{s'}d^{\pi}_\rho(s')\sum_{a'}\pi(a'|s')A^{\pi^\dagger}_{c\cdot r^\dagger}(s',a')\right)\label{eq:feasibility_proof_04}\\
        & = -\left( \sum_{s'}d^{\pi}_\rho(s')\sum_{a'\neq\pi^\dagger(s')}\pi(a'|s')A^{\pi^\dagger}_{c\cdot r^\dagger}(s',a')\right)\label{eq:feasibility_proof_05}\\
            & = \sum_{s'}d^{\pi}_\rho(s')\sum_{a'\neq\pi^\dagger(s')}\pi(a'|s')\cdot\left(- A^{\pi^\dagger}_{c\cdot r^\dagger}(s',a')\right)\nonumber\\
        & \geq \sum_{s'}d^{\pi}_\rho(s')\sum_{a'\neq\pi^\dagger(s')}\pi(a'|s')\cdot c\cdot \Delta \label{eq:feasibility_proof_06}\\
            & \geq l\cdot\sum_{s',a'\neq\pi^\dagger(s')}\pi(a'|s')\cdot c\cdot \Delta\label{eq:feasibility_proof_07}~,
    \end{align}
    where \eqref{eq:feasibility_proof_04} follows from the Performance Difference Lemma; \eqref{eq:feasibility_proof_05} uses the fact that $A^{\pi^\dagger}_{c\cdot r^\dagger}(s,\pi^\dagger(s))=c\cdot A^{\pi^\dagger}_{r^\dagger}(s,\pi^\dagger(s))=0$; \eqref{eq:feasibility_proof_06} uses the fact that $A^{\pi^\dagger}_{c\cdot r^\dagger}(s,\pi^\dagger(s))\leq -c\cdot \Delta$; \eqref{eq:feasibility_proof_07} uses the definition of $l$. 

    Now, let us define the set $\Pi^\epsilon_\kappa = \left\{ \pi\in\Pi_\kappa : \norm{\pi-\pi^\dagger}_\infty \leq\epsilon\right\}$. \eqref{eq:feasibility_proof_07} above implies that, for any $\pi\not\in \Pi^\epsilon_\kappa$,
    \begin{align}
        V^{\pi^\dagger}_{c\cdot r^\dagger}(\rho)- V^{\pi}_{c\cdot r^\dagger}(\rho) & \geq l\cdot\sum_{s',a'\neq\pi^\dagger(s')}\pi(a'|s')\cdot c\cdot \Delta\nonumber \\
        & \geq l\cdot \max_{s,a\neq\pi^\dagger(s)} \pi(a|s)\cdot c\cdot\Delta \nonumber\\
        & > l\cdot\epsilon\cdot c\cdot\Delta~,\label{eq:feasibility_proof_08}
    \end{align}
    where the third inequality follows from the fact that $\norm{\pi-\pi^\dagger}_\infty > \epsilon$. Thus, combining \eqref{eq:feasibility_proof_03} and \eqref{eq:feasibility_proof_08} into the original suboptimality gap, for $\pi\not\in\Pi^\epsilon_\kappa$, we obtain
    \begin{align}
        V^{\pi_1}_{c\cdot r^\dagger}(\rho) - V^{\pi}_{c\cdot r^\dagger}(\rho) > l\cdot\epsilon\cdot c\cdot\Delta - \frac{\gamma\cdot c \cdot\kappa\cdot r^{\dagger}_{\max}}{1-\gamma}~.\label{eq:feasibility_proof_09}
    \end{align}
    Now, for a given $\kappa>0$, the exists $C^\kappa>0$, such that $\KL^\gamma(\pi_1||\mu) < C^\kappa$, since both $\pi_1$ and $\mu$ have full support and thus $\KL^\gamma(\pi_1||\mu)\leq (1/(1-\gamma))\max_s \KL(\pi_1(\cdot|s)||\mu(\cdot|s)):= C^\kappa$. Hence, we have that
    \begin{align*}
        \max_{\pi\in\Pi^\epsilon_\kappa} V^{\pi}_{c\cdot r^\dagger} -\beta \KL^\gamma(\pi||\mu) \geq V^{\pi_1}_{c\cdot r^\dagger} -\beta \cdot C^\kappa~.
    \end{align*}
    Moreover, using the fact that $\KL^\gamma(\pi||\mu)\geq 0$, we also have
    \begin{align*}
         \max_{\pi\in\Pi^\epsilon_\kappa} V^{\pi}_{c\cdot r^\dagger} -\beta \KL^\gamma(\pi||\mu) \leq \max_{\pi\in\Pi^\epsilon_\kappa} V^{\pi}_{c\cdot r^\dagger} < V^{\pi_1}_{c\cdot r^\dagger}(\rho) - c\cdot\left(l\cdot\epsilon\cdot\Delta - \frac{\gamma\cdot\kappa\cdot r^\dagger_{\max}}{1-\gamma}\right)~,
    \end{align*}
    where the second inequality follows from \eqref{eq:feasibility_proof_09}. Finally, let $\kappa << l\cdot\epsilon\cdot(1-\gamma)/(\gamma\cdot r^\dagger_{\max})$ and $c = \beta\cdot C^\kappa/(l\cdot\epsilon - (\gamma/(1-\gamma))\cdot\kappa\cdot r^\dagger_{\max})$. Then, we obtain
    \begin{align*}
        \max_{\pi\in\Pi^\epsilon_\kappa} V^{\pi}_{c\cdot r^\dagger} -\beta \KL^\gamma(\pi||\mu) & \geq V^{\pi_1}_{c\cdot r^\dagger} -\beta \cdot C^\kappa\\
            & = V^{\pi_1}_{c\cdot r^\dagger} - c\cdot\left(l\cdot\epsilon\cdot\Delta - \frac{\gamma\cdot\kappa\cdot r^\dagger_{\max}}{1-\gamma}\right)\\
        & > \max_{\pi\not\in\Pi^\epsilon_\kappa} V^\pi_{c\cdot r^\dagger}(\rho) -\beta \cdot \KL^\gamma(\pi||\mu)~.
    \end{align*}
    Hence, we conclude that, for the choice of $r=c\cdot r^\dagger$ (and thus $\omega = c\cdot \omega^\dagger$), the solution to the regularized value maximization problem is necessarily in $\Pi^\epsilon_\kappa$, and thus, $\epsilon$-close to $\pi^\dagger$. 

    Next, we consider the second statement. Let $r=\beta(\log\pi^\dagger-\log\mu)$. Note that we have 
    \begin{align*}
        \arg\max_\pi V^\pi_r(\rho) - \beta \KL^\gamma(\pi||\mu) & =  \expp\left[\sum_{t\geq 0}\gamma^t\left(r(s_t,a_t)-\beta\log\frac{\pi(a_t|s_t)}{\mu(a_t|s_t)}\right)\bigg|\pi,\rho\right] \\ 
            & = \arg\max_\pi \expp\left[\sum_{t\geq 0}\gamma^t\left(\beta\log\frac{\pi^\dagger(a_t|s_t)}{\mu(a_t|s_t)}-\beta\log\frac{\pi(a_t|s_t)}{\mu(a_t|s_t)}\right)\bigg|\pi,\rho\right] \\
        & = \arg\max_\pi \expp\left[\sum_{t\geq 0}\gamma^t\left(\beta\log\frac{\pi^\dagger(a_t|s_t)}{\pi(a_t|s_t)}\right)\bigg|\pi,\rho\right] \\
            & = \arg\max_\pi -\expp\left[\sum_{t\geq 0}\gamma^t\left(\beta\log\frac{\pi(a_t|s_t)}{\pi^\dagger(a_t|s_t)}\right)\bigg|\pi,\rho\right] \\
        & = \arg\min_\pi \KL^\gamma\left(\pi||\pi^\dagger\right) \\
            & = \pi^\dagger~.
    \end{align*}
    This implies that, as long as we can find a reward function as proposed, enforcing $\pi^\dagger$ is feasible. Since $\pi^\dagger,\mu\in\Pi^\textnormal{log}$, there exist parameters $\theta^\dagger,\theta_\mu$ that realize these policies in the loglinear space. Thus, to find a linear reward function that satisfies our equation, we can equivalently solve
    \begin{align*}
        \Phi\omega = \beta\cdot \Psi\left(\theta^\dagger-\theta_\mu\right)~,
    \end{align*}
    which would automatically be a solution. Lemma 4.1 of \citep{nika2024reward} guarantees that we can find such a vector $\omega$ whenever the column space of $\Phi$ is a subspace of the column space of $\Psi$. By assumption, this condition is satisfied, and thus there exists $\widehat{\omega}$ for which we have $r_{\widehat{\omega}}=\beta(\log\pi^\dagger-\log\mu)$.
\end{proof}

\section{Proofs of Section \ref{sec:attack-to-rlhf}}\label{appendix:rlhf-attacks}

In this section, we provide the full proofs of results from Section \ref{sec:attack-to-rlhf}.

\subsection{Unregularized RLHF with Nonempty Existing Data}

We begin with the unregularized setting when the pre-existing data is non-empty and prove the following result.

\rlhfsynthesisabs*

\begin{proof}
Lemma \ref{lem:surrogate_connection_unreg_rlhf} implies that any feasible solution to Problem \ref{op:syn-rlhf-augment} is feasible for Problem \ref{op:rlhf-attack}. Thus, we focus on Problem \ref{op:syn-rlhf-augment}. 

First, note that the condition $\gamma\geq 1-2\norm{\omega^\dagger}/(\xi_{\max}+1)$ is needed to ensure a well-defined feature construction for our dataset, based on the condition provided by Lemma \ref{lem:bounded_feature_norm}. 

Now, let us consider the simpler problem of augmenting the data so that the solution to the logistic regression subproblem is a given $\widehat{\omega}$. The subproblem can be written as
\begin{align*}
    \min_{D}\;\; & |D|\\
    \text{s.t}\;\; & \widehat{\omega} = \arg\min_{\omega} \sum_{(\tau,\tau',o)\in \cleandata\cup D}\log\Big( 1 + \exp\left( -o\cdot \omega^\top\left( \phi(\tau)-\phi(\tau')\right)\right)\Big) + \frac{\lambda}{2}\norm{\omega}^2~.
\end{align*}
Lemma \ref{lem:attack_subproblem_rlhf} implies that the solution to the above is the dataset of 
\begin{align*}
    \left\lceil\frac{\left|(\widehat{\omega})^\top \nabla_\omega\ell^{\widehat{\omega}}_\textnormal{RLHF}(\cleandata)\right|}{\xi_{\max}}\right\rceil
\end{align*}
identical samples satisfying 
\begin{align*}
    \phi(\tau_j)-\phi(\tau'_j)= \xi^{-1}\left((\widehat{\omega})^\top \nabla_\omega\ell^{\widehat{\omega}}_\textnormal{RLHF}(\cleandata)\left\lceil\frac{\left|(\widehat{\omega})^\top \nabla_\omega\ell^{\widehat{\omega}}_\textnormal{RLHF}(\cleandata)\right|}{\xi_{\max}}\right\rceil^{-1}\right)\frac{\nabla_\omega\ell^{\widehat{\omega}}_\textnormal{RLHF}(\cleandata)}{(\widehat{\omega})^\top \nabla_\omega\ell^{\omega^\dagger}_\textnormal{RLHF}(\cleandata)}~, o_j=1~.
\end{align*}
Given this solution, we can equivalently write Problem \eqref{op:syn-rlhf-augment} in terms of $\omega$ as
\begin{align*}
    \min_\omega & \; \left|\omega^\top \nabla_\omega\ell^{\omega^\dagger}_\textnormal{RLHF}(\cleandata)\right|\\
    \text{s.t.} & \; \boldsymbol{\epsilon} - M_{\pi^\dagger}^\top\omega \leq \mathbf{0}~.
\end{align*}
Now, before we go any further, it is important to note that, whenever we have 
\begin{align*}
    \nabla_\omega\ell^{\omega^\dagger}_\textnormal{RLHF}(\cleandata) = \mathbf{0}~,
\end{align*}
the number of samples needed for the attack to succeed is $0$ since $\omega^\dagger$ is already optimal with respect to $\cleandata$.
Recall that
\begin{align*}
    \Sigma^\phi_{\cleandata} = \frac{1}{\overline{n}}\sum_{(\tau,\tau')\in \cleandata}\left(\phi(\tau)-\phi(\tau')\right)\left(\phi(\tau)-\phi(\tau')\right)^\top
\end{align*}
denotes the sample covariance matrix with respect to $\cleandata$ and let $C^\phi_{\cleandata}$ denote its minimum eigenvalue.
We will rewrite the problem using a different notation for simplicity in calculations. First, let 
\begin{align}\label{eq:X_vector_definition}
    X^\omega_{\cleandata} = \sum_{(\tau,\tau',o)\in \cleandata}\frac{-o\left(\phi(\tau)-\phi(\tau')\right)}{1+\exp\left(o\omega^\top\left(\phi(\tau)-\phi(\tau')\right)\right)}~,
\end{align}
and
\begin{align}\label{eq:Y_matrix_definition}
    Y^\omega_{\cleandata} & = \nabla_\omega (X^\omega_{\cleandata} + 2\lambda\omega)  \\
         & = \sum_{(\tau,\tau',o)\in \cleandata} \frac{\exp(o\omega^\top\left(\phi(\tau)-\phi(\tau')\right)}{\left(1+\exp(o\omega^\top\left(\phi(\tau)-\phi(\tau')\right)\right)^2}\left(\phi(\tau)-\phi(\tau')\right)\left(\phi(\tau)-\phi(\tau')\right)^\top +2\lambda I~.
\end{align}
Note that
\begin{align}
    \norm{X^\omega_{\cleandata}} & = \norm{\sum_{(\tau,\tau',o)\in \cleandata}\frac{-o\left(\phi(\tau)-\phi(\tau')\right)}{1+\exp\left(o\omega^\top\left(\phi(\tau)-\phi(\tau')\right)\right)}}\nonumber\\
        & \leq \sum_{(\tau,\tau',o)\in \cleandata}\frac{1}{1+\exp\left(o\omega^\top\left(\phi(\tau)-\phi(\tau')\right)\right)}\norm{\phi(\tau)-\phi(\tau')}\label{eq:rlhf-warmup-augm-Xeq01}\\
    & \leq \sum_{(\tau,\tau',o)\in \cleandata} \left(\norm{\phi(\tau)}+\norm{\phi(\tau')}\right)\label{eq:rlhf-warmup-augm-Xeq001}\\
        & \leq \frac{2\overline{n}}{(1-\gamma)}\label{eq:rlhf-warmup-augm-Xeq02}~,
\end{align}
where  \eqref{eq:rlhf-warmup-augm-Xeq01} follows from the triangle inequality, while  \eqref{eq:rlhf-warmup-augm-Xeq001} uses that fact that
\begin{align*}
    \norm{\sum^\infty_{t=0}\gamma^t\phi(s_t,a_t)} \leq \sum^\infty_{t=0}\gamma^t\norm{\phi(s_t,a_t)} \leq \frac{1}{1-\gamma}~,
\end{align*}
by assumption, and the fact that $1/(1+\exp(x)) \leq 1$. Also, note that, for strictly positive $\lambda$, the matrix $Y^\omega_{\cleandata}$ is symmetric positive definite. To obtain further information about the spectrum of $Y^\omega_{\cleandata}$, which we will use for solving our problem, we prove the following intermediate result. 

Going back to our optimization problem and using the above notation and result, the Lagrangian can be written as
\begin{align*}
    \calL(\omega,\alpha) = \left|\omega^\top\left( X^\omega_{\cleandata} +\lambda\omega\right)\right| + ( \boldsymbol{\epsilon} - M_{\pi^\dagger}^\top\omega)^\top\alpha~,
\end{align*}
and its gradient is
\begin{align*}
    \nabla_\omega\calL(\omega,\alpha) & =  e_\omega\cdot \left(X^\omega_{\cleandata} + \lambda\omega + \nabla_\omega (X^\omega_{\cleandata} + \lambda\omega)\omega\right) - M_{\pi^\dagger}\alpha \\
        & = e_\omega\cdot\left(X^\omega_{\cleandata} +   Y^\omega_{\cleandata}\omega\right) - M_{\pi^\dagger}\alpha~,
\end{align*}
where 
\begin{align*}
    e_\omega = \textnormal{sgn}\left(\omega^\top\left( X^\omega_{\cleandata} +\lambda\omega\right)\right)
\end{align*}
is the sign of the quantity inside the brackets.
Thus, the first-order condition implies
\begin{align*}
    \omega^\dagger & = (Y^{\omega^\dagger}_{\cleandata})^{-1}\left(e_\omega\cdot M_{\pi^\dagger}\alpha - X^{\omega^\dagger}_{\cleandata}\right)~.
\end{align*}
Complementary slackness implies that 
\begin{align*}
   \text{either}\;\; M_{\pi^\dagger}^\top\omega^\dagger = \boldsymbol{\epsilon}, \;\; \text{or} \;\; \alpha = \mathbf{0}~.
\end{align*}

We consider the first case, since the goal of the attacker is to be able to enforce $\pi^\dagger$, thus we cannot have $\alpha=0$ since this would disregard the constraint altogether. If $M_{\pi^\dagger}^\top{\omega^\dagger} = \boldsymbol{\epsilon}$, then this,  together with the first-order condition, imply
\begin{align*}
    M_{\pi^\dagger}^\top(Y^{\omega^\dagger}_{\cleandata})^{-1}\left(e_\omega\cdot M_{\pi^\dagger}\alpha - X^{\omega^\dagger}_{\cleandata}\right) = \boldsymbol{\epsilon}~.
\end{align*}
Using this and Cauchy-Schwarz, we have
\begin{align*}
    \norm{M_{\pi^\dagger}}\norm{(Y^{\omega^\dagger}_{\cleandata})^{-1}\left(e\cdot M_{\pi^\dagger}\alpha - X^{\omega^\dagger}_{\cleandata}\right)} & \geq \norm{M_{\pi^\dagger}^\top(Y^{\omega^\dagger}_{\cleandata})^{-1}\left(e_\omega\cdot M_{\pi^\dagger}\alpha - X^{\omega^\dagger}_{\cleandata}\right)}\\
        & = \sqrt{SA}\epsilon
\end{align*}
which yields
\begin{align}\label{eq:rlhf-warmup-augm-lower-omega}
    \norm{\omega^\dagger} = \norm{(Y^{\omega^\dagger}_{\cleandata})^{-1}\left(e\cdot M_{\pi^\dagger}\alpha - X^{\omega^\dagger}_{\cleandata}\right)} \geq \frac{\sqrt{SA}\epsilon}{\norm{M_{\pi^\dagger}}} \geq \frac{\sqrt{SA}\epsilon}{\sigma_{\max}\left(M_{\pi^\dagger}\right)}~.
\end{align}
On the other hand, using again Cauchy-Schwarz on the complementary slackness condition, we get
\begin{align*}
    \norm{\omega^\dagger} \leq \frac{\sqrt{SA}\epsilon}{\sigma_{\min}(M_{\pi^\dagger})}~.
\end{align*}
Next, we will provide an upper bound on the norm of the gradient of $\omega^\dagger$ with respect to $\cleandata$. Let $\overline{\omega}$ be the optimal point with respect to $\cleandata$ and let $\proj_{\omega:M_{\pi^\dagger}\omega\geq\boldsymbol{\epsilon}}\left(\overline{\omega}\right)$ denote the projection of $\overline{\omega}$ onto the $\epsilon$-robust optimality polytope. 
To derive the upper bound, we use the fact that the projection of $\overline{\omega}$ onto the polytope is a feasible solution and thus provides an immediate upper bound. Note that
\begin{align}
    \left|\left(\omega^\dagger\right)^\top\left(X^{\omega^\dagger}_{\cleandata}+\lambda\omega^\dagger\right)\right| & =  \left|\left(\omega^\dagger\right)^\top\nabla_\omega\ell^{\omega^\dagger}_\textnormal{RLHF}(\cleandata)\right|\nonumber\\
        & \leq\left|\left(\proj_{\omega:M_{\pi^\dagger}\omega\geq\boldsymbol{\epsilon}}\left(\overline{\omega}\right)\right)^\top\nabla_\omega\ell^{\proj_{\omega:M_{\pi^\dagger}\omega\geq\boldsymbol{\epsilon}}\left(\overline{\omega}\right)}_\textnormal{RLHF}(\cleandata)\right|\nonumber\\
    & \leq \norm{\proj_{\omega:M_{\pi^\dagger}\omega\geq\boldsymbol{\epsilon}}\left(\overline{\omega}\right)}\norm{\nabla_\omega\ell^{\proj_{\omega:M_{\pi^\dagger}\omega\geq\boldsymbol{\epsilon}}\left(\overline{\omega}\right)}_\textnormal{RLHF}(\cleandata)}\nonumber \\
        & \leq (2\overline{n}+\lambda)\norm{\proj_{\omega:M_{\pi^\dagger}\omega\geq\boldsymbol{\epsilon}}\left(\overline{\omega}\right)}\norm{\proj_{\omega:M_{\pi^\dagger}\omega\geq\boldsymbol{\epsilon}}\left(\overline{\omega}\right)-\overline{\omega}}~,\label{eq:rlhf-projection-argument-eq03}
\end{align}
where the second inequality uses Cauchy-Schwarz and the third inequality uses the fact that the loss is Lipschitz (see Lemma \ref{lem:logistic_regression_properties}). Now, to compute the projection, we solve the following problem:
\begin{align*}
    \min_\omega \frac{1}{2}\norm{\omega-\overline{\omega}}^2\;\;\textnormal{such that}\;\; M_{\pi^\dagger}\omega \geq \boldsymbol{\epsilon}~.
\end{align*}
The Lagrangian of this problem is
\begin{align*}
    \calL(\omega,\nu) = \frac{1}{2}\norm{\omega-\overline{\omega}}^2 - \nu\left(M_{\pi^\dagger}\omega -\boldsymbol{\epsilon}\right)~,
\end{align*}
and the first order condition becomes
\begin{align*}
    \omega - \overline{\omega} - M_{\pi^\dagger}\nu=\mathbf{0}~,
\end{align*}
which gives us
\begin{align*}
    \omega = \overline{\omega} + M_{\pi^\dagger}\nu~.
\end{align*}
To compute $\nu$ we use complementary slackness to obtain
\begin{align*}
    M_{\pi^\dagger}\left(\overline{\omega} + M_{\pi^\dagger}\nu\right) = \boldsymbol{\epsilon}~,
\end{align*}
which in turn gives us
\begin{align*}
    \nu = \left(M_{\pi^\dagger}^\top M_{\pi^\dagger}\right)^+\left(\boldsymbol{\epsilon} - M_{\pi^\dagger}\overline{\omega}\right)~.
\end{align*}
Hence, we get
\begin{align}\label{eq:omega-projection}
    \proj_{\omega:M_{\pi^\dagger}\omega\geq\boldsymbol{\epsilon}}\left(\overline{\omega}\right) = \overline{\omega} + M_{\pi^\dagger}\left(M_{\pi^\dagger}^\top M_{\pi^\dagger}\right)^+\left(\boldsymbol{\epsilon} - M_{\pi^\dagger}\overline{\omega}\right)~.
\end{align}
We can upper bound the projection norm by using complementary slackness as:
\begin{align*}
    \epsilon\sqrt{SA} = \norm{M_{\pi^\dagger}\left(\overline{\omega} + M_{\pi^\dagger}\nu\right) } \leq \norm{M_{\pi^\dagger}} \norm{\overline{\omega} + M_{\pi^\dagger}\nu}\leq \sigma_{\max}(M_{\pi^\dagger})\norm{\proj_{\omega:M_{\pi^\dagger}\omega\geq\boldsymbol{\epsilon}}\left(\overline{\omega}\right)}~,
\end{align*}
and
\begin{align*}
     \epsilon\sqrt{SA} = \norm{M_{\pi^\dagger}\left(\overline{\omega} + M_{\pi^\dagger}\nu\right) } \geq\sigma_{\min}\left(M_{\pi^\dagger}\right)\norm{\overline{\omega} + M_{\pi^\dagger}\nu} = \sigma_{\min}\left(M_{\pi^\dagger}\right)\norm{\proj_{\omega:M_{\pi^\dagger}\omega\geq\boldsymbol{\epsilon}}\left(\overline{\omega}\right)}~,
\end{align*}
where we use the fact that $\sigma_{\min}(M)\norm{x}\leq \norm{Mx}\leq\sigma_{\max}(M)\norm{x}$. 
Using the overall case when $\alpha=0$ and  \eqref{eq:rlhf-projection-argument-eq03}, we have
\begin{align*}
    \left|\left(\omega^\dagger\right)^\top\left(X^{\omega^\dagger}_{\cleandata}+\lambda\omega^\dagger\right)\right| & \leq (2\overline{n}+\lambda)\norm{\proj_{\omega:M_{\pi^\dagger}\omega\geq\boldsymbol{\epsilon}}\left(\overline{\omega}\right)}\norm{\proj_{\omega:M_{\pi^\dagger}\omega\geq\boldsymbol{\epsilon}}\left(\overline{\omega}\right)-\overline{\omega}}\\
        & \leq (2\overline{n}+\lambda) \frac{\epsilon\sqrt{SA}}{\sigma_{\min}(M_{\pi^\dagger})}\left(\frac{\epsilon\sqrt{SA}}{\sigma_{\min}(M_{\pi^\dagger})} + \norm{\overline{\omega}}\right)\\
    & \leq (2\overline{n}+\lambda) \left(\frac{\epsilon^2SA}{\sigma^2_{\min}\left(M_{\pi^\dagger}\right)} +\norm{\overline{\omega}}  \frac{\epsilon\sqrt{SA}}{\sigma_{\min}(M_{\pi^\dagger})}\right)~.
\end{align*}
Putting things together, we finally obtain
\begin{align*}
    \widehat{n}_\textnormal{RLHF} \leq 
    \left\lceil\frac{2\overline{n}+\lambda}{\xi_{\max}} \left(\frac{\epsilon^2SA}{\sigma^2_{\min}\left(M_{\pi^\dagger}\right)} +\norm{\overline{\omega}}  \frac{\epsilon\sqrt{SA}}{\sigma_{\min}(M_{\pi^\dagger})}\right)\right\rceil
\end{align*}
\end{proof}

\subsection{Unregularized RLHF with Empty Existing Data.}

Next, we consider the unregularized setting when the pre-existing dataset $\cleandata$ is empty. 

\warmupresultrlhf*

\begin{proof}
    Lemma \ref{lem:surrogate_connection_unreg_rlhf} implies that any feasible solution to Problem \ref{op:syn-rlhf-augment} is feasible for Problem \ref{op:rlhf-attack}. Thus, we focus on Problem \ref{op:syn-rlhf-augment}.

    First, note that, for $\beta =0$ the optimal policy is deterministic. Now, given $\omega^\dagger\neq \boldsymbol{0}$, Lemma \ref{lem:attack_subproblem_rlhf} implies that the solution to the problem 
    \begin{align*}
        \min_D & \; |D| \\
        \text{s.t.} & \; \omega^\dagger = \arg\min_\omega \sum_{(\tau,\tau',o)\in D}\log\left( 1+ \exp\left( - o\cdot\omega^\top\left( \phi(\tau)-\phi(\tau')\right)\right)\right) + \frac{\lambda}{2}\norm{\omega}^2
    \end{align*}
    is the dataset of 
    \begin{align*}
        \left\lceil \frac{\lambda \norm{\omega^\dagger}^2}{\xi_{\max}}\right\rceil
    \end{align*}
    identical samples satisfying
    \begin{align*}
        \phi(\tau_i)-\phi(\tau'_i)=\xi^{-1}\left( \lambda\norm{\omega^\dagger}^2\left\lceil\frac{\lambda\norm{\omega^\dagger}^2}{\xi_{\max}}\right\rceil^{-1}\right)\frac{\omega^\dagger}{\norm{\omega^\dagger}^2}, \; o_i=1~.
    \end{align*}
Since the optimal sample size to solve the logistic regression subproblem depends on the norm of the parameter $\omega^\dagger$, then, by using the construction above we can directly minimize $\norm{\omega^\dagger}$ and constrain $\omega^\dagger$ to remain in the desired region. In this case, we can equivalently rewrite our original problem in terms of $\omega$ as
\begin{align*}
    \min_\omega & \;\frac{1}{2}\norm{\omega}^2 \\
    \text{s.t.} & \; \left(\sum_{s'}d^{\pi^\dagger}_\rho(s')\phi(s',\pi^\dagger(s')) - \sum_{s'}d^{\pi^\dagger\{s,a\}}_\rho(s')\phi(s',\pi^\dagger\{s,a\}(s'))\right)^\top\omega \geq \epsilon, \; \forall s,a\neq \pi^\dagger(s)\nonumber~.
\end{align*}
Using vector notation, we have
\begin{align*}
    \min_\omega &\; \frac{1}{2}\norm{\omega}^2\\
    \text{s.t.} &\; \boldsymbol{\epsilon} - M_{\pi^\dagger}^\top\omega \leq \mathbf{0}~,
\end{align*}
where $\boldsymbol{\epsilon}=\epsilon \mathbf{1}$. This is a convex program and thus local minima are global. The Lagrangian of the above is
\begin{align*}
    \calL(\omega, \alpha) = \frac{1}{2}\norm{\omega}^2 + \alpha^\top\left(\boldsymbol{\epsilon}-M_{\pi^\dagger}^\top\omega\right)~,
\end{align*}
and setting its gradient to zero gives us
\begin{align*}
    \omega^\dagger = M_{\pi^\dagger}\alpha~.
\end{align*}
Complementary slackness implies
\begin{align*}
    M^\top_{\pi^\dagger}\omega^\dagger = \boldsymbol{\epsilon}~,
\end{align*}
which, together with the first-order condition, imply
\begin{align*}
    \left(M_{\pi^\dagger}^\top M_{\pi^\dagger}\right)\alpha = \boldsymbol{\epsilon}~.
\end{align*}
Note that 
\begin{align*}
    \norm{M^\top_{\pi^\dagger}M_{\pi^\dagger}} = \norm{M_{\pi^\dagger}M_{\pi^\dagger}^\top}\leq Tr\left(M_{\pi^\dagger}M_{\pi^\dagger}^\top\right) \leq 2SA~,
\end{align*}
due to the fact that 
\begin{align*}
    & \norm{\sum_{s'}d^{\pi^\dagger}_\rho(s')\phi(s',\pi^\dagger(s')) - \sum_{s'}d^{\pi^\dagger\{s,a\}}_\rho(s')\phi(s',\pi^\dagger\{s,a\}(s'))} \\ & \quad\quad\quad \leq \norm{\sum_{s'}d^{\pi^\dagger}_\rho(s')\phi(s',\pi^\dagger(s')) } + \norm{ \sum_{s'}d^{\pi^\dagger\{s,a\}}_\rho(s')\phi(s',\pi^\dagger\{s,a\}(s'))}\\
        & \quad\quad\quad\leq \sum_{s'}d^{\pi^\dagger}_\rho(s')\norm{\phi(s',\pi^\dagger(s'))} + \sum_{s'}d^{\pi^\dagger\{s,a\}}_\rho(s')\norm{\phi(s',\pi^\dagger\{s,a\}(s'))} \\
    & \quad\quad\quad\leq 2~,
\end{align*}
for every $(s,a)$ pair. Thus, it follows that
\begin{align*}
    \norm{\omega^\dagger}\norm{M_{\pi^\dagger}}  \geq \norm{M_{\pi^\dagger}\omega^\dagger} = \norm{\left(M_{\pi^\dagger}^\top M_{\pi^\dagger}\right)\alpha} = \norm{\boldsymbol{\epsilon}} = \epsilon\sqrt{SA}~,
\end{align*}
and
\begin{align*}
    \sigma_{\min}\left(M_{\pi^\dagger}\right)\norm{\omega^\dagger} \leq \norm{M_{\pi^\dagger}\omega^\dagger} = \epsilon\sqrt{SA}~.
\end{align*}
From this, we have
\begin{align*}
    \frac{\epsilon\sqrt{SA}}{\sigma_{\max}\left(M_{\pi^\dagger}\right)} \leq \norm{\omega^\dagger} \leq \frac{\epsilon\sqrt{SA}}{\sigma_{\min}\left(M_{\pi^\dagger}\right)}~,
\end{align*}
which implies that
\begin{align*}
    \widehat{n}_\textnormal{RLHF} \leq  \left\lceil\frac{\epsilon^2\lambda SA}{\xi_{\max}\sigma^2_{\min}\left(M_{\pi^\dagger}\right)}\right\rceil~.
\end{align*}

\end{proof}

\subsection{Regularized RLHF with Non-empty Existing Data}\label{appendix:rlhf-reg}

In this section, we provide the full proof of the following result for the regularized setting.

\synrlhfwarmupreg*

\begin{proof}
Let $\epsilon'\leq\epsilon$. Lemma \ref{lem:surrogate_connection_reg_rlhf} implies that any feasible solution to Problem \ref{op:reg-rlhf-attack} is feasible for Problem \ref{op:rlhf-attack}. Thus, we focus on Problem \ref{op:reg-rlhf-attack}.

First, let us recall the KL-regularized objective in this setting. Given learned reward $r$, the objective is
\begin{align*}
    \max_{\pi} \mathcal{V}^\pi_{r}(\rho) := \expp_{s\sim\rho,\pi} \left[ \sum^\infty_{t=0}\gamma^t \left( r(s_t,a_t) - \beta \log \frac{\pi(a_t|s_t)}{\mu(a_t|s_t)}\right) \Big| s_0=s \right]~.
\end{align*}

Let us use the shorthand notation $\calV^r(s)$, $\calQ^r(s,a)$ and $\calA^r(s,a)$ to denote the functions $\calV^{\pi^\textnormal{reg}_r}_r(s)$, $\calQ^{\pi^\textnormal{reg}_r}_r(s,a)$ and $\calA^{\pi^\textnormal{reg}_r}_r(s,a)$, respectively. Similarly, for parametrized reward function $r_\omega$, Let us use the short-hand notation $\calV^\omega(s), \calQ^\omega(s,a), \calA^\omega(s,a)$ to denote $\calV^{\pi^\textnormal{reg}_{r_\omega}}_{r_\omega}(s)$, $\calQ^{\pi^\textnormal{reg}_{r_\omega}}_{r_\omega}(s,a)$ and $\calA^{\pi^\textnormal{reg}_{r_\omega}}_{r_\omega}(s,a)$, respectively. 

Given reward $r$, Lemma \ref{lem:gradient_rlhf_mdp} implies that the regularized optimal policy in this case can be written as
\begin{align*}
    \pi^\textnormal{reg}_{r}(a|s) =\mu(a|s) \exp\left(\frac{1}{\beta}\calQ^r(s,a) - \frac{1}{\beta}\calV^r(s)\right)~,
\end{align*}
with $\exp(\calV^r(s)/\beta)$ as the partition constant, since it is action-independent. This implies the following relation between the two value functions holds:
\begin{align}\label{eq:relation_between_v_and_q}
    \calV^r(s) = \beta\log\sum_{a}\mu(a|s)\exp\left(\frac{1}{\beta}\calQ^r(s,a)\right)~.
\end{align}

Now, note that the second constraint of Problem \ref{op:reg-rlhf-attack} can be written as
\begin{align*}
    D_\textnormal{KL}\left( \pi^\dagger || \pi^\textnormal{reg}_{r_{\omega^\dagger}}\right)  & = \sum_{s,a} \rho(s) \pi^\dagger(a|s) \log\frac{\pi^\dagger(a|s)}{\pi^\textnormal{reg}_{r_{\omega^\dagger}}(a|s)} \\ & = \sum_{s,a} \rho(s) \pi^\dagger(a|s) \left( \log \frac{\pi^\dagger(a|s)}{\mu(a|s)} - \frac{1}{\beta}\calQ^{\omega^\dagger}(s,a)+ \frac{1}{\beta}\calV^{\omega^\dagger}(s)\right)\\
        & = D_\textnormal{KL}(\pi^\dagger ||\mu) -\frac{1}{\beta} \expp_{s\sim\rho,a\sim\pi^\dagger(\cdot|s)}\left[\calA^{\omega^\dagger}(s,a)\right]~.
\end{align*}

Thus, the attack optimization problem can be written as
\begin{align*}
    \min_{D}\; & |D| \\
    \text{s.t} \; & \omega^\dagger = \arg\min_\omega \sum_{(\tau,\tau',o)\in \overline{D}\cup D}\log\left( 1+ \exp\left( -o\cdot  \left( \phi(\tau)-\phi(\tau')\right)\right)\right) + \frac{\lambda}{2}\norm{\omega}^2 \\
    &  D_\textnormal{KL}(\pi^\dagger ||\mu) -\frac{1}{\beta} \expp_{s\sim\rho,a\sim\pi^\dagger(\cdot|s)}\left[\calA^{\omega^\dagger}(s,a)\right] -\epsilon' \leq 0~.
\end{align*}
Lemma \ref{lem:attack_subproblem_rlhf} implies that the solution to the subproblem
\begin{align*}
    \min_{D}\; & |D| \\
    \text{s.t} \; & \omega^\dagger = \arg\min_\omega \sum_{(\tau,\tau',o)\in \overline{D}\cup D}\log\left( 1+ \exp\left( -o\cdot  \left( \phi(\tau)-\phi(\tau')\right)\right)\right) + \frac{\lambda}{2}\norm{\omega}^2
\end{align*}
is the dataset of 
\begin{align*}
        \left\lceil \frac{\left|\nabla_\omega \ell^\omega_\textnormal{RLHF}(\overline{D})^\top\omega^\dagger\right|}{\xi_{\max}}\right\rceil
    \end{align*}
    identical samples satisfying 
    \begin{align*}
        \phi(\tau)-\phi(\tau') = \xi^{-1}\left( \frac{(\omega^\dagger)^\top\nabla_\omega\ell^\omega_\textnormal{RLHF}(\overline{D})}{\left\lceil \frac{\left|\nabla_\omega \ell^\omega_\textnormal{RLHF}(\overline{D})^\top\omega^\dagger\right|}{\xi_{\max}}\right\rceil}\right)\frac{\omega^\dagger}{\norm{\omega^\dagger}^2}, \;\; o=1~.
    \end{align*}
Thus, we can equivalently write the original problem in terms of $\omega$ as
\begin{align*}
        \min_\omega\;\; & \left|\omega^\top \nabla_\omega\ell^\omega_\textnormal{RLHF}(\cleandata)\right| \\
        \text{s.t.}\;\; & D_\textnormal{KL}(\pi^\dagger ||\mu) -\frac{1}{\beta} \expp_{s\sim\rho,a\sim\pi^\dagger(\cdot|s)}\left[\calA^\omega(s,a)\right] -\epsilon' \leq 0~,
    \end{align*}
with Lagrangian 
 \begin{align*}
        \calL(\omega,\alpha) = \left|\omega^\top \nabla_\omega\ell^\omega_\textnormal{RLHF}(\cleandata)\right| + \alpha\left( D_\textnormal{KL}(\pi^\dagger ||\mu) -\frac{1}{\beta} \expp_{s\sim\rho,a\sim\pi^\dagger(\cdot|s)}\left[\calA^\omega(s,a)\right] -\epsilon'\right)~.
    \end{align*} 
Before we consider the first-order conditions of the problem, we rewrite the expected advantage function in the constraint and derive its gradient below.
\begin{lemma}\label{lem:advantage_gradient}
    Given any $\omega\in\R^d$, let us define
    \begin{align}\label{eq:rlhf-reg-Gamma-sa-def}
        \Gamma^{\pi}_\phi(s,a) = \expp\left[ \sum^\infty_{t=0}\gamma^t \phi(s_t,a_t)|a_0=a,s_0=s,\pi \right]~,
    \end{align}
    for any policy $\pi$ and feature mapping $\phi$, and let
    \begin{align}\label{eq:rlhf-reg-Gamma-def}
        \Gamma^{\pi^\textnormal{reg}_{r_\omega}}_\phi\left(\pi^\dagger||\pi^\textnormal{reg}_{r_{\omega}}\right) =  \expp_{s\sim\rho,a\sim\pi^\dagger(\cdot|s)}\left[\Gamma^{\pi^\textnormal{reg}_{r_\omega}}_\phi(s,a)\right] - \expp_{s\sim\rho,a\sim\pi^\textnormal{reg}_{r_{\omega}}(\cdot|s)}\left[\Gamma^{\pi^\textnormal{reg}_{r_\omega}}_\phi(s,a)\right]~.
    \end{align}
    Then, we have
    \begin{align*}
        \expp_{s\sim\rho,a\sim\pi^\dagger(\cdot|s)}\left[\calA^\omega(s,a)\right] = \omega^\top\Gamma^{\pi^\textnormal{reg}_{r_\omega}}_\phi\left(\pi^\dagger||\pi^\textnormal{reg}_{r_\omega}\right)~,
    \end{align*}
    and
    \begin{align*}
        \nabla_\omega \calA^\omega(s,a) = \Gamma^{\pi^\textnormal{reg}_{r_\omega}}_\phi(s,a) - \sum_{a'}\pi^\textnormal{reg}_{r_\omega}(a'|s)\Gamma^{\pi^\textnormal{reg}_{r_\omega}}_\phi(s,a') ~.
    \end{align*}
\end{lemma}
\begin{proof}
We start by deriving the gradient of the action-value function. 
Given reward function $r$, policy $\pi$ and trajectory $\tau$, let
$$\mathbb{P}^\pi_s(\tau)=\pi(a_0|s)P(s,a_0,s_1)\pi(a_1|s_1)P(s_1,a_1,s_2)\pi(a_2|s_2)\ldots$$
and
$$\mathbb{P}^\pi_{s,a}(\tau)=P(s,a,s_1)\pi(a_1|s_1)P(s_1,a_1,s_2)\pi(a_2|s_2)\ldots~.$$
Now, for any given state-action pair $(s,a)$ and parameter $\omega$, observe that
\begin{align}
    & \nabla_\omega \calQ^\omega(s,a)  = \nabla_\omega\left( r_{\omega}(s,a) + \gamma\sum_{s'}P(s,a,s')\calV^\omega(s')\right) \nonumber\\ 
        & = \phi(s,a) + \gamma \sum_{s'}P(s,a,s')\beta \nabla_\omega \log \sum_{a'}\mu(a'|s')\exp\left(\frac{1}{\beta}\calQ^\omega(s',a')\right)\label{eq:_gradientofvalue_01}\\
    & = \phi(s,a) + \gamma\beta \sum_{s'}P(s,a,s')\frac{1}{\sum_{a''}\mu(a''|s')\exp\left(\frac{1}{\beta}\calQ^\omega(s',a'')\right)}  \sum_{a'}\mu(a'|s')\exp\left(\frac{1}{\beta}\calQ^\omega(s',a')\right) \nonumber\\ & \quad\quad \cdot \frac{1}{\beta}\nabla_\omega \calQ^\omega(s',a') \nonumber\\
        & = \phi(s,a) + \gamma \sum_{s'}P(s,a,a')\sum_{a'}\frac{\mu(a'|s')\exp\left(\frac{1}{\beta}\calQ^\omega(s',a')\right)}{\exp\left(\frac{1}{\beta}\calV^\omega(s')\right)}\nabla_\omega \calQ^\omega(s',a')\label{eq:gradientofvalue_02}\\
    & = \phi(s,a) + \gamma\sum_{s',a'}P(s,a,s')\pi^\textnormal{reg}_{r_\omega}(a'|s')\nabla_\omega \calQ^\omega(s',a')\nonumber\\
        & = \expp_{\tau\sim \mathbb{P}^{\pi^\textnormal{reg}_{r_\omega}}_{s,a}}\left[\phi(s,a)\right] + \gamma \expp_{\tau\sim \mathbb{P}^{\pi^\textnormal{reg}_{r_\omega}}_{s,a}}\left[ \phi(s',a')\right] + \ldots \label{eq_gradientofvalue_03}\\
    & = \expp_{\tau\sim \mathbb{P}^{\pi^\textnormal{reg}_{r_\omega}}_{s,a}}\left[ \sum^\infty_{t=0}\gamma^t \phi(s_t,a_t)\right]\nonumber\\
        & = \expp_{\tau\sim \mathbb{P}^{\pi^\textnormal{reg}_{r_\omega}}_{s,a}}\left[\phi(\tau)\right]\nonumber\\
    & = \Gamma^{\pi^\textnormal{reg}_{r_\omega}}_\phi(s,a)~,
\end{align}
where \eqref{eq:_gradientofvalue_01} and \eqref{eq:gradientofvalue_02} follow from  \eqref{eq:relation_between_v_and_q}, while  \eqref{eq_gradientofvalue_03} follows from recursion.
Similarly, for the gradient of the advantage function, we have
\begin{align*}
    \nabla_\omega \calA^\omega(s,a) & = \nabla_\omega \left( \calQ^\omega(s,a)-\calV^\omega(s)\right) \\
        & = \expp_{\tau\sim \mathbb{P}^{\pi^\textnormal{reg}_{r_\omega}}_{s,a}}\left[\phi(\tau)\right] - \nabla_\omega \beta \log\sum_{a'}\mu(a'|s)\exp\left(\frac{1}{\beta}\calQ^\omega(s,a')\right)\frac{1}{\beta}\expp_{\tau\sim \mathbb{P}^{\pi^\textnormal{reg}_{r_\omega}}_{s,a'}}\left[\phi(\tau)\right] \\
    & = \expp_{\tau\sim \mathbb{P}^{\pi^\textnormal{reg}_{r_\omega}}_{s,a}}\left[\phi(\tau)\right] - \expp_{\tau\sim \mathbb{P}^{\pi^\textnormal{reg}_{r_\omega}}_{s}}\left[\phi(\tau)\right]\\
        & = \Gamma^{\pi^\textnormal{reg}_{r_\omega}}_\phi(s,a) - \sum_{a'}\pi^\textnormal{reg}_{r_\omega}(a'|s)\Gamma^{\pi^\textnormal{reg}_{r_\omega}}_\phi(s,a')~,
\end{align*}
where, for the last equality, we have used the definition of the trajectory feature as a discounted sum of state-action feature vectors and rewritten the expectation. 
Now, let us use the short-hand notation $\Gamma^{\pi^\textnormal{reg}_{r_\omega}}_\phi(s,a)$ by $\Gamma^\omega_\phi(s,a)$, following our earlier convention, for brevity. For the final statement of our result, we rewrite the expected regularized advantage function with respect to $\pi^\dagger$ as
\begin{align}
    & \expp_{s\sim\rho,a\sim\pi^\dagger(\cdot|s)}\left[\calA^\omega(s,a)\right] = \expp_{s\sim\rho,a\sim\pi^\dagger(\cdot|s)}\left[\calQ^\omega(s,a)-\sum_{a'}\pi^\textnormal{reg}_{r_\omega}(a'|s)\calQ^\omega(s,a') \right]\nonumber\\
        & = \sum_{s,a}\rho(s)\left(\pi^\dagger(s|a)-\pi^\textnormal{reg}_{r_\omega}(a|s)\right)\calQ^\omega(s,a)\nonumber\\
    & = \sum_{s,a}\rho(s)\left(\pi^\dagger(s|a)-\pi^\textnormal{reg}_{r_\omega}(a|s)\right) \expp\left[\sum_{t\geq 0}\gamma^t\left(\omega^\top\phi(s_t,a_t)-\beta\log\frac{\pi^\textnormal{reg}_{r_\omega}(a_t|s_t)}{\mu(a_t|s_t)}\right)\Big|s_0=s,a_0=a,\pi^\textnormal{reg}_{r_\omega}\right]\nonumber\\
        & = \sum_{s,a}\rho(s)\left(\pi^\dagger(s|a)-\pi^\textnormal{reg}_{r_\omega}(a|s)\right)\nonumber\\ & \quad\quad\quad \cdot\left( \omega^\top \Gamma^\omega_\phi(s,a) - \expp\left[\sum_{t\geq 0}\gamma^t\beta \log\frac{\mu(a_t|s_t)\exp\left(\frac{1}{\beta}\calA^\omega(s_t,a_t)\right)}{\mu(a_t|s_t)} \Bigg| s_0=s,a_0=a,\pi^\textnormal{reg}_{r_\omega}\right]\right)\nonumber\\
    & = \sum_{s,a}\rho(s)\left(\pi^\dagger(s|a)-\pi^\textnormal{reg}_{r_\omega}(a|s)\right)\left( \omega^\top \Gamma^\omega_\phi(s,a) -\expp\left[\sum_{t\geq 0}\gamma^t \calA^\omega(s_t,a_t)\Big| s_0=s,a_0=a,\pi^\textnormal{reg}_{r_\omega}\right] \right)\nonumber\\
        & = \omega^\top\Gamma^\omega_\phi\left(\pi^\dagger||\pi^\textnormal{reg}_{r_\omega}\right)\nonumber \\ & \quad\quad\quad - \sum_{s,a}\rho(s)\left(\pi^\dagger(s|a)-\pi^\textnormal{reg}_{r_\omega}(a|s)\right)\expp\left[\sum_{t\geq 0}\gamma^t \sum_{a}\pi^\textnormal{reg}_{r_\omega}(a|s_t)\calA^\omega(s_t,a)\Big| s_0=s,a_0=a,\pi^\textnormal{reg}_{r_\omega}\right]\nonumber\\
    & = \omega^\top\Gamma^\omega_\phi\left(\pi^\dagger||\pi^\textnormal{reg}_{r_\omega}\right)~,\label{eq:rlhf-reg-advantage-reg}
\end{align}
where the  \eqref{eq:rlhf-reg-advantage-reg} follows from the fact that
\begin{align*}
    \sum_a \pi(a|s)\calA^\pi(s,a) = \sum_a\pi(a|s)\calQ^\pi(s,a) - \sum_a\pi(a|s)\calQ^\pi(s,a) = 0~.
\end{align*}
\end{proof}

Now, let us return to the proof of our main result. First, note that Lemma \ref{lem:advantage_gradient} implies that the constraint can be written as
\begin{align*}
    D_\textnormal{KL}(\pi^\dagger ||\mu) -\frac{1}{\beta} \omega^\top\Gamma^\omega_\phi(\pi^\dagger||\pi^\textnormal{reg}_{r_\omega}) -\epsilon' \leq 0~.
\end{align*}
As in the proof of Lemma \ref{lem:advantage_gradient}, let us denote $\Gamma^{\pi^\textnormal{reg}_{r_\omega}}_\phi(s,a)$ by $\Gamma^\omega_\phi(s,a)$. Moreover, let us reuse some definitions from the proof of Theorem \ref{thm:warmup-rlhf-augment}. Recall that
    \begin{align*}
        X^\omega_{\cleandata} = \sum_{(\tau,\tau',o)\in \cleandata}\frac{-o\left(\overline{\phi}(\tau)-\overline{\phi}(\tau')\right)}{1+\exp\left(o\omega^\top\left(\overline{\phi}(\tau)-\overline{\phi}(\tau')\right)\right)}~,
    \end{align*}
    and
    \begin{align*}
        Y^\omega_{\cleandata} & = \nabla_\omega (X^\omega_{\cleandata} + 2\lambda\omega)  \\
            & = \sum_{(\tau,\tau',o)\in \cleandata} \frac{\exp(o\omega^\top\left(\overline{\phi}(\tau)-\overline{\phi}(\tau')\right)}{\left(1+\exp(o\omega^\top\left(\overline{\phi}(\tau)-\overline{\phi}(\tau')\right)\right)^2}\left(\overline{\phi}(\tau)-\overline{\phi}(\tau')\right)\left(\overline{\phi}(\tau)-\overline{\phi}(\tau')\right)^\top +2\lambda I~.
    \end{align*}
    Then, we can equivalently write the Lagrangian as
    \begin{align*}
        \calL(\omega,\alpha) = \left|\omega^\top \left(X^\omega_{\cleandata}+\lambda\omega\right)\right| + \alpha\left( D_\textnormal{KL}(\pi^\dagger ||\mu) -\frac{1}{\beta} \expp_{s\sim\rho,a\sim\pi^\dagger(\cdot|s)}\left[\calA^\omega(s,a)\right] -\epsilon'\right)
    \end{align*}

 Using this notation, the first-order condition then implies
    \begin{align*}
        \nabla_\omega\calL(\omega,\alpha) = e_\omega\cdot\left(X^\omega_{\cleandata}+Y^\omega_{\cleandata}\omega\right) -\frac{\alpha}{\beta}\Gamma^\omega_\phi\left(\pi^\dagger||\pi^\textnormal{reg}_{r_\omega}\right) = \mathbf{0},
    \end{align*}
    where
    \begin{align*}
        e_\omega=\textnormal{sgn}\left(\omega^\top \left(X^\omega_{\cleandata}+\lambda\omega\right)\right)
    \end{align*}
    denotes the sign of the quantity inside the brackets. In the above, we have used the fact that
    \begin{align*}
        \expp_{s\sim\rho,a\sim\pi^\dagger(\cdot|s)}\left[\calA^\omega(s,a)\right] = \omega^\top \Gamma^\omega_\phi\left(\pi^\dagger||\pi^\textnormal{reg}_{r_\omega}\right)~,
    \end{align*}
    from Lemma \ref{lem:advantage_gradient}.
    Assuming $\alpha \neq 0$, we can write the primal solution in terms of the fixed-point equation
    \begin{align}\label{eq:rlhf-reg-fixedpoint-omega}
        \omega = \left(Y^{\omega}_{\cleandata}\right)^{-1}\left(\frac{e_\omega\alpha}{\beta}\Gamma^{\omega}_\phi\left(\pi^\dagger||\pi^\textnormal{reg}_{r_{\omega}}\right) - X^\omega_{\cleandata}\right)~.
    \end{align}
    Note that, if $\Gamma^{\widehat{\omega}}_\phi\left(\pi^\dagger||\pi^\textnormal{reg}_{r_{\omega}}\right)=\boldsymbol{0}$, at the solution of the fixed-point \eqref{eq:rlhf-reg-fixedpoint-omega}, then the primal condition for this case reduces to the following solution, for any $\alpha$:
    \begin{align*}
        \widehat{\omega} = - \left(Y^{\widehat{\omega}}_{\cleandata}\right)^{-1}X^{\widehat{\omega}}_{\cleandata}~,
    \end{align*}
    in which case we obtain
    \begin{align}
        \widehat{n}_\textnormal{RLHF}=\left|(\widehat{\omega})^\top \nabla_\omega\ell^\omega_\textnormal{RLHF}(\cleandata)\right| & = \left|\left(X^{\widehat{\omega}}_{\cleandata}+\lambda\widehat{\omega}\right)^\top\left(Y^{\widehat{\omega}}_{\cleandata}\right)^{-1}X^{\widehat{\omega}}_{\cleandata}\right|\nonumber\\
            & = \abs{\left(X^{\widehat{\omega}}_{\cleandata}-\lambda\left(Y^{\widehat{\omega}}_{\cleandata}\right)^{-1}X^{\widehat{\omega}}_{\cleandata}\right)^\top\left(Y^{\widehat{\omega}}_{\cleandata}\right)^{-1}X^{\widehat{\omega}}_{\cleandata}}\nonumber\\
        & \leq \norm{X^{\widehat{\omega}}_{\cleandata}}^2\norm{\left(Y^{\widehat{\omega}}_{\cleandata}\right)^{-1}}\norm{I-\left(Y^{\widehat{\omega}}_{\cleandata}\right)^{-1}}\nonumber\\
            & \leq \left(\frac{2\overline{n}}{1-\gamma}\right)^2\frac{1}{\overline{n}\sigma_{\min}(\Sigma^\phi_{\cleandata})+2\lambda}\label{eq:rlhf-reg-alpha0-case-01}\\
        & \leq O\left(\frac{\overline{n}}{(1-\gamma)^2\sigma_{\min}\left(\Sigma^\phi_{\cleandata}\right)}\right)~,\label{eq:rlhf-reg-alpha0-case-02}
    \end{align}
    where \eqref{eq:rlhf-reg-alpha0-case-01} follows from \eqref{eq:rlhf-warmup-augm-Xeq02} and Lemma \ref{lem:spectrum_of_X_and_Y}. Having taken care of this case, let us now assume that the solution to the fixed point equation implies $\Gamma^{\omega}_\phi\left(\pi^\dagger||\pi^\textnormal{reg}_{r_{\widehat{\omega}}}\right)\neq\boldsymbol{0}$.
    Using Lemma \ref{lem:advantage_gradient}, complementary slackness corresponding to the KL constraint can be written as
    \begin{align*}
        {\omega}^\top \Gamma^{\omega}_\phi\left(\pi^\dagger||\pi^\textnormal{reg}_{r_{\omega}}\right) = \beta\left(D_\textnormal{KL}\left(\pi^\dagger||\mu\right)-\epsilon'\right)~.
    \end{align*}
    Using \eqref{eq:rlhf-reg-fixedpoint-omega} above and expanding, we obtain the following result:
    \begin{align*}
         \beta\left(D_\textnormal{KL}\left(\pi^\dagger||\mu\right)-\epsilon'\right) & = \left(\left(Y^{\omega}_{\cleandata}\right)^{-1}\left(\frac{e_\omega\alpha}{\beta}\Gamma^{\omega}_\phi\left(\pi^\dagger||\pi^\textnormal{reg}_{r_{\omega}}\right) - X^{\omega}_{\cleandata}\right)\right)^\top\Gamma^{\omega}_\phi\left(\pi^\dagger||\pi^\textnormal{reg}_{r_{\omega}}\right)\\
            & = \left(\frac{e_\omega\alpha}{\beta}\Gamma^{\omega}_\phi\left(\pi^\dagger||\pi^\textnormal{reg}_{r_{\omega}}\right) - \left(X^\omega_{\cleandata}\right)\right)^\top\left(\left(Y^{\omega}_{\cleandata}\right)^{-1}\right)^\top \Gamma^{\omega}_\phi\left(\pi^\dagger||\pi^\textnormal{reg}_{r_{\omega}}\right)\\
        & = \frac{e_\omega\alpha}{\beta}\Gamma^{\omega}_\phi\left(\pi^\dagger||\pi^\textnormal{reg}_{r_{\omega}}\right)^\top\left(Y^{\omega}_{\cleandata}\right)^{-1}\Gamma^{\omega}_\phi\left(\pi^\dagger||\pi^\textnormal{reg}_{r_{\omega}}\right) - \left(X^{\omega}_{\cleandata}\right)^\top \left(Y^{\omega}_{\cleandata}\right)^{-1}\Gamma^{\omega}_\phi\left(\pi^\dagger||\pi^\textnormal{reg}_{r_{\omega}}\right)~.
    \end{align*}
    This further implies that
    \begin{align*}
        \alpha = e_\omega\cdot\frac{\beta^2\left(D_\textnormal{KL}\left(\pi^\dagger||\mu\right)-\epsilon'\right) + \beta\left(X^\omega_{\cleandata}\right)^\top \left(Y^{\omega}_{\cleandata}\right)^{-1}\Gamma^{\widehat{\omega}}_\phi\left(\pi^\dagger||\pi^\textnormal{reg}_{r_{\widehat{\omega}}}\right)}{\Gamma^{\widehat{\omega}}_\phi\left(\pi^\dagger||\pi^\textnormal{reg}_{r_{\widehat{\omega}}}\right)^\top\left(Y^{\omega}_{\cleandata}\right)^{-1}\Gamma^{\widehat{\omega}}_\phi\left(\pi^\dagger||\pi^\textnormal{reg}_{r_{\widehat{\omega}}}\right)}~,
    \end{align*}
    where we have used the fact that $Y^\omega_{\cleandata}$ is positive definite from Lemma \ref{lem:spectrum_of_X_and_Y}, so the denominator is well-defined for any non-zero $\Gamma^{\widehat{\omega}}_\phi\left(\pi^\dagger||\pi^\textnormal{reg}_{r_{\widehat{\omega}}}\right)$. Putting things together we get
    \begin{align*}
        & {\widehat{\omega}} = \left(Y^{\omega}_{\cleandata}\right)^{-1}\cdot  \left( e_{\widehat{\omega}}\cdot\frac{\beta\left(D_\textnormal{KL}\left(\pi^\dagger||\mu\right)-\epsilon'\right) + \left(X^\omega_{\cleandata}\right)^\top \Gamma^{\widehat{\omega}}_{\phi,\cleandata}\left(\pi^\dagger||\pi^\textnormal{reg}_{r_{\widehat{\omega}}}\right)}{\Gamma^{\widehat{\omega}}_\phi\left(\pi^\dagger||\pi^\textnormal{reg}_{r_{\widehat{\omega}}}\right)^\top\Gamma^{\widehat{\omega}}_{\phi,\cleandata}\left(\pi^\dagger||\pi^\textnormal{reg}_{r_{\widehat{\omega}}}\right)}\Gamma^{\widehat{\omega}}_\phi\left(\pi^\dagger||\pi^\textnormal{reg}_{r_{\widehat{\omega}}}\right) - X^\omega_{\cleandata}\right)~,
    \end{align*}
    where
    \begin{align*}
        \Gamma^{\widehat{\omega}}_{\phi,\cleandata}\left(\pi^\dagger||\pi^\textnormal{reg}_{r_{\widehat{\omega}}}\right) = \left(Y^{\omega}_{\cleandata}\right)^{-1}\Gamma^{\widehat{\omega}}_\phi\left(\pi^\dagger||\pi^\textnormal{reg}_{r_{\widehat{\omega}}}\right)~.
    \end{align*}
    Note that, as long as we choose $\gamma$ so that 
    $$1-\gamma \leq \frac{2\norm{\widehat{\omega}}}{\xi_{\max}+1}~,
    $$
    where $\widehat{\omega}$ is the solution to the above fixed-point equation, then Lemma \ref{lem:bounded_feature_norm} guarantees that the dataset construction formula satisfies the feature boundedness condition.
    
    We will now derive upper bounds on the norm of the parameter $\widehat{\omega}$ as given by the fixed-point equation above -- we will use this bound later for the final result. Observe that
    \begin{align}
        & \norm{{\widehat{\omega}}} = \nonumber\\ & \norm{\left(Y^{\widehat{\omega}}_D\right)^{-1}  \left( e_{\widehat{\omega}}\cdot\frac{\beta\left(D_\textnormal{KL}\left(\pi^\dagger||\mu\right)-\epsilon'\right) + \left(X^{\widehat{\omega}}_{\cleandata}\right)^\top \Gamma^{\widehat{\omega}}_{\phi,\cleandata}\left(\pi^\dagger||\pi^\textnormal{reg}_{r_{\widehat{\omega}}}\right)}{\Gamma^{\widehat{\omega}}_\phi\left(\pi^\dagger||\pi^\textnormal{reg}_{r_{\widehat{\omega}}}\right)^\top\Gamma^{\widehat{\omega}}_{\phi,\cleandata}\left(\pi^\dagger||\pi^\textnormal{reg}_{r_{\widehat{\omega}}}\right)}\Gamma^{\widehat{\omega}}_\phi\left(\pi^\dagger||\pi^\textnormal{reg}_{r_{\widehat{\omega}}}\right) - X^\omega_{\cleandata}\right)}\nonumber\\
            & \leq \norm{\left(Y^{\widehat{\omega}}_{\cleandata}\right)^{-1}}\norm{e_{\widehat{\omega}}\cdot\frac{\beta\left(D_\textnormal{KL}\left(\pi^\dagger||\mu\right)-\epsilon'\right) + \left(X^{\widehat{\omega}}_{\cleandata}\right)^\top \left(Y^{\widehat{\omega}}_{\cleandata}\right)^{-1}\Gamma^{\widehat{\omega}}_\phi\left(\pi^\dagger||\pi^\textnormal{reg}_{r_{\widehat{\omega}}}\right)}{\Gamma^{\widehat{\omega}}_\phi\left(\pi^\dagger||\pi^\textnormal{reg}_{r_{\widehat{\omega}}}\right)^\top\left(Y^{\widehat{\omega}}_{\cleandata}\right)^{-1}\Gamma^{\widehat{\omega}}_\phi\left(\pi^\dagger||\pi^\textnormal{reg}_{r_{\widehat{\omega}}}\right)}\Gamma^{\widehat{\omega}}_\phi\left(\pi^\dagger||\pi^\textnormal{reg}_{r_{\widehat{\omega}}}\right) - X^{\widehat{\omega}}_{\cleandata}}\label{eq:rlhf-reg-eq001}\\
        & \leq \norm{\left(Y^{\widehat{\omega}}_{\cleandata}\right)^{-1}}\nonumber\\ & \cdot\left( \frac{\left|\beta\left(D_\textnormal{KL}\left(\pi^\dagger||\mu\right)-\epsilon'\right)\right| + \norm{X^{\widehat{\omega}}_{\cleandata}}\norm{\left(Y^{\widehat{\omega}}_{\cleandata}\right)^{-1}}\norm{\Gamma^{\widehat{\omega}}_\phi\left(\pi^\dagger||\pi^\textnormal{reg}_{r_{\widehat{\omega}}}\right)}}{\norm{\Gamma^{\widehat{\omega}}_\phi\left(\pi^\dagger||\pi^\textnormal{reg}_{r_{\widehat{\omega}}}\right)}^2\sigma_{\min}\left(\left(Y^{\widehat{\omega}}_{\cleandata}\right)^{-1}\right)}\norm{\Gamma^{\widehat{\omega}}_\phi\left(\pi^\dagger||\pi^\textnormal{reg}_{r_{\widehat{\omega}}}\right)} + \norm{X^{\widehat{\omega}}_{\cleandata}}\right)\label{eq:rlhf-reg-eq002}\\
            & = \frac{\left|\beta\left(D_\textnormal{KL}\left(\pi^\dagger||\mu\right)-\epsilon'\right)\right|\norm{\left(Y^{\widehat{\omega}}_{\cleandata}\right)^{-1}}}{\norm{\Gamma^{\widehat{\omega}}_\phi\left(\pi^\dagger||\pi^\textnormal{reg}_{r_{\widehat{\omega}}}\right)}\sigma_{\min}\left(\left(Y^{\widehat{\omega}}_{\cleandata}\right)^{-1}\right)} + \frac{\norm{X^{\widehat{\omega}}_{\cleandata}}\norm{(Y^{\widehat{\omega}}_{\cleandata})^{-1}}^2}{\sigma_{\min}\left(\left(Y^{\widehat{\omega}}_{\cleandata}\right)^{-1}\right)} + \norm{X^{\widehat{\omega}}_{\cleandata}}\norm{(Y^{\widehat{\omega}}_{\cleandata})^{-1}}\label{eq:rlhf-reg-eq003}\\
        & \leq \frac{\left|\beta\left(D_\textnormal{KL}\left(\pi^\dagger||\mu\right)-\epsilon'\right)\right|\norm{\left(Y^{\widehat{\omega}}_{\cleandata}\right)^{-1}}}{\norm{\Gamma^{\widehat{\omega}}_\phi\left(\pi^\dagger||\pi^\textnormal{reg}_{r_{\widehat{\omega}}}\right)}\sigma_{\min}\left(\left(Y^{\widehat{\omega}}_{\cleandata}\right)^{-1}\right)} + \frac{2\overline{n}(\overline{n}+2(1-\gamma)\lambda)^2}{(1-\gamma)^2(\overline{n}\widetilde{C}^\phi_{\cleandata}C^\phi_{\cleandata}+2\lambda)^2} + \frac{2\overline{n}}{(1-\gamma)(\overline{n}\widetilde{C}^\phi_{\cleandata}C^\phi_{\cleandata}+2\lambda)}\label{eq:rlhf-reg-eq004}\\
            & \leq \frac{\left|\beta\left(D_\textnormal{KL}\left(\pi^\dagger||\mu\right)-\epsilon'\right)\right|}{\norm{\Gamma^{\widehat{\omega}}_\phi\left(\pi^\dagger||\pi^\textnormal{reg}_{r_{\widehat{\omega}}}\right)}}O\left( \frac{\overline{n}+2(1-\gamma)\lambda}{(1-\gamma)\left(\overline{n}\sigma_{\min}(\Sigma^\phi_{\cleandata})+2\lambda\right)}\right) + O\left(\frac{\overline{n}}{\left((1-\gamma)\sigma_{\min}(\Sigma^\phi_{\cleandata})\right)^2}\right)\label{eq:rlhf-reg-eq005}\\
        & \leq  O\left( \frac{\left|\beta\left(D_\textnormal{KL}\left(\pi^\dagger||\mu\right)-\epsilon'\right)\right|}{(1-\gamma)\sigma_{\min}(\Sigma^\phi_{\cleandata})\norm{\Gamma^{\widehat{\omega}}_\phi\left(\pi^\dagger||\pi^\textnormal{reg}_{r_{\widehat{\omega}}}\right)}} + \frac{\overline{n}}{\left((1-\gamma)\sigma_{\min}(\Sigma^\phi_{\cleandata})\right)^2}\right)~,\label{eq:rlhf-reg-eq006}
    \end{align}
    where \eqref{eq:rlhf-reg-eq001} follows from the Cauchy-Schwarz inequality; \eqref{eq:rlhf-reg-eq002} follows from the triangle inequality and Cauchy-Schwarz applied on the spectral norm; in \eqref{eq:rlhf-reg-eq003} we expand and cancel out equal terms; finally, for \eqref{eq:rlhf-reg-eq004} and \eqref{eq:rlhf-reg-eq005} we have used \eqref{eq:rlhf-warmup-augm-Xeq02}, and \eqref{eq:rlhf-warmup-augm-Ymatrix_lower} and \eqref{eq:rlhf-warmup-augm-Ymatrix_upper} from the proof of Lemma \ref{lem:spectrum_of_X_and_Y}.

    Now, in order to obtain upper bounds, we need to deal with the term in the denominator, which depends on the optimal solution $\widehat{\omega}$ itself. To address this, we will use a feasible solution to the problem. 
    
    Now, let $\widehat{\pi}^\dagger$ be a policy that is at most $\epsilon$-close to $\pi^\dagger$ and that the solution $\omega^\dagger$ to $\Phi\omega = \beta\log(\widehat{\pi}^\dagger-\mu)$ yields $\Gamma^\omega_\phi(\pi^\dagger||\pi^\textnormal{reg}_{r_{\omega^\dagger}})\neq 0$. Theorem \ref{thm:rlhf_feasibility} shows that $\omega^\dagger$ makes $\widehat{\pi}^\dagger$ optimal for the regularized problem, and thus, $\omega^\dagger$ is feasible. Using this, and the above derivations, we finally obtain:
    \begin{align*}
        \widehat{n}_\textnormal{RLHF} (\cleandata) & = \left|(\widehat{\omega})^\top\nabla_\omega\ell^{\widehat{\omega}}_\textnormal{RLHF}(\cleandata)\right| \leq \norm{\widehat{\omega}}^2 + \norm{\widehat{\omega}}\norm{X^{\widehat{\omega}}_{\cleandata}}\\
            & \quad \leq 
            O\left(\frac{\beta^2\left(D_\textnormal{KL}\left(\pi^\dagger||\mu\right)-\epsilon\right)^2}{(1-\gamma)^2\sigma^2_{\min}(\Sigma^\phi_{\cleandata})\norm{\Gamma^{\omega^\dagger}_\phi\left(\pi^\dagger||\pi^\textnormal{reg}_{r_{\omega^\dagger}}\right)}^2} + \frac{\overline{n}^2}{(1-\gamma)^4\sigma^4_{\min}(\Sigma^\phi_{\cleandata})}\right)~,
    \end{align*}
    where the first inequality uses Cauchy-Schwarz and the last inequality uses \eqref{eq:rlhf-warmup-augm-Xeq02} and \eqref{eq:rlhf-reg-eq006}.
\end{proof}

\section{Proofs of Section \ref{sec:dpo-attack}}\label{appendix:dpo-attacks}

In this section, we provide the full proofs of the results from Section \ref{sec:dpo-attack}. 

\subsection{DPO with Non-empty Existing Data}

In this section, we provide the full proof of Theorem \ref{thm:dpo-augment}.

\dpoaugment*

\begin{proof}
    Lemma \ref{lem:surrogate_connection_dpo} implies that, for any $\epsilon'\leq \epsilon/2$, any feasible solution for Problem \ref{op:dpo-aug} is feasible for Problem \ref{op:attack-dpo}. Thus, we focus on Problem \ref{op:dpo-aug}. 
    
    First, note that, since $\mu$ is loglinear, Lemma \ref{lem:loglinear_dpo_loss} implies that Problem \eqref{op:dpo-aug} can be written as
    \begin{align*}
         \min_D\; & |D| \\
        \text{s.t}\; & \widetilde{\theta} = \arg\min_{\theta} \sum_{(s,a,a',o)\in \cleandata\cup D}\log\left( 1 + \exp\left( -o\cdot \beta (\theta-\theta_\mu)^\top\left( \psi(s,a)-\psi(s,a')\right) \right)\right) + \frac{\lambda}{2}\norm{\theta-\theta_\mu}^2\\
        & \norm{\widetilde{\theta}-\theta^\dagger}^2 \leq \epsilon'~.
    \end{align*}
    Given fixed $\widetilde{\theta}\neq\boldsymbol{0}$, Lemma \ref{lem:attack_subproblem_dpo} implies that the solution to the subproblem
    \begin{align*}
        \min_D\; & |D| \\
        \text{s.t}\; & \widetilde{\theta} = \arg\min_{\theta} \sum_{(s,a,a',o)\in \cleandata\cup D}\log\left( 1 + \exp\left( -o\cdot \beta (\theta-\theta_\mu)^\top\left( \psi(s,a)-\psi(s,a')\right) \right)\right) + \frac{\lambda}{2}\norm{\theta-\theta_\mu}^2
    \end{align*}
    is the set of 
    \begin{align*}
        2\ceil{\frac{\abs{(\nabla_\theta\ell^{{\theta^\dagger}}_\textnormal{DPO}(\cleandata))^\top({\theta^\dagger}-\theta_\mu)}}{2\xi_{\max}}}
    \end{align*}
    identical samples satisfying
    \begin{align*}
        \beta\br{\theta^\dagger-\theta_\mu}^\top\br{\psi (s,a)-\psi(s,a')} = o \cdot \xi^{-1}\br{\frac{(\nabla_\theta\ell^{{\theta^\dagger}}_\textnormal{DPO}(\cleandata))^\top\br{{\theta^\dagger}-\theta_\mu}}{2\ceil{\frac{\abs{(\nabla_\theta\ell^{{\theta^\dagger}}_\textnormal{DPO}(\cleandata))^\top({\theta^\dagger}-\theta_\mu)}}{2\xi_{\max}}}}}~,
    \end{align*}
    with $o=1$ for half the samples and $o=-1$ for the remaining samples.
    Thus, the attack problem can be equivalently written in terms of $\theta$ as:
    \begin{align*}
        \min\;& \left| \left(\nabla_\theta\ell^\theta_\textnormal{DPO}(\cleandata)\right)^\top\left(\theta-\theta_\mu\right)\right| \\
        \text{s.t}\; & \norm{\theta-\theta^\dagger}^2\leq \epsilon'~.
    \end{align*}
    The Lagrangian of the above can be written as
    \begin{align*}
        \calL(\theta,\alpha) = \left| \left(\nabla_\theta\ell^\theta_\textnormal{DPO}(\cleandata)\right)^\top\left(\theta-\theta_\mu\right)\right| + \alpha\left(\norm{\theta-\theta^\dagger}^2- \epsilon''\right)~,
    \end{align*}
    and the first-order condition is
    \begin{align*}
        \nabla_\theta\calL(\theta,\alpha) = e_\theta\cdot\left(\nabla^2_\theta\ell^\theta_\textnormal{DPO}(\cleandata)\left(\theta -\theta_\mu\right) + \nabla_\theta\ell^\theta_\textnormal{DPO}(\cleandata)\right) + 2\alpha\left(\theta-\theta^\dagger\right) = \mathbf{0}~,
    \end{align*}
    where 
    \begin{align*}
        e_\theta = \textnormal{sgn}\left(\left(\nabla_\theta\ell^\theta_\textnormal{DPO}(\cleandata)\right)^\top\left(\theta-\theta_\mu\right)\right) 
    \end{align*}
    denotes the sign of the quantity inside the brackets. This yields
    \begin{align*}
        \theta^* = \left(\nabla^2_\theta\ell^{\theta^*}_\textnormal{DPO}(\cleandata) + 2\alpha I\right)^{-1}\left( e_{\theta^*}\nabla^2_\theta\ell^{\theta^*}_\textnormal{DPO}(\cleandata) \theta_\mu -e_{\theta^*}\nabla_\theta\ell^{\theta^*}_\textnormal{DPO}(\cleandata) + 2\alpha\theta^\dagger\right)~,
    \end{align*}
    and
    \begin{align*}
        \theta^*-\theta_\mu = \left(\nabla^2_\theta\ell^{\theta^*}_\textnormal{DPO}(\cleandata)\right)^{-1}\left(2e_{\theta^*}\alpha\left(\theta^*-\theta^\dagger\right) - \nabla_\theta\ell^{\theta^*}_\textnormal{DPO}(\cleandata)\right)~.
    \end{align*}
    For non-zero $\alpha$, we use complementary slackness to obtain
    \begin{align*}
        \norm{\left(\nabla^2_\theta\ell^{\theta^*}_\textnormal{DPO}(\cleandata) + 2\alpha I\right)^{-1}\left( e_{\theta^*}\nabla^2_\theta\ell^{\theta^*}_\textnormal{DPO}(\cleandata) \theta_\mu -e_{\theta^*}\nabla_\theta\ell^{\theta^*}_\textnormal{DPO}(\cleandata) + 2\alpha\theta^\dagger\right) -\theta^\dagger}^2 = \epsilon'~.
    \end{align*}
    Before we proceed any further, note that, since the Hessian of the loss is symmetric positive definite, then
    \begin{align*}
        \nabla^2_\theta\ell^{\theta^*}_\textnormal{DPO}(\cleandata) = U \Sigma U~,
    \end{align*}
    where $U$ is an orthonormal matrix and $\Sigma$ is the diagonal matrix with the eigenvalues of the Hessian. Using the Woodbury inversion formula, we have
    \begin{align*}
        \left(\nabla^2_\theta\ell^{\theta^*}_\textnormal{DPO}(\cleandata) + 2\alpha I\right)^{-1} & = \frac{1}{2\alpha}I - \frac{1}{4\alpha^2}U\left(\Sigma^{-1}+\frac{1}{2\alpha}U^\top UI\right)^{-1}U\\
            & = \frac{1}{2\alpha}I - \frac{1}{4\alpha^2}U\left(\Sigma^{-1}+\frac{1}{2\alpha}I\right)^{-1}U\\
        & = \frac{1}{2\alpha} \left(I - U\Sigma_\alpha U\right)~,
    \end{align*}
    where $\Sigma_\alpha$ is a diagonal matrix with entries $\sigma_i/(\sigma_i+2\alpha)$, where $\sigma_i$ are the entries of $\Sigma$, for every $1\leq i\leq d$. Let $M_\alpha = I-U\Sigma_\alpha U$, and 
    \begin{align*}
        g = e_{\theta^*}\nabla^2_\theta\ell^{\theta^*}_\textnormal{DPO}(\cleandata) \theta_\mu -e_{\theta^*}\nabla_\theta\ell^{\theta^*}_\textnormal{DPO}(\cleandata)~.
    \end{align*}
    Then, we have
    \begin{align*}
        4\alpha^2\epsilon' & =  \norm{\left(I-U\Sigma_\alpha U\right)\left( e_{\theta^*}\nabla^2_\theta\ell^{\theta^*}_\textnormal{DPO}(\cleandata) \theta_\mu -e_{\theta^*}\nabla_\theta\ell^{\theta^*}_\textnormal{DPO}(\cleandata) + 2\alpha\theta^\dagger\right) -2\alpha\theta^\dagger}^2 \\
            & = \norm{M_\alpha g + 2\alpha\left(M_\alpha\theta^\dagger -\theta^\dagger\right)}^2\\
        & = \norm{M_\alpha g}^2 + 4\alpha\left\langle M_\alpha g, M_\alpha \theta^\dagger-\theta^\dagger\right\rangle + 4\alpha^2\norm{M_\alpha\theta^\dagger-\theta^\dagger}^2~.
    \end{align*}
    We can write the above as a quadratic equation of $\alpha$ as
    \begin{align*}
        4\left(\norm{M_\alpha\theta^\dagger-\theta^\dagger}^2-\epsilon'\right)\alpha^2 + 4\left\langle M_\alpha g, M_\alpha \theta^\dagger-\theta^\dagger\right\rangle\alpha +  \norm{M_\alpha g}^2  = 0~.
    \end{align*}
    Note that we are treating as constants some terms that involve $\alpha$ non-linearly. This implies fixed-point solutions in terms of $\alpha$. We follow this route since we are only interested in the final bounds, which will be independent of such components. 
    The discriminant of this equation is
    \begin{align*}
        \Delta & = 16\norm{M_\alpha g}^2\norm{M_\alpha \theta^\dagger-\theta^\dagger}^2 - 16\left(\norm{M_\alpha\theta^\dagger-\theta^\dagger}^2-\epsilon'\right)\norm{M_\alpha g}^2 = 16\norm{M_\alpha g}^2\epsilon'~.
    \end{align*}
    This implies the following fixed-point equations:
    \begin{align*}
        \alpha_{1,2} = \frac{-\left\langle M_\alpha g, M_\alpha \theta^\dagger-\theta^\dagger\right\rangle \pm \norm{M_\alpha g}\sqrt{\epsilon'}}{2\left(\norm{M_\alpha\theta^\dagger-\theta^\dagger}^2-\epsilon'\right)}~.
    \end{align*}
    Note that we have
    \begin{align*}
        \frac{-\left\langle M_\alpha g, M_\alpha \theta^\dagger-\theta^\dagger\right\rangle - \norm{M_\alpha g}\sqrt{\epsilon'}}{2\left(\norm{M_\alpha\theta^\dagger-\theta^\dagger}^2-\epsilon'\right)} & \geq \frac{-\norm{M_\alpha g}\left(\norm{M_\alpha\theta^\dagger-\theta^\dagger}+\sqrt{\epsilon'}\right)}{2\left(\norm{M_\alpha\theta^\dagger-\theta^\dagger}^2-\epsilon'\right)} = \frac{\norm{M_\alpha g}}{2\left(\sqrt{\epsilon'}-\norm{M_\alpha\theta^\dagger-\theta^\dagger}\right)}
    \end{align*}
    which, since the maximum eigenvalue of $M_\alpha$ is bounded by $1$, is positive whenever
    \begin{align*}
        \norm{M_\alpha\theta^\dagger-\theta^\dagger} \leq 2\norm{\theta^\dagger}\leq \sqrt{\epsilon'}~.
    \end{align*}
    This condition would guarantee a positive root of $\alpha$, which is required as $\alpha$ is a Lagrange multiplier. If there is none, we do not take into consideration complementary slackness and solve for $\alpha=0$. For now, let us assume that there exists a positive root and consider the $\alpha=0$ case later. Since the final bounds do not depend on which of the solutions we pick, let us pick the one with $+ \norm{M_\alpha g}\sqrt{\epsilon'}$ without loss of generality.
    
    Plugging it into the first-order solution, we get
    \begin{align}\label{eq:fixed-point-dpo-dagger}
        \theta^* = \frac{1}{2\alpha}M_\alpha\left(g + 2\alpha\theta^\dagger\right) = M_\alpha\theta^\dagger + \frac{1}{2\alpha}M_\alpha g = M_\alpha \theta^\dagger + \frac{\left(\norm{M_\alpha\theta^\dagger-\theta^\dagger}^2-\epsilon'\right)}{\norm{M_\alpha g}\sqrt{\epsilon'} -\left\langle M_\alpha g, M_\alpha \theta^\dagger-\theta^\dagger\right\rangle}M_\alpha g~.
    \end{align}
    This implies that
    \begin{align*}
        \norm{\theta^*} & \leq \norm{M_\alpha\theta^\dagger} + \left|\frac{\norm{M_\alpha\theta^\dagger-\theta^\dagger}^2-\epsilon'}{\sqrt{\epsilon'} -\left\langle M_\alpha g, M_\alpha \theta^\dagger-\theta^\dagger\right\rangle/\norm{M_\alpha g}}\right|\\
            &\leq \norm{M_\alpha\theta^\dagger} + \frac{\left|\norm{M_\alpha\theta^\dagger-\theta^\dagger}^2-\epsilon'\right|}{\left|\sqrt{\epsilon'} -\norm{M_\alpha\theta^\dagger-\theta^\dagger}\right|}\\
        & = \norm{M_\alpha\theta^\dagger} + \norm{M_\alpha\theta^\dagger-\theta^\dagger} + \sqrt{\epsilon'}\\
            & \leq \norm{\theta^\dagger} + 2\norm{\theta^\dagger} + \sqrt{\epsilon'}\\
        & = 3\norm{\theta^\dagger} + \sqrt{\epsilon'}~,
    \end{align*}
    where we have used the triangle inequality, Cauchy Schwarz and the fact that the maximum eigenvalue of $M_\alpha$ is upper-bounded by $1$.
    Now we provide upper bounds on the gradient with respect to the pre-existing data for $\theta^*$, similar to the proof of Theorem \ref{thm:warmup-rlhf-augment}. Let $\overline{\theta}$ be the optimal parameter with respect to the loss on the pre-existing dataset. 
    Note that we have
    \begin{align*}
        \norm{\nabla_\theta\ell^{\theta^*}_\textnormal{DPO}(\cleandata)} & \leq (\overline{n}\beta + \lambda)\norm{\theta^*-\overline{\theta}}\\
            & \leq (\overline{n}\beta + \lambda)\left(\norm{\theta^*-\proj_{\theta:\norm{\theta-\theta^\dagger}\leq\epsilon'}\left(\overline{\theta} \right)} + \norm{\proj_{\theta:\norm{\theta-\theta^\dagger}\leq\epsilon'}\left(\overline{\theta} \right)-\overline{\theta}}\right)\\
        & \leq  (\overline{n}\beta + \lambda)\left(\epsilon' + \norm{\proj_{\theta:\norm{\theta-\theta^\dagger}\leq\epsilon'}\left(\overline{\theta} \right)-\overline{\theta}}\right)~,
    \end{align*}
    where the first inequality follows from Lemma \ref{lem:logistic_regression_properties}; the second inequality uses triangle inequality and the third inequality uses the fact that any two points in the $\epsilon'$-ball are no farther than $\epsilon'$ away from each-other.
    To compute the projection onto the $\epsilon'$-ball, we solve the following problem:
    \begin{align*}
        \min_\theta \norm{\theta-\overline{\theta}}^2,\;\;\textnormal{such that}\;\; \norm{\theta-\theta^\dagger}^2\leq(\epsilon')^2~.
    \end{align*}
    The first-order of the problem with respect to dual variable $\alpha$ of the Lagrangian 
    \begin{align*}
        \calL(\theta,\alpha) =  \norm{\theta-\overline{\theta}}^2 + \alpha\left(\norm{\theta-\theta^\dagger}^2-(\epsilon')^2\right)
    \end{align*}
    gives us
    \begin{align*}
        \theta = \frac{1}{1+\alpha}\left(\alpha\theta^\dagger+\overline{\theta}\right)~.
    \end{align*}
    Complementary slackness implies that
    \begin{align*}
        (1+\alpha)(\epsilon')^2 = \norm{\overline{\theta}-\theta^\dagger}^2
    \end{align*}
    which in turn implies that we have
    \begin{align*}
        \widetilde{\theta}=\theta^\dagger + \frac{(\epsilon')^2}{\norm{\overline{\theta}-\theta^\dagger}^2}\left(\overline{\theta}-\theta^\dagger\right)~.
    \end{align*}
    
    Thus, we have 
    \begin{align*}
          \widehat{n}_\textnormal{DPO}& = \left| \left(\nabla_\theta\ell^{\widetilde{\theta}}_\textnormal{DPO}(\cleandata)\right)^\top\left(\widetilde{\theta}-\theta_\mu\right)\right| \\ & \leq \norm{\nabla_\theta\ell^{\widetilde{\theta}}_\textnormal{DPO}(\cleandata)}\left( 3\norm{\theta^\dagger} + \norm{\theta_\mu} + \sqrt{\epsilon'}\right)\\
            & = \left(\overline{n}\beta+\lambda\right)\norm{\theta^\dagger + \frac{(\epsilon')^2}{\norm{\overline{\theta}-\theta^\dagger}^2}\left(\overline{\theta}-\theta^\dagger\right)-\overline{\theta}}\left( 3\norm{\theta^\dagger} + \norm{\theta_\mu} + \sqrt{\epsilon'}\right)\\
        & \leq (\overline{n}\beta+\lambda)\frac{\abs{\norm{\overline{\theta}-\theta^\dagger}^2-(\epsilon')^2}}{\norm{\theta^\dagger-\overline{\theta}}}\left( 3\norm{\theta^\dagger} + \norm{\theta_\mu} + \sqrt{\epsilon'}\right)~.
    \end{align*}

    If $\alpha=0$, then the first-order condition yields
    \begin{align*}
        \widetilde{\theta} -\theta_\mu = -\left(\nabla^2_\theta\ell^{\widetilde{\theta}}_\textnormal{DPO}(\cleandata)\right)^{-1}\nabla_\theta\ell^{\widetilde{\theta}}_\textnormal{DPO}(\cleandata)~,
    \end{align*}
    which implies 
    \begin{align*}
        \widehat{n}_\textnormal{DPO} & = \left|\nabla_\theta\ell^{\widetilde{\theta}}_\textnormal{DPO}(\cleandata)^\top  \left(\nabla^2_\theta\ell^{\widetilde{\theta}}_\textnormal{DPO}(\cleandata)\right)^{-1}\nabla_\theta\ell^{\widetilde{\theta}}_\textnormal{DPO}(\cleandata)\right| \\
            & \leq \frac{\left(\overline{n}\beta +\lambda\right)}{\sigma_\textnormal{min}(\Sigma^\psi_{\cleandata})}\cdot\frac{\abs{\norm{\overline{\theta}-\theta^\dagger}^2-(\epsilon')^2}}{\norm{\theta^\dagger-\overline{\theta}}}~,
    \end{align*}
    using Lemma \ref{lem:logistic_regression_properties} and the projection derivations above.
    
    For the lower bound, note that, in the best case scenario for the attacker, $\overline{D}$ is a subset of an optimal solution $\widehat{D}$ for Problem \ref{op:attack-dpo} in the case when $\overline{D}=\emptyset$. Thus, the attack sample size in this case is the lower bound of Theorem \ref{thm:dpo-gen-dagger} without the size $\overline{n}$ of $\overline{D}$.
\end{proof}

\subsection{DPO with Empty Existing Data}\label{sec:dpo_empty_data}

In this section, we provide additional upper bounds on the sample complexity for the DPO setting when $\cleandata=\emptyset$. Our aim is to obtain bounds that are tighter than the ones obtained by directly instantiating the bounds of Theorem \ref{thm:dpo-augment}. Moreover, we also provide lower bounds for this setting and use them for the general lower bounds of Theorem \ref{thm:dpo-augment}.

\dpogenerationdagger

\begin{proof}
    First, note that, since $\mu$ is loglinear, Lemma \ref{lem:loglinear_dpo_loss} implies that Problem \eqref{op:dpo-aug} can be written as
    \begin{align*}
         \min_D\; & |D| \\
        \text{s.t}\; & \widetilde{\theta} = \arg\min_{\theta} \sum_{(s,a,a',o)\in D}\log\left( 1 + \exp\left( -o\cdot \beta (\theta-\theta_\mu)^\top\left( \psi(s,a)-\psi(s,a')\right) \right)\right) + \frac{\lambda}{2}\norm{\theta-\theta_\mu}^2\\
        & \norm{\widetilde{\theta}-\theta^\dagger}^2 \leq \epsilon'~.
    \end{align*}
    Now, given $\widetilde{\theta}\neq\boldsymbol{0}$, Lemma \ref{lem:attack_subproblem_dpo} implies that the solution to the problem
    \begin{align*}
        \min_D\; & |D| \\
        \text{s.t}\; & \widetilde{\theta} = \arg\min_{\theta} \sum_{(s,a,a',o)\in D}\log\left( 1 + \exp\left( -o\cdot \beta (\theta-\theta_\mu)^\top\left( \psi(s,a)-\psi(s,a')\right) \right)\right) + \frac{\lambda}{2}\norm{\theta-\theta_\mu}^2
    \end{align*}
    is the set of 
    \begin{align*}
        2\ceil{\frac{\lambda\norm{\widetilde{\theta}-\theta_\mu}^2}{2\xi_{\max}}}
    \end{align*}
    identical samples satisfying
    \begin{align*}
        \beta\br{{\widetilde{\theta}}-\theta_\mu}^\top\br{\psi (s,a)-\psi(s,a')} = o \cdot \xi^{-1}\br{\frac{\lambda\norm{\widetilde{\theta}-\theta_\mu}^2}{2\ceil{\frac{\lambda\norm{\widetilde{\theta}-\theta_\mu}^2}{2\xi_{\max}}}}}~,
    \end{align*}
    with $o=1$ for half the samples and $o=-1$ for the remaining samples. 
        
    Next, for the upper bound, we will consider the surrogate problem. Using Lemma \ref{lem:surrogate_connection_dpo}, for any $\epsilon'\leq\epsilon/(2\sqrt{d'})$, any feasible solution for Problem \ref{op:dpo-aug} is feasible for Problem \ref{op:attack-dpo}. Thus, we focus on the former. We can rewrite the problem directly in terms of a variable $\theta$:
    \begin{align*}
        \min_\theta\; \norm{{\theta}-\theta_\mu}^2 \;\; \textnormal{such that} \;\; \norm{\theta-\theta^\dagger}^2 \leq \epsilon'~.
    \end{align*}
    The Lagrangian of the above can be written as
    \begin{align*}
        \calL(\theta,\alpha) = \norm{{\theta}-\theta_\mu}^2 + \alpha\left(\norm{\theta-\theta^\dagger}^2 - \epsilon'\right)~.
    \end{align*}
    The first-order condition can be written as
    \begin{align*}
        \nabla_\theta\calL(\theta,\alpha) & = \left(\theta-\theta_\mu\right) + \alpha\left(\theta-\theta^\dagger\right) = \boldsymbol{0}~,
    \end{align*}
    which implies that 
    \begin{align*}
        \widetilde{\theta} = \frac{1}{1+\alpha}\left(\alpha\theta^\dagger+\theta_\mu\right)~.
    \end{align*}
    Complementary slackness implies
    \begin{align*}
        \norm{\frac{1}{1+\alpha}\left(\alpha\theta^\dagger+\theta_\mu\right)-\theta^\dagger} = \sqrt{\epsilon'}~.
    \end{align*}
    Equivalently,
    \begin{align*}
        \norm{\theta^\dagger-\theta_\mu}=\abs{1+\alpha}\sqrt{\epsilon'}~,
    \end{align*}
    which, because all three terms are non-negative, implies that
    \begin{align*}
        \alpha = \frac{\norm{\theta^\dagger-\theta_\mu}}{\sqrt{\epsilon'}}-1~,
    \end{align*}
    Plugging this into the first-order condition and subtracting both sides by $\theta_\mu$, we get
    \begin{align*}
        \widetilde{\theta}-\theta_\mu & = \frac{\sqrt{\epsilon'}}{\norm{\theta^\dagger-\theta_\mu}}\left(\left(\frac{\norm{\theta^\dagger-\theta_\mu}}{\sqrt{\epsilon'}}-1\right)\theta^\dagger+\theta_\mu\right) -\theta_\mu = \theta^\dagger-\theta_\mu - \frac{\sqrt{\epsilon'}}{\norm{\theta^\dagger-\theta_\mu}}\left(\theta^\dagger-\theta_\mu\right) \\ & = \left(\frac{\sqrt{\epsilon'}}{\norm{\theta^\dagger-\theta_\mu}}-1\right)\left(\theta^\dagger-\theta_\mu\right)
    \end{align*}
    Using the above and some algebraic manipulations, we can write
    \begin{align*}
        \norm{\widetilde{\theta}-\theta_\mu}^2 & = \norm{ \left(\frac{\sqrt{\epsilon'}}{\norm{\theta^\dagger-\theta_\mu}}-1\right)\left(\theta^\dagger-\theta_\mu\right)}^2 = \norm{\norm{\theta^\dagger-\theta_\mu}-\sqrt{\epsilon'}}^2~.
    \end{align*}
    The desired bounds follow as an immediate consequence of the above. 
\end{proof}

\section{Proofs of Section \ref{sec:comparison-of-paradigms}}\label{appendix:comparison-generation}

In this section, we provide the proof of the following result. 

\comparisonwithdata*

\begin{proof}
Note that we have
\begin{align*}
    \widehat{n}_\textnormal{DPO} & \geq 2\ceil{\frac{\lambda}{2\xi_{\max}}\left(\norm{\theta^\dagger-\theta_\mu}-\sqrt{\epsilon'}\right)^2} -\overline{n}
    \\ 
    & \geq \frac{\lambda}{\xi_{\max}}\left(\norm{\theta^\dagger-\theta_\mu}-\sqrt{\epsilon'}\right)^2 -\overline{n} \\
    & =  \left(\frac{\lambda}{\xi_{\max}}\left(\norm{\theta^\dagger-\theta_\mu}-\sqrt{\epsilon'}\right)^2 -\overline{n}\right)\cdot  \ceil{\frac{1}{\xi_{\max}}\left(\lambda\norm{\omega^\dagger}^2 + \norm{\omega^\dagger}\cdot\frac{2\overline{n}}{(1-\gamma)^2}\right)}^{-1} \\ & \quad\quad \quad \cdot \ceil{\frac{1}{\xi_{\max}}\left(\lambda\norm{\omega^\dagger}^2 + \norm{\omega^\dagger}\cdot\frac{2\overline{n}}{(1-\gamma)^2}\right)} \\
    & \geq \kappa_1\cdot\widehat{n}_\textnormal{RLHF}~,
\end{align*}
where the first inequality follows from Theorem \ref{thm:dpo-lower-bound} and the first inequality follows from Corollary \ref{cor:rlhf-cor-1}.
\end{proof}

\section{Technical Lemmas}\label{appendix:technical_lemmas}

This section includes miscellaneous technical results used throughout the paper. We begin by stating a result about the structure of the regularized optimal policy.

\begin{lemma}\label{lem:gradient_rlhf_mdp}
    Given policy $\pi$ and reward function $r$, let 
    \begin{align*}
        \calV^{\pi}_r(s) = \expp_{s\sim\rho, a_t\sim\pi(\cdot|s_t)}\left[ \sum_{t\geq 0}\gamma^t \left(r(s_t,a_t) -\beta\log\frac{\pi(a_t|s_t)}{\mu(a_t|s_t)}\right)  \right]
    \end{align*}
    and
    \begin{align*}
        \calQ^{\pi}_r(s,a) = r(s,a) + \gamma \sum_{s'}P(s,a,s')\calV^{\pi}_r(s')
    \end{align*}
    denote the regularized value and action value functions with respect to $\pi$ and $r$, respectively. Moreover, let the regularized advantage function be defined as
    \begin{align*}
        \calA^{\pi}_r(s,a) = \calQ^{\pi}_r(s,a)-\calV^{\pi}_r(s)~.
    \end{align*}
    Then, the unique optimal regularized policy can be written, for every $(s,a)$, as
    \begin{align*}
        \pi^\textnormal{reg}_r(a|s) = \mu(a|s)\exp\left(\frac{1}{\beta}\calA^{\pi^\textnormal{reg}_r(s,a)}\right)~.
    \end{align*}
\end{lemma}
\begin{proof}
    The proof of this result is a straightforward application of the results from Appendix C of \citep{nachum2017bridging} from the entropy-based to KL divergence-based regularization. 
\end{proof}

Now we prove that $M_{\pi^\dagger}$ is full rank. 
\begin{lemma}\label{lem:M_full_rank}
    Let $\pi^\dagger\in\Pi^\textnormal{det}$ and $\Phi$ be of rank $d$. Assume that the set $\{\phi(s,a):s\in\mathcal{S}, a\in\mathcal{A}\setminus\textnormal{supp}(\pi^\dagger)\}$, where $\textnormal{supp}(\pi)$ denotes the set of actions chosen by deterministic policy $\pi$, contains $d$ linearly independent vectors. Then, the matrix $M_{\pi^\dagger}$ is full rank.
\end{lemma}
\begin{proof}
    Let $v$ be an arbitrary column of $M_{\pi^\dagger}$. Then, there exists $(s,a)$ and coefficients $\alpha_{s',a'}$ such that 
    \begin{align*}
        v = \sum_{s'} d^{\pi^\dagger}(s')\phi(s',\pi^\dagger(s')) - \sum_{s'} d^{\pi^\dagger\{s,a\}}(s')\phi(s',\pi^\dagger\{s,a\}(s')) = \sum_{s',a'}\alpha_{s',a'}\phi(s',a')~.
    \end{align*}
    By definition of the neighbors of $\pi^\dagger$, we have that $a=\pi^\dagger \{s,a\}(s)\neq \pi^\dagger(s) = \pi(s)$, for all $\pi\in \mathcal{N}(\pi^\dagger) \setminus \{ \pi\{ s,a\}\}$. This is because all neighboring policies are deterministic and change only in one state from $\pi^\dagger$, and, consequently,  from each-other. 

    This implies that, there is no column $v'$ in $M_{\pi^\dagger}$, such that $\phi(s,a)$ appears in the decomposition of $v'$. There are $S(A-1)$ such vectors, since there are $S(A-1)$ neighbors of $\pi^\dagger$. Assuming that they contain all vectors that span the column space of $\Phi$, this means that the rank of $M_{\pi^\dagger}$ is equal to the rank of $\Phi$. 
\end{proof}

Next, we provide three results that connect the solutions of the surrogate problems to their original problems throughout the paper. We start with the unregularized RLHF setting.
\begin{lemma}\label{lem:surrogate_connection_unreg_rlhf}
    Let $\epsilon' >0$. Then, any feasible solution to Problem \ref{op:syn-rlhf-augment} is a feasible solution to Problem \ref{op:rlhf-attack}.
\end{lemma}
\begin{proof}
    First, note that, when $\beta=0$, given reward $r$, then $\pi^\textnormal{reg}_r(\cdot|s)\in\arg\max_\pi V^\pi_r(s)$, for all $s$. Such optimal policies are known to be deterministic. Now, given $\widehat{\omega}$, note that, if $\epsilon'>0$, then $V^{\pi^\dagger}_{r_{\widehat{\omega}}}(s) > V^\pi_{r_{\widehat{\omega}}}(s)$, for any state $s$. This means that $\pi^\dagger$ is the unique optimal policy  under $r_{\widehat{{\omega}}}$. This further implies that $\pi^\textnormal{reg}_{r_{\widehat{{\omega}}}}=\pi^\dagger$, and thus, $\KL(\pi^\dagger||\pi^\textnormal{reg}_{r_{\widehat{{\omega}}}})=0 < \epsilon$, for any $\epsilon >0$.
\end{proof}
Next, we consider the regularized RLHF setting.
\begin{lemma}\label{lem:surrogate_connection_reg_rlhf}
    Let $\epsilon'\leq (2\ln 2)\epsilon$. Then, any feasible solution to Problem \ref{op:reg-rlhf-attack} is a feasible solution to Problem \ref{op:rlhf-attack}.
\end{lemma}
\begin{proof}
    Given two policies $\pi$ and $\pi'$, we have
    \begin{align*}
        \norm{\pi-\pi'}^2_1 & =
        \sum_{s,a}\rho(s)\abs{\pi(\cdot|s)-\pi'(\cdot|s)}^2\\
            & \leq \sum_{s}\rho(s) 2\ln 2 \KL\left(\pi(\cdot|s)||\pi'(\cdot|s)\right)\\
            & = (2\ln 2) \KL(\pi||\pi')~,
    \end{align*}
    where 
    the first inequality follows from Pinsker's inequality. This implies that, as long as $\KL(\pi||\pi')\leq \epsilon'$, then $\norm{\pi-\pi'}^2_1\leq (2\ln 2)\epsilon' \leq \epsilon$.
\end{proof}
Finally, we consider the DPO setting. 
\begin{lemma}\label{lem:surrogate_connection_dpo}
    Let $\epsilon'\leq \epsilon/(2\sqrt{d'})$. Then, any feasible solution to Problem \ref{op:dpo-aug} is a feasible solution to Problem \ref{op:attack-dpo}. Moreover, there exists $\eta_{\min}>0$ such that, for all $\epsilon' \geq \epsilon/\eta_{\min}$, any feasible solution to Problem \ref{op:attack-dpo} is also a feasible solution to Problem \ref{op:dpo-aug}.
\end{lemma}
\begin{proof}
    First, it can be easily shown that the gradient of a loglinear policy is 
    \begin{align*}
        \nabla_\theta \pi_\theta(a|s) = \pi_\theta(a|s)\left(\psi(s,a)-\expp_{a'\sim\pi_\theta(\cdot|s)}\left[\psi(s,a')\right]\right).
    \end{align*}
    We can bound this gradient by $\norm{\nabla_\theta\pi_\theta(a|s)}\leq 2$. Thus, loglinear policies are $2$-Lipschitz in their parameters $\theta$. Now, given policies $\pi_\theta$ and $\pi_{\theta'}$, due to the above argument, we can write
    \begin{align*}
        \norm{\pi_\theta-\pi_{\theta'}}_1\leq 2\norm{\theta-\theta'}\leq 2\epsilon' \leq \epsilon.
    \end{align*}
    This implies that, for any $\epsilon'\leq \epsilon/2$, any feasible solution to Problem \ref{op:dpo-aug} will be feasible for Problem \ref{op:attack-dpo}.

    For the second statement, we argue as follows. Since the function $\pi_\theta$ is continuously differentiable in $\theta$, the Mean Value Theorem implies that, for any $\theta$ and $(s,a)$, there exists $\theta_M(s,a)$ such that 
    \begin{align}\label{eq:mean_value}
        \pi_\theta(a|s) - \pi_{\theta^\dagger}(a|s) = \nabla_\theta\pi_{\theta_M(s,a)}(a|s)^\top\left(\theta-\theta^\dagger\right) = \pi_{\theta_M(s,a)}(a|s)\overline{\psi}_{\theta_M(s,a)}(s,a)^\top\left(\theta-\theta^\dagger\right)~,
    \end{align}
    where
    \begin{align*}
        \overline{\psi}_\theta(s,a) = \psi(s,a) - \sum_{a'}\pi_\theta(a'|s)\psi(s,a')~.
    \end{align*}
    Let $\nabla_\theta\pi_{\theta_M}$ be the $S\cdot A$-dimensional matrix with columns $\rho(s)\pi_{\theta_M(s,a)}(a|s)\overline{\psi}_{\theta_M(s,a)}(s,a)$, and let $\pi$ denote the $S\cdot A$-dimensional vector with entries $\rho(s)\pi(a|s)$, for each $(s,a)$. Then, we have
    \begin{align*}
        \norm{\pi_\theta - \pi_{\theta^\dagger}}_1 \geq \norm{\pi_\theta - \pi_{\theta^\dagger}} = \norm{\nabla_\theta\pi_{\theta_M}\left(\theta-\theta^\dagger\right)} \geq \sigma_{\min}\left(\nabla_\theta\pi_{\theta_M}\right)\norm{\theta-\theta^\dagger}~,
    \end{align*}
    where the first inequality follows from the relationship between $\ell_1$ and $\ell_2$ norms; the first equality follows from \eqref{eq:mean_value} and the second inequality follows from the fact that $\norm{Ax}_2\geq\sigma_{\min}(A)\norm{x}_2$, for compatible matrix $A$ and vector $x$, where $\sigma_{\min}(A)$ denotes the minimum singular value of $A$. Now, observe that
    \begin{align*}
        \nabla_\theta\pi_{\theta_M}^\top\nabla_\theta\pi_{\theta_M} = \sum_{s,a}\rho(s)^2\pi_{\theta_M(s,a)}(a|s)^2\overline{\psi}_{\theta_M(s,a)}(s,a)\overline{\psi}_{\theta_M(s,a)}(s,a)^\top~.
    \end{align*}
    We will show that this matrix is positive definite, whenever the vectors $\psi(s,a)$ span $\mathbb{R}^{d'}$. First, the ergodicity assumption and the loglinearity of the policies imply that $\rho(s)\pi_\theta(a|s) >0$, for any $(s,a)$ and $\theta$. 
    
    Next, observe that, since vectors $\psi(s,a)$ span $\mathbb{R}^{d'}$, then vectors $\overline{\psi}_{\theta_M(s,a)}(s,a)$ also do, as translations of basis vectors via mean vectors. It is clear that $\nabla_\theta\pi_{\theta_M}^\top\nabla_\theta\pi_{\theta_M}$ is positive semi-definite, meaning that, for every non-zero vector $v$, 
    \begin{align*}
        v^\top\left(\nabla_\theta\pi_{\theta_M}^\top\nabla_\theta\pi_{\theta_M}\right) v \geq 0~.
    \end{align*}
    The only way the above can be zero is if, for every $(s,a)$, we have
    \begin{align*}
        v^\top\overline{\psi}_{\theta_M(s,a)}(s,a) = 0~.
    \end{align*}
    But since $\overline{\psi}_{\theta_M(s,a)}(s,a)$ span $\mathbb{R}^{d'}$, then there exist coefficients $\alpha_{s,a}$ such that 
    \begin{align*}
        v = \sum_{s,a}\alpha_{s,a}\overline{\psi}_{\theta_M(s,a)}(s,a)~,
    \end{align*}
    which implies that
    \begin{align*}
        v^\top v = \sum_{s,a}\alpha_{s,a}\overline{\psi}_{\theta_M(s,a)}(s,a)^\top v = 0~,
    \end{align*}
    which is a contradiction, since $v$ is assumed to be non-zero. Thus, the matrix $\nabla_\theta\pi_{\theta_M}^\top\nabla_\theta\pi_{\theta_M}$ is positive definite, for any given $\theta$. This means that
    \begin{align*}
        \sigma_{\min}\left(\nabla_\theta\pi_{\theta_M}\right) = \sqrt{\lambda_{\min}\left(\nabla_\theta\pi_{\theta_M}^\top\nabla_\theta\pi_{\theta_M}\right)} = \eta_{\min}(\theta) \geq \min_\theta \eta_{\min}(\theta) := \eta_{\min} > 0~,
    \end{align*}
    which further implies that
    \begin{align*}
         \norm{\theta-\theta^\dagger} \leq \frac{1}{\eta_{\min}}\norm{\pi_\theta-\pi_{\theta^\dagger}} \leq \frac{1}{\eta_{\min}}\norm{\pi_\theta-\pi_{\theta^\dagger}}_1 \leq \frac{1}{\eta_{\min}}\epsilon \leq \epsilon'~.
    \end{align*}
\end{proof}

We now prove two important lemmas which we use in most of the results of the paper. They provide solutions to the attack subproblems for RLHF and DPO. We start with the RLHF result.

\begin{lemma}\label{lem:attack_subproblem_rlhf}
    Let $\omega^\dagger \in\mathbb{R}^d$ and let $\overline{D}=\{(\tau,\tau',o)\}$ be a given preference dataset of $n$ samples. Consider the problem
    \begin{align*}
        \min_D |D| \;\; \textnormal{such that} \;\; \omega^\dagger = \arg\min_\omega \ell^\omega_\textnormal{RLHF}\left(\overline{D}\cup D\right)~.
    \end{align*}
    Then, the solution to the above problem is the dataset of 
    \begin{align*}
        \left\lceil \frac{\left|\nabla_\omega \ell^\omega_\textnormal{RLHF}(\overline{D})^\top\omega^\dagger\right|}{\xi_{\max}}\right\rceil
    \end{align*}
    identical samples satisfying 
    \begin{align*}
        \phi(\tau)-\phi(\tau') = \xi^{-1}\left( \frac{(\omega^\dagger)^\top\nabla_\omega\ell^\omega_\textnormal{RLHF}(\overline{D})}{\left\lceil \frac{\left|\nabla_\omega \ell^\omega_\textnormal{RLHF}(\overline{D})^\top\omega^\dagger\right|}{\xi_{\max}}\right\rceil}\right)\frac{\omega^\dagger}{\norm{\omega^\dagger}^2}, \;\; o=1~.
    \end{align*}
    Moreover, if $\overline{D}=\emptyset$, the solution of the problem is the set of 
    \begin{align*}
         \left\lceil \frac{\lambda\norm{\omega^\dagger}^2}{\xi_{\max}}\right\rceil
    \end{align*}
    identical samples satisfying 
    \begin{align*}
        \phi(\tau)-\phi(\tau') = \xi^{-1}\left( \frac{\lambda\norm{\omega^\dagger}^2}{ \left\lceil \frac{\lambda\norm{\omega^\dagger}^2}{\xi_{\max}}\right\rceil}\right)\frac{\omega^\dagger}{\norm{\omega^\dagger}^2}, \;\; o=1~.
    \end{align*}
\end{lemma}
\begin{proof}
    First, note that the second case, when $\overline{D}=\emptyset$, is a direct consequence of Theorem \ref{thm:logistic_regression_teaching}. We thus focus on the general case when $\overline{D}\neq\emptyset$. 

    The first-order condition of our problem can be written as
\begin{align*}
     \sum_{(\tau,\tau',o)\in\overline{D}}\frac{-o\left(\phi(\tau)-\phi(\tau')\right)}{1+\exp\left(o(\omega^\dagger)^\top\left(\phi(\tau)-\phi(\tau')\right)\right)} + \sum_{(\tau,\tau',o)\in D}\frac{o\left(\phi(\tau)-\phi(\tau')\right)}{1+\exp\left(o(\omega^\dagger)^\top\left(\phi(\tau)-\phi(\tau')\right)\right)} + \lambda\omega = \mathbf{0}~.
\end{align*}
We can equivalently write the above as 
\begin{align*}
    \sum_{(\tau,\tau',o)\in D}\frac{o(\phi(\tau)-\phi(\tau'))}{1+\exp\left(o(\omega^\dagger)^\top(\phi(\tau)-\phi(\tau'))\right)} + \nabla_\omega\ell^{\omega^\dagger}_\textnormal{RLHF}(\cleandata) = \mathbf{0}~.
\end{align*}
Recall that we have defined
\begin{align*}
    \xi_\textnormal{max}:=\max_t \frac{t}{1+\exp(t)}~.
\end{align*}
Let us denote by $\xi^{-1}(a)$ the solution to $a=x/(1+\exp(x))$, for $a\leq\xi_{\max}$. Such a solution exists and can be written in closed form \citep{DBLP:journals/jmlr/LiuZ16} as
\begin{align*}
    \xi^{-1}(a) = a-W_\textnormal{Lam}(-a\exp(a))~,
\end{align*}
for every $a\leq \xi_{\max}$, where $W_\textnormal{Lam}$ denotes the Lambert W function. Now,
let $\widehat{D}$ be a preference dataset of 
\begin{align*}
    \widehat{n}:=\left\lceil \frac{\left|\nabla_\omega \ell^\omega_\textnormal{RLHF}(\overline{D})^\top\omega^\dagger\right|}{\xi_{\max}}\right\rceil
\end{align*}
identical samples satisfying
\begin{align*}
    \phi(\tau)-\phi(\tau') = \xi^{-1}\left( \frac{(\omega^\dagger)^\top\nabla_\omega\ell^\omega_\textnormal{RLHF}(\overline{D})}{\left\lceil \frac{\left|\nabla_\omega \ell^\omega_\textnormal{RLHF}(\overline{D})^\top\omega^\dagger\right|}{\xi_{\max}}\right\rceil}\right)\frac{\omega^\dagger}{\norm{\omega^\dagger}^2}, \;\; o=1~.
\end{align*}
Note that the first-order condition with respect to $\widehat{D}$ yields
\begin{align*}
    & \sum_{(\tau,\tau',o)\in \widehat{D}} \frac{-o(\phi(\tau)-\phi(\tau'))}{1+\exp\left(o(\omega^\dagger)^\top(\phi(\tau)-\phi(\tau'))\right)} + \nabla_\omega \ell^\omega_\textnormal{RLHF}(\overline{D}) \\ & = -\widehat{n}\frac{\xi^{-1}\left((\omega^\dagger)^\top \nabla_\omega \ell^\omega_\textnormal{RLHF}(\overline{D})\left\lceil\frac{\left|(\omega^\dagger)^\top \nabla_\omega \ell^\omega_\textnormal{RLHF}(\overline{D})\right|}{\xi_{\max}}\right\rceil^{-1}\right)}{1+\exp\left(\xi^{-1}\left((\omega^\dagger)^\top \nabla_\omega \ell^\omega_\textnormal{RLHF}(\overline{D})\left\lceil\frac{\left|(\omega^\dagger)^\top \nabla_\omega \ell^\omega_\textnormal{RLHF}(\overline{D})\right|}{\xi_{\max}}\right\rceil^{-1}\right)\right)}\cdot\frac{\nabla_\omega \ell^\omega_\textnormal{RLHF}(\overline{D})}{(\omega^\dagger)^\top \nabla_\omega \ell^\omega_\textnormal{RLHF}(\overline{D})} + \nabla_\omega \ell^\omega_\textnormal{RLHF}(\overline{D}) \\
        & = -\widehat{n} (\omega^\dagger)^\top \nabla_\omega \ell^\omega_\textnormal{RLHF}(\overline{D}) \left\lceil\frac{\left|(\omega^\dagger)^\top \nabla_\omega \ell^\omega_\textnormal{RLHF}(\overline{D})\right|}{\xi_{\max}}\right\rceil^{-1} \frac{\nabla_\omega \ell^\omega_\textnormal{RLHF}(\overline{D})}{(\omega^\dagger)^\top \nabla_\omega \ell^\omega_\textnormal{RLHF}(\overline{D})} + \nabla_\omega \ell^\omega_\textnormal{RLHF}(\overline{D}) \\
    & = \mathbf{0}~,
\end{align*}
where the penultimate equality is due to the fact that 
\begin{align*}
    \omega^\top \nabla_\omega \ell^\omega_\textnormal{RLHF}(\overline{D}) \left\lceil\frac{(\omega^\dagger)^\top \nabla_\omega \ell^\omega_\textnormal{RLHF}(\overline{D})}{\xi_{\max}}\right\rceil^{-1} \leq \xi_{\max}
\end{align*}
and the property of $\xi^{-1}(\cdot)$; the last equality follows by the definition of $\widehat{n}$. Since the function $\ell^\omega_\textnormal{RLHF}(D)$ is strongly convex, the first-order condition is enough to determine the optimal solution.
\end{proof}
Similarly, we prove an analogous result for DPO. The difference here is that the problem of interest is not a homogeneous logistic regression anymore, due to the presence of $\theta_\mu$. 

\begin{lemma}\label{lem:attack_subproblem_dpo}
    Let $\theta^\dagger, \theta_\mu \in\mathbb{R}^{d'}$ for some reference policy $\mu$ and let $\overline{D}=\{(\tau,\tau',o)\}$ be a given preference dataset of $n$ samples. Consider the problem
    \begin{align*}
        \min_D |D| \;\; \textnormal{such that} \;\; \theta^\dagger = \arg\min_\theta \ell^\theta_\textnormal{DPO}\left(\overline{D}\cup D\right)~.
    \end{align*}
    Then, the solution to the above problem is the dataset of 
    \begin{align*}
        2\ceil{\frac{\abs{(\nabla_\theta\ell^{{\theta^\dagger}}_\textnormal{DPO}(\cleandata))^\top({\theta^\dagger}-\theta_\mu)}}{2\xi_{\max}}}
    \end{align*}
    identical samples satisfying
    \begin{align*}
        \beta\br{\theta^\dagger-\theta_\mu}^\top\br{\psi (s,a)-\psi(s,a')} = o \cdot \xi^{-1}\br{\frac{(\nabla_\theta\ell^{{\theta^\dagger}}_\textnormal{DPO}(\cleandata))^\top\br{{\theta^\dagger}-\theta_\mu}}{2\ceil{\frac{\abs{(\nabla_\theta\ell^{{\theta^\dagger}}_\textnormal{DPO}(\cleandata))^\top({\theta^\dagger}-\theta_\mu)}}{2\xi_{\max}}}}}~,
    \end{align*}
    with $o=1$ for half the samples and $o=-1$ for the remaining samples. Moreover, if $\overline{D}\neq \emptyset$, the solution of the problem is the dataset of 
    \begin{align*}
        2\ceil{\frac{\lambda\norm{\theta^\dagger-\theta_\mu}^2}{2\xi_{\max}}}
    \end{align*}
    identical samples satisfying
    \begin{align*}
        \beta\br{{\theta^\dagger}-\theta_\mu}^\top\br{\psi (s,a)-\psi(s,a')} = o \cdot \xi^{-1}\br{\frac{\lambda\norm{\theta^\dagger-\theta_\mu}^2}{2\ceil{\frac{\lambda\norm{\theta^\dagger-\theta_\mu}^2}{2\xi_{\max}}}}}~,
    \end{align*}
    with $o=1$ for half the samples and $o=-1$ for the remaining samples.
\end{lemma}
\begin{proof}
    We prove the general case. The second case follows directly from the fact that, if $\overline{D}=\emptyset$, we have
    \begin{align*}
        \nabla_\theta\ell^{{\theta^\dagger}}_\textnormal{DPO}(\cleandata) = \lambda\norm{\theta^\dagger-\theta_\mu}^2~.
    \end{align*}
    First, note that, for loglinear $\mu$ with parameter $\theta_\mu$, the optimization problem of interest becomes
    \begin{align*}
        \min_{D}\; & |D|\\
        \textnormal{s.t}\; & {\theta^\dagger} = \arg\min_{\theta} \sum_{(s,a,a',o)\in D\cup \cleandata}\log\left( 1 + \exp\left( -o\cdot \beta (\theta-\theta_\mu)^\top\left( \psi(s,a)-\psi(s,a')\right)\right)\right) + \frac{\lambda}{2}\norm{\theta-\theta_\mu}^2~.
    \end{align*}
    The first-order condition of the above can be written as
    \begin{align*}
        \sum_{(s,a,a',o)\in D}\frac{-\beta\left( \psi(s,a)-\psi(s,a')\right)}{ 1 + \exp\left( o\cdot \beta (\theta^\dagger-\theta_\mu)^\top\left( \psi(s,a)-\psi(s,a')\right)\right)} + \nabla_\theta\ell^{\theta^\dagger}_\textnormal{DPO}(\overline{D}) = \mathbf{0}~. 
    \end{align*}
    Now let us consider the following construction. Let
    \begin{align*}
        \widehat{n} = 2\left\lceil\frac{\left|\nabla_\theta\ell^{\theta^\dagger}_\textnormal{DPO}(\overline{D})^\top\left(\theta^\dagger-\theta_\mu\right)\right|}{2\xi_{\max}}\right\rceil~.
    \end{align*}
    For every $1\leq i\leq \widehat{n}/2$, let $\psi_+$ be such that
    \begin{align*}
        \beta\left(\theta^\dagger-\theta_\mu\right)^\top\psi_+ = z~,
    \end{align*}
    where
    \begin{align*}
        z = \xi^{-1}\left(\ \frac{\nabla_\theta\ell^{\theta^\dagger}_\textnormal{DPO}(\overline{D})^\top\left(\theta^\dagger-\theta_\mu\right)}{\widehat{n}}\right)~,
    \end{align*}
    and for every $\widehat{n}/2 + 1\leq j\leq \widehat{n}$, let
    \begin{align*}
        \psi_- = \psi_+ -\frac{2 z}{\beta\nabla_\theta\ell^{\theta^\dagger}_\textnormal{DPO}(\overline{D})^\top\left(\theta^\dagger-\theta_\mu\right)}\nabla_\theta\ell^{\theta^\dagger}_\textnormal{DPO}(\overline{D})~.
    \end{align*}
    Note that
    \begin{align*}
        \beta\left(\theta^\dagger -\theta_\mu\right)^\top\psi_- = -z~.
    \end{align*}
    Using dataset $\widehat{D}=\left\{ (s_i,a_i,a'_i,o_i)\right\}^{\widehat{n}}_{i=1}$ such that $\psi(s_i,a_i)-\psi(s_i,a'_i)=\psi_+$ and $o_i=1$, for all $1\leq i\leq \widehat{n}/2$, and $\psi(s_i,a_i)-\psi(a_i,s'_i)=\psi_-$ and $o_i=-1$, for all $\widehat{n}/2+1\leq j\leq \widehat{n}$, we consider the first-order condition of our problem:
    \begin{align*}
        & \frac{-\beta \widehat{n}}{2}\cdot \frac{1}{1+\exp\left(\beta\left(\theta^\dagger-\theta_\mu\right)^\top\psi_+\right)}\psi_+ + \frac{\beta \widehat{n}}{2}\cdot\frac{1}{1+\exp\left(-\beta\left(\theta^\dagger-\theta_\mu\right)^\top\psi_-\right)}\psi_- + \nabla_\theta\ell^{\theta^\dagger}_\textnormal{DPO}(\overline{D}) \\
            & = \frac{-\beta \widehat{n}}{2}\cdot\frac{1}{1+\exp\left( z\right)}\psi_+ + \frac{\beta \widehat{n}}{2}\cdot\frac{1}{1+\exp\left(z\right)}\left(\psi_+ - \frac{2 z}{\beta\nabla_\theta\ell^{\theta^\dagger}_\textnormal{DPO}(\overline{D})^\top\left(\theta^\dagger-\theta_\mu\right)}\nabla_\theta\ell^{\theta^\dagger}_\textnormal{DPO}(\overline{D})\right) + \nabla_\theta\ell^{\theta^\dagger}_\textnormal{DPO}(\overline{D})\\
        & = -\widehat{n}\cdot\frac{z}{1+\exp(z)}\left( \frac{1}{\nabla_\theta\ell^{\theta^\dagger}_\textnormal{DPO}(\overline{D})^\top\left(\theta^\dagger-\theta_\mu\right)}\nabla_\theta\ell^{\theta^\dagger}_\textnormal{DPO}(\overline{D})\right) + \nabla_\theta\ell^{\theta^\dagger}_\textnormal{DPO}(\overline{D})\\
            & = -\widehat{n}\cdot\frac{ \xi^{-1}\left(\ \frac{\nabla_\theta\ell^{\theta^\dagger}_\textnormal{DPO}(\overline{D})^\top\left(\theta^\dagger-\theta_\mu\right)}{\widehat{n}}\right)}{1+\exp\left( \xi^{-1}\left(\ \frac{\nabla_\theta\ell^{\theta^\dagger}_\textnormal{DPO}(\overline{D})^\top\left(\theta^\dagger-\theta_\mu\right)}{\widehat{n}}\right)\right)}\left( \frac{1}{\nabla_\theta\ell^{\theta^\dagger}_\textnormal{DPO}(\overline{D})^\top\left(\theta^\dagger-\theta_\mu\right)}\nabla_\theta\ell^{\theta^\dagger}_\textnormal{DPO}(\overline{D})\right) + \nabla_\theta\ell^{\theta^\dagger}_\textnormal{DPO}(\overline{D})\\
        & = -\widehat{n}\cdot  \ \frac{\nabla_\theta\ell^{\theta^\dagger}_\textnormal{DPO}(\overline{D})^\top\left(\theta^\dagger-\theta_\mu\right)}{\widehat{n}}\left( \frac{1}{\nabla_\theta\ell^{\theta^\dagger}_\textnormal{DPO}(\overline{D})^\top\left(\theta^\dagger-\theta_\mu\right)}\nabla_\theta\ell^{\theta^\dagger}_\textnormal{DPO}(\overline{D})\right) + \nabla_\theta\ell^{\theta^\dagger}_\textnormal{DPO}(\overline{D})\\
            & = \mathbf{0}~.
    \end{align*}
    Strong convexity of $\ell^\theta_\textnormal{DPO}(D)$ (see Lemma \ref{lem:logistic_regression_properties}) implies that the first-order condition is enough to guarantee optimality. 
\end{proof}

Next, we derive the DPO loss for loglinear policies and reference policy.

\begin{lemma}\label{lem:loglinear_dpo_loss}
    The DPO loss for loglinear policy parametrization and loglinear reference policy $\mu$ can be written as
    \begin{align*}
        \ell^\theta_\textnormal{DPO}(D) = \sum_{(s,a,a',o)\in D} \log\left(1+\exp\left( -o\beta \left(\theta-\theta_\mu\right)^\top\left(\psi(s,a)-\psi(s,a')\right)\right)\right) + \frac{\lambda}{2}\norm{\theta-\theta_\mu}^2~.
    \end{align*}
\end{lemma}
\begin{proof}
    Note that 
    \begin{align*}
        & \ell^\theta_\textnormal{DPO}\left( D\right) =- \sum_{(\tau,\tau',o)\in D} \log \sigma \br{o\cdot\br{\beta \log \frac{\pi_\theta(a|s)}{\mu(a|s)} - \beta\log\frac{\pi_\theta(a'|s)}{\mu(a'|s)}}} + \frac{\lambda}{2}\norm{\theta-\theta_\mu}^2\\
            & = \sum_{(\tau,\tau',o)\in D} \log \left(1+\exp\br{-o\cdot\br{\beta \log \frac{\pi_\theta(a|s)}{\mu(a|s)} - \beta\log\frac{\pi_\theta(a'|s)}{\mu(a'|s)}}} \right) + \frac{\lambda}{2}\norm{\theta-\theta_\mu}^2\\
        & = \sum_{(s,a,a',o)\in D} \log \left(1+\exp\br{-o\beta \log \frac{\pi_\theta(a|s)\mu(a'|s)}{\mu(a|s)\pi_\theta(a'|s)}} \right) + \frac{\lambda}{2}\norm{\theta-\theta_\mu}^2\\
            & = \sum_{(s,a,a',o)\in D} \log \Bigg(1+\exp\Bigg(-o\beta  \\ &   \log \frac{\exp(\theta^\top\psi(s,a))\exp(\theta_\mu^\top\psi(s,a'))\sum_{a''}\exp(\theta_\mu^\top\psi(s,a''))\sum_{a''}\exp(\theta^\top\psi(s,a''))}{\exp(\theta^\top_\mu\psi(s,a'))\exp(\theta^\top\psi(s,a'))\sum_{a''}\exp(\theta_\mu^\top\psi(s,a''))\sum_{a''}\exp(\theta^\top\psi(s,a''))} \Bigg)\Bigg) + \frac{\lambda}{2}\norm{\theta-\theta_\mu}^2\\
        & = \sum_{(s,a,a',o)\in D} \log\left(1+\exp\left( -o\beta \left(\theta-\theta_\mu\right)^\top\left(\psi(s,a)-\psi(s,a')\right)\right)\right) + \frac{\lambda}{2}\norm{\theta-\theta_\mu}^2~.
    \end{align*}
\end{proof}

Next, we prove that the KL constraints are convex for loglinear policies.
\begin{lemma}\label{lem:convex-kl-constraint}
    Let $\pi_\theta\in\Pi^\textnormal{log}$ be a loglinear policy with respect to feature mapping $\psi$. Then, $D_\textnormal{KL}(\pi^\dagger||\pi_\theta)$ is convex.
\end{lemma}
\begin{proof}
    Using the gradient of $\pi_\theta$ from the proof of Lemma \ref{lem:gradient_rlhf_mdp}, we have
    \begin{align*}
        \nabla_\theta D_\textnormal{KL}(\pi^\dagger||\pi_\theta) & = \nabla_\theta\sum_{s,a}\rho(s)\pi^\dagger(a|s)\left(\log\pi^\dagger(a|s)-\log\pi_\theta(a|s)\right)\\
            & = \nabla_\theta\left(\sum_s\rho(s)\log\sum_a\exp\left(\theta^\top\psi(s,a)\right) - \sum_{s,a}\rho(s)\pi^\dagger(a|s)\theta^\top\psi(s,a)\right)\\
        & = \sum_{s,a}\rho(s)\pi_\theta(a|s)\psi(s,a) - \sum_{s,a}\rho(s)\pi^\dagger(a|s)\psi(s,a)~.
    \end{align*}
    Furthermore, note that for the Hessian we have
    \begin{align*}
        & \nabla^2_\theta D_\textnormal{KL}(\pi^\dagger||\pi_\theta) = \nabla_\theta\left(\sum_{s,a}\rho(s)\pi_\theta(a|s)\psi(s,a) - \sum_{s,a}\rho(s)\pi^\dagger(a|s)\psi(s,a)\right)\\
            & = \sum_{s,a}\rho(s)\pi_\theta(a|s)\left(\psi(s,a)-\sum_{a'}\pi_\theta(a'|s)\psi(s,a')\right)\psi(s,a)^\top\\
        & = \sum_{s}\rho(s)\left(\sum_a \pi_\theta(a|s)\psi(s,a)\psi(s,a)^\top -\left(\sum_a \pi_\theta(a|s)\psi(s,a)\right)\left(\sum_a \pi_\theta(a|s)\psi(s,a)\right)^\top\right)\\
            & = \sum_s\rho(s) \expp_{a\sim\pi_\theta(\cdot|s)}\left[\left(\psi(s,a)-\expp_{a'\sim\pi_\theta(\cdot|s)}\left[\psi(s,a')\right]\right)\left(\psi(s,a)-\expp_{a'\sim\pi_\theta(\cdot|s)}\left[\psi(s,a')\right]\right)^\top\right]\\
        & \succeq 0~.
    \end{align*}
\end{proof}

Next, we provide the necessary condition for the constructed dataset to satisfy the feature conditions.

\begin{lemma}\label{lem:bounded_feature_norm} 
    Let $\omega_0$ be a given parameter and $f$ be an arbitrary function of $\omega_0$. Assume that $\gamma$ is such that 
    \begin{align*}
        1-\gamma \leq \frac{2\norm{\omega_0}}{\xi_{\max}+1}~.
    \end{align*}
    Then, if there exist trajectory pairs $(\tau,\tau')$, for which
    \begin{align*}
        \phi(\tau)-\phi(\tau') = \xi^{-1}\left(f(\omega_0)\ceil{\frac{f(\omega_0)}{\xi_{\max}}}^{-1}\right)\frac{\omega_0}{\norm{\omega_0}}~,
    \end{align*}
    they are feasible in the feature space, in the sense that they satisfy $\norm{\phi(\tau)}_2,\norm{\phi(\tau')}_2\leq 1$.
\end{lemma}
\begin{proof}
    First, we will consider the case when the range of $f$ is non-negative, i.e., $f(\omega_0)\geq 0$. For this case, note that 
    \begin{align*}
        0 \leq f(\omega_0)\left\lceil\frac{f(\omega_0)}{\xi_{\max}}\right\rceil^{-1} \leq \xi_{\max} = W_\textnormal{Lam}(1/e) < 0.3~,
    \end{align*}
    where we recall that 
    \begin{align*}
        \xi_{\max} = \max_t \frac{t}{1+\exp(t)}~.
    \end{align*}
    Now, note that we can write
    \begin{align}
        \xi^{-1}(a) = a - W_\textnormal{Lam}(-a\exp(a))~,\label{eq:lambert_01}
    \end{align}
    for any $a\geq -1/e$, since, if we let $t^*=a - W_\textnormal{Lam}(-a\exp(a))$, we obtain
    \begin{align*}
        \frac{t^*}{1+\exp(t^*)} = \frac{a - W_\textnormal{Lam}(-a\exp(a))}{1+\exp(a - W_\textnormal{Lam}(-a\exp(a)))} = \frac{a + a\exp(a)/\exp(a-W_\textnormal{Lam}(-a\exp(a)))}{1+\exp(a - W_\textnormal{Lam}(-a\exp(a)))} = a~.
    \end{align*}
    Our aim is to show that we can find pairs $(\tau,\tau')$ that satisfy
    \begin{align*}
        \phi(\tau)-\phi(\tau') = \xi^{-1}\left(f(\omega_0)\ceil{\frac{f(\omega_0)}{\xi_{\max}}}^{-1}\right)\frac{\omega_0}{\norm{\omega_0}}~.
    \end{align*}
    This amounts to showing the right-hand side of the above equation satisfies the feature boundedness, i.e., that it is consistent with the assumption that $\norm{\phi(\tau)}\leq 1/(1-\gamma)$, for all $\tau$. 
    
    To that end, we will derive upper bounds on the norm of the quantity on the right hand side of the equation above. First, we consider the $\xi^{-1}(\cdot)$ term. We have
    \begin{align}
        & \abs{\xi^{-1}\left(f(\omega_0)\ceil{\frac{f(\omega_0)}{\xi_{\max}}}^{-1}\right)} \nonumber\\ & = \abs{f(\omega_0)\ceil{\frac{f(\omega_0)}{\xi_{\max}}}^{-1} - W_\textnormal{Lam}\left(-f(\omega_0)\ceil{\frac{f(\omega_0)}{\xi_{\max}}}^{-1} \exp\left(f(\omega_0)\ceil{\frac{f(\omega_0)}{\xi_{\max}}}^{-1}\right)\right)}\label{eq:lambert_001} \\
            & \leq \abs{f(\omega_0)\ceil{\frac{f(\omega_0)}{\xi_{\max}}}^{-1}} + \abs{W_\textnormal{Lam}\left(-f(\omega_0)\ceil{\frac{f(\omega_0)}{\xi_{\max}}}^{-1} \exp\left(f(\omega_0)\ceil{\frac{f(\omega_0)}{\xi_{\max}}}^{-1}\right)\right)} \label{eq:lambert_002}\\
        & \leq \xi_{\max} + \abs{\log\left(1 - f(\omega_0)\ceil{\frac{f(\omega_0)}{\xi_{\max}}}^{-1} \exp\left(f(\omega_0)\ceil{\frac{f(\omega_0)}{\xi_{\max}}}^{-1}\right)\right)} \label{eq:lambert_003}\\
            & \leq \xi_{\max} + \abs{\log\left(1-\xi_{\max}\exp\left(\xi_{\max}\right)\right)}\label{eq:lambert_004}~,
    \end{align}
    where \eqref{eq:lambert_001} follows from \eqref{eq:lambert_01}; \eqref{eq:lambert_002} follows from the triangle inequality; \eqref{eq:lambert_003} follows from Theorem 2.3 of \citep{hoorfar2008inequalities}, where we let $y=1$; \eqref{eq:lambert_004} follows from the fact that the function $\log(1-xe^x)$ is negative and decreasing in the range $(0,\xi_{\max})$. 

    Since we should have 
    \begin{align*}
        \norm{\phi(\tau)-\phi(\tau')} \leq \norm{\phi(\tau)} + \norm{\phi(\tau')} \leq \frac{2}{1-\gamma}~,
    \end{align*}
    for the construction to be feasible in this case, we need
    \begin{align*}
        \norm{ \xi^{-1}\left(f(\omega_0)\left\lceil\frac{f(\omega_0)}{\xi_{\max}}\right\rceil^{-1}\right)\frac{\omega_0}{\norm{\omega_0}}} \leq \frac{\xi_{\max} + \abs{\log\left(1-\xi_{\max}\exp\left(\xi_{\max}\right)\right)}}{\norm{\omega_0}} \leq \frac{2}{1-\gamma}~.
    \end{align*}
    Next, let us now consider the case when $f(\omega_0) \leq 0$. In this case, the dataset construction uses the term
    \begin{align*}
        f(\omega_0)\ceil{\frac{\abs{f(\omega_0)}}{\xi_{\max}}}^{-1}~.
    \end{align*}
    Note that, since $f(\omega_0)\leq 0$, we have
    \begin{align*}
        -\xi_{\max} \leq f(\omega_0)\ceil{\frac{\abs{f(\omega_0)}}{\xi_{\max}}}^{-1} \leq -\frac{\abs{f(\omega_0)}\xi_{\max}}{\abs{f(\omega_0)}+\xi_{\max}} \leq 0~.
    \end{align*}
    In this case, we obtain
    \begin{align}
        \abs{\xi^{-1}\left(f(\omega_0)\ceil{\frac{f(\omega_0)}{\xi_{\max}}}^{-1}\right)}  & \leq \xi_{\max} + \abs{\log\left(1 - f(\omega_0)\ceil{\frac{f(\omega_0)}{\xi_{\max}}}^{-1} \exp\left(f(\omega_0)\ceil{\frac{f(\omega_0)}{\xi_{\max}}}^{-1}\right)\right)} \label{eq:lambert_005}\\
            & \leq \xi_{\max} + 1~,\label{eq:lambert_006}
    \end{align}
    where \eqref{eq:lambert_005} follows directly from \eqref{eq:lambert_003}, while \eqref{eq:lambert_006} follows from the fact that the maximum value of $\log(1-xe^x)$ in the interval $(-\infty,0)$ is upper bounded by $1$. 
    
    Thus, combining both cases, we obtain the condition
    \begin{align*}
        \max\left\{ \frac{\xi_{\max} + \abs{\log\left(1-\xi_{\max}\exp\left(\xi_{\max}\right)\right)}}{\norm{\omega_0}}, \frac{\xi_{\max}+1}{\norm{\omega_0}}\right\} = \frac{\xi_{\max}+1}{\norm{\omega_0}}\leq \frac{2}{1-\gamma}~.
    \end{align*}
    This finally implies that we should pick $\gamma$ such that 
    \begin{align*}
        1-\gamma \leq \frac{2\norm{\omega_0}}{\xi_{\max}+1}~.
    \end{align*}
    
\end{proof}

Next, we provide results characterizing the spectra of matrices relevant to our setting.
\begin{lemma}\label{lem:spectral_bounds}
    We have
    \begin{align*}
        \sigma_{\max}\left(M_{\pi^\dagger}\right) \leq \sqrt{SA}
    \end{align*}
    and
    \begin{align*}
        \sigma_{\max}\left(\Sigma_D\right) \leq \frac{1}{1-\gamma}~.
    \end{align*}
\end{lemma}
\begin{proof}
    Note that
    \begin{align*}
        \norm{M_{\pi^\dagger}} & = \sqrt{\norm{M_{\pi^\dagger}^\top M_{\pi^\dagger}}} = \sqrt{\norm{M_{\pi^\dagger}M_{\pi^\dagger}^\top}} \\ &  \leq \sqrt{2\min\{ SA, d\}\max_{s,a}\norm{\phi(s,a)}} \\ & \leq \sqrt{2\min\{ SA, d\}}~. 
\end{align*} 
    Moreover, we have that
    \begin{align*}
        \sigma_{\max}\left(\Sigma_D\right) & = \sigma_{\max}\left(\frac{1}{n}\sum_{(\tau,\tau')\in D}\left(\phi(\tau)-\phi(\tau')\right)\left(\phi(\tau)-\phi(\tau')\right)^\top\right)\\
            & \leq \frac{1}{n}\sum_{(\tau,\tau')\in D} \sigma_{\max}\left( \left(\phi(\tau)-\phi(\tau')\right)\left(\phi(\tau)-\phi(\tau')\right)^\top\right)\\
        & \leq \frac{1}{n}\sum_{(\tau,\tau')\in D} Tr\left( \left(\phi(\tau)-\phi(\tau')\right)\left(\phi(\tau)-\phi(\tau')\right)^\top\right)\\
            & = \frac{1}{n}\sum_{(\tau,\tau')\in D} \left(\phi(\tau)-\phi(\tau')\right)^\top\left(\phi(\tau)-\phi(\tau')\right)^\top\\
        & \leq \frac{1}{n}\sum_{(\tau,\tau')\in D}\left(\norm{\phi(\tau)} + \norm{\phi(\tau')}\right)\\
            & \leq \frac{2}{1-\gamma}~,
    \end{align*}
    where $Tr(M)$ denotes the trace of matrix $M$.
\end{proof}

Next, we provide a characterization of the spectral properties of a matrix of interest for the RLHF setting.

\begin{lemma}\label{lem:spectrum_of_X_and_Y}
    Given parameter $\omega$, let $Y^\omega_{\cleandata}$ be defined as in \eqref{eq:Y_matrix_definition}. Then, the following inequalities hold:
    \begin{align}\label{eq:rlhf-warmup-augm-covmatrix-spectrum}
    \overline{n}\widetilde{C}^\phi_{\omega,\cleandata}C^\phi_{\cleandata} + 2\lambda  \leq \sigma_{\min}\left(Y^\omega_{\cleandata}\right)\leq \sigma_{\max}\left(Y^\omega_{\cleandata}\right)\leq \overline{n}\sigma_{\max}\left(\Sigma^\phi_{\cleandata}\right) + 2\lambda\leq \frac{\overline{n}}{1-\gamma}+2\lambda~,
\end{align}
and
\begin{align*}
    \frac{1-\gamma}{\overline{n} + 2(1-\gamma)\lambda} \leq \sigma_{\min}\left((Y^\omega_{\cleandata})^{-1}\right) \leq \sigma_{\max}((Y^\omega_{\cleandata})^{-1}) \leq \frac{1}{\overline{n}\widetilde{C}^\phi_{\omega,\cleandata}C^\phi_{\cleandata} +2\lambda}~,
\end{align*}
where 
\begin{align*}
    \widetilde{C}^\phi_{\omega,\cleandata} = \min_{(\tau,\tau',o)\in \cleandata} \frac{\exp(o\omega^\top\left(\phi(\tau)-\phi(\tau')\right)}{\left(1+\exp(o\omega^\top\left(\phi(\tau)-\phi(\tau')\right)\right)^2}~.
\end{align*}
\end{lemma}
\begin{proof}
The first inequality in  \eqref{eq:rlhf-warmup-augm-covmatrix-spectrum} follows from
\begin{align}
    \sigma_{\min}(Y^\omega_{\cleandata}) & \geq \sigma_{\min}\left(  \sum_{(\tau,\tau',o)\in \cleandata} \frac{\exp(o\omega^\top\left(\phi(\tau)-\phi(\tau')\right)}{\left(1+\exp(o\omega^\top\left(\phi(\tau)-\phi(\tau')\right)\right)^2}\left(\phi(\tau)-\phi(\tau')\right)\left(\phi(\tau)-\phi(\tau')\right)^\top + 2\lambda I\right) \nonumber\\
        & \geq \widetilde{C}^\phi_{\omega,\cleandata} \sigma_{\min}\left(\sum_{(\tau,\tau',o)\in \cleandata}\left(\phi(\tau)-\phi(\tau')\right)\left(\phi(\tau)-\phi(\tau')\right)^\top\right) + 2\lambda \label{eq:rlhf-warmup-augm-Yeq01}\\
    & \geq \overline{n}\widetilde{C}^\phi_{\omega,\cleandata}C^\phi_{\cleandata} + 2\lambda\label{eq:rlhf-warmup-augm-Yeq02}~,
\end{align}
where  \eqref{eq:rlhf-warmup-augm-Yeq01} uses the definition of $\widetilde{C}^\phi_{\omega,\cleandata}$
and  \eqref{eq:rlhf-warmup-augm-Yeq02} follows by the assumption on $\Sigma^\phi_D$.  The upper bound in  \eqref{eq:rlhf-warmup-augm-covmatrix-spectrum} follows from
\begin{align}
    \norm{Y^\omega_{\cleandata}} & = \sigma_{\max}\left(Y^\omega_{\cleandata}\right) \nonumber\\ & = \sigma_{\max}\left( \sum_{(\tau,\tau',o)\in \cleandata} \frac{\exp(o\omega^\top\left(\phi(\tau)-\phi(\tau')\right)}{\left(1+\exp(o\omega^\top\left(\phi(\tau)-\phi(\tau')\right)\right)^2}\left(\phi(\tau)-\phi(\tau')\right)\left(\phi(\tau)-\phi(\tau')\right)^\top + 2\lambda I\right) \nonumber\\
        & \leq \sum_{(\tau,\tau',o)\in \cleandata} \frac{\exp(o\omega^\top\left(\phi(\tau)-\phi(\tau')\right)}{\left(1+\exp(o\omega^\top\left(\phi(\tau)-\phi(\tau')\right)\right)^2}\sigma_{\max}\left(\left(\phi(\tau)-\phi(\tau')\right)\left(\phi(\tau)-\phi(\tau')\right)^\top\right) + 2\lambda \nonumber\\
    & \leq \sum_{(\tau,\tau',o)\in \cleandata} \frac{1}{2} Tr\left(\left(\phi(\tau)-\phi(\tau')\right)\left(\phi(\tau)-\phi(\tau')\right)^\top\right) + 2\lambda \label{eq:rlhf-warmup-augm-Yeq03}\\
        & = \sum_{(\tau,\tau',o)\in \cleandata} \frac{1}{2} \left(\phi(\tau)-\phi(\tau')\right)^\top\left(\phi(\tau)-\phi(\tau')\right) + 2\lambda \nonumber\\
    & \leq \frac{\overline{n}}{1-\gamma} + 2\lambda\label{eq:rlhf-warmup-augm-Yeq04}~,
\end{align}
where  \eqref{eq:rlhf-warmup-augm-Yeq03} follows from the fact that
\begin{align*}
    \sigma_{\max}(vv^\top) \leq Tr\left(vv^\top\right) = v^\top v~,
\end{align*}
for any nonzero vector $v$, and the fact that
\begin{align*}
    \frac{x}{(1+x)^2} \leq \frac{1}{2}
\end{align*}
for all positive $x$;  \eqref{eq:rlhf-warmup-augm-Yeq04} follows from 
\begin{align*}
    \left(\phi(\tau)-\phi(\tau')\right)^\top\left(\phi(\tau)-\phi(\tau')\right) \leq 2\max_\tau\norm{\phi(\tau)} = \frac{2}{1-\gamma}~.
\end{align*}
This concludes the first part of our result.

For the characterization of the eigenspectrum of the inverse of $Y^\omega_{\cleandata}$, we make use of the above derivations and immediately observe that
\begin{align}\label{eq:rlhf-warmup-augm-Ymatrix_lower}
    \sigma_{\min}\left((Y^\omega_{\cleandata})^{-1}\right) \geq \frac{1}{\sigma_{\max}(Y^\omega_{\cleandata})} \geq \frac{1-\gamma}{\overline{n} + 2(1-\gamma)\lambda}~,
\end{align}
and 
\begin{align}
    \sigma_{\max}((Y^\omega_{\cleandata})^{-1}) \leq \frac{1}{\overline{n}\widetilde{C}^\phi_{\omega,\cleandata}C^\phi_{\cleandata} +2\lambda}~.\label{eq:rlhf-warmup-augm-Ymatrix_upper}
\end{align}
\end{proof}

Next, we show that the KL divergence of two loglinear policies is Lipshcitz with respect to their respective parameters.
\begin{lemma}\label{lem:kl-lipschitz}
    Let $\pi_\theta$ and $\pi_{\theta'}$ be two loglinear policies with respect to feature mapping $\psi$. Then, 
    \begin{align*}
        D_\textnormal{KL}\left(\pi_\theta||\pi_{\theta'}\right) \leq 2\norm{\theta-\theta'}~.
    \end{align*}
\end{lemma}
\begin{proof}
    Note that 
    \begin{align*}
        & D_\textnormal{KL}\left(\pi_\theta||\pi_{\theta'}\right) = \sum_{s,a}\pi_\theta(a|s)\left(\log\pi_\theta(a|s)-\log\pi_{\theta'}(a|s)\right)\\
            & = \sum_{s,a}\pi_\theta(a|s)\left(\theta-\theta'\right)^\top\psi(s,a) + \sum_{s}\rho(s)\left(\log\sum_{a'}\exp\left((\theta')^\top\psi(s,a')\right) - \log\sum_{a'}\exp\left(\theta^\top\psi(s,a')\right)\right)\\
        & \leq \sum_{s,a}\rho(s)\pi_\theta(a|s)\norm{\theta-\theta'}\norm{\psi(s,a)} + \sum_s\rho(s)\norm{\theta-\theta'}\norm{\psi(s,a)}\\
            & \leq 2\norm{\theta-\theta'}~,
    \end{align*}
    where the third inequality uses Cauchy-Schwarz and the fact that the log-sum-exp function is Lipshcitz with parameter $1$, since 
    \begin{align*}
        \norm{\nabla_\theta \sum_s\rho(s)\log\sum_{a}\exp\left(\theta^\top\psi(s,a)\right)} = \norm{\sum_{s,a}\rho(s)\pi_\theta(a|s)\psi(s,a)} \leq \max_{s,a}\norm{\psi(s,a)} \leq 1~.
    \end{align*}
\end{proof}

Our next result states some nice properties of the regularized logistic regression loss. These implications are easy to prove. Nevertheless, we provide the full proofs for completion.

\begin{lemma}\label{lem:logistic_regression_properties}
    Given dataset $D=\{(x_i,y_i)\}^n_{i=1}$, let 
    \begin{align*}
        \ell^\upsilon(D) = \sum_{(x,y)\in D}\log\left(1+\exp\left(\beta\cdot y \upsilon^\top x + b\right)\right) + \frac{\lambda}{2}\norm{\upsilon - \zeta}^2
    \end{align*}
    denote a regularized logistic regression loss with respect to $\upsilon\in\mathbb{R}^d$, for some $d\in \mathbb{N}$, where $b\in\mathbb{R}$ , $\zeta\in\mathbb{R}^{d}$, and $\lambda > 0$. Moreover, let
    \begin{align*}
        \Sigma_D = \frac{1}{n}\sum_{x \in D}xx^\top
    \end{align*}
    be the data covariance matrix with minimum eigenvalue $\sigma$, and let $\upsilon^*$ denote an optimal point for $\ell^\upsilon(D)$. 
    Then, the following hold:
    \begin{enumerate}
        \item The function $\ell^\upsilon(D)$ is strongly convex with parameter $n\beta C_\upsilon\sigma + \lambda$.
        \item We have $\norm{\nabla_\upsilon\ell_\upsilon(D)}\geq 2(n\beta C_\upsilon \sigma + \lambda)\norm{\upsilon-\upsilon^*}$.
        \item The function
        \begin{align*}
            X^\upsilon_D = \sum_{(x,y)\in D}\frac{\beta\cdot y}{1+\exp\left(\beta\cdot y \upsilon^\top x + b\right)}x
        \end{align*}
        is Lipschitz with parameter $n$.
        \item  We have $\norm{\nabla_\upsilon \ell^\upsilon(D)}\leq (n\beta+\lambda)\norm{\upsilon-\upsilon^*}$.
    \end{enumerate}
\end{lemma}
\begin{proof}
    First, note that the Hessian of $\ell^\upsilon(D)$ can be written as
    \begin{align*}
        \nabla^2_\upsilon\ell^\upsilon(D) & = \sum_{(x,y)\in D}\frac{\beta\exp(\beta\cdot y\upsilon^\top x+b)}{\left(1+\exp(\beta\cdot y\upsilon^\top x+b)\right)^2}xx^\top + \lambda I\\
            & \succeq \beta nC_\upsilon \frac{1}{n}\sum_{(x,y)\in D}xx^\top + \lambda I\\
        & \succeq (n\beta C_\upsilon \sigma + \lambda) I~,
    \end{align*}
    where we have only used the assumptions of the statement. This means that $\ell^\upsilon(D)$ is strongly convex with parameter $n\beta C_\upsilon\sigma + \lambda$. The second statement is a direct implication of this. 

    For the third statement, first note that we have
    \begin{align*}
        \norm{X^\upsilon_D} & = \norm{\sum_{(x,y)\in D}\frac{-\beta x}{1+\exp\left(\beta\cdot y\upsilon^\top x +b\right)}}\\
            & \leq \beta n.
    \end{align*}
    Then, we can write
    \begin{align*}
        \norm{\nabla_\upsilon\ell^\upsilon(D)} & = \norm{\nabla_\upsilon\ell^\upsilon(D)-\nabla^\upsilon\ell^{\upsilon^*}(D)}\\
            & = \norm{X^\upsilon_D+\lambda(\upsilon-\zeta) - X^{\upsilon^*}_D + \lambda(\upsilon^*-\zeta)} \\
        & \leq n\beta \norm{\upsilon-\upsilon^*} + \lambda\norm{\upsilon-\upsilon^*}\\
            & \leq (n\beta+\lambda)\norm{\upsilon-\upsilon^*}~,
    \end{align*}
    where the penultimate inequality uses Lipschitzness of $X^\upsilon_D$.
\end{proof}

The next result is used for the solution of the logistic regression subproblems. 
    \begin{theorem}[Proposition 3 of \citep{DBLP:journals/jmlr/LiuZ16}]\label{thm:logistic_regression_teaching}
        Given any target model ${\omega^\dagger}\neq \mathbf{0}$, the following is a teaching set for the logistic regression problem
        \begin{align*}
            \min_{D'}\; &  |D'|\;\; 
            \textnormal{such that}\;\;  {\omega^\dagger} \in \arg\min_\omega \sum_{(x_i,y_i)\in D'}\log\left(1+\exp\left(-y_ix_i^\top\omega\right)\right) + \frac{\lambda}{2}\norm{\omega}^2~.
        \end{align*}
        There are $\widehat{n} = \left\lceil \frac{\lambda\norm{{\omega^\dagger}}^2}{\xi_{\max}}\right\rceil$
        identical training samples, each taking the form
        \begin{align*}
            x_i = \xi^{-1}\left(\frac{\lambda\norm{{\omega^\dagger}}^2}{\left\lceil\frac{\lambda\norm{\omega^\dagger}^2}{\xi_{\max}}\right\rceil}\right)\frac{\omega^\dagger}{\norm{{\omega^\dagger}}^2},\;\;\; y_i = 1~.
        \end{align*}
    \end{theorem}

\end{document}